\documentclass[10pt,twocolumn,letterpaper]{article}

\usepackage[pagenumbers]{cvpr} 

\usepackage{graphicx}
\usepackage{amsmath}
\usepackage{amsthm}
\usepackage{amssymb}
\usepackage{booktabs}
\usepackage{enumitem}
\usepackage{bigdelim}
\usepackage{float}

\usepackage{adjustbox}
\usepackage{algorithmicx}
\usepackage{algorithm}
\usepackage[noend]{algpseudocode}
\usepackage[noend]{algcompatible}
\algrenewcommand\algorithmicrequire{\textbf{function}} 
\algnewcommand\algorithmicreturn{\textbf{return} } 
\algnewcommand\RETURN{\State \algorithmicreturn}%

\usepackage{multirow}
\usepackage{array}
\usepackage{bbm}
\usepackage{bm}

\usepackage{paralist}

\theoremstyle{plain}
\newtheorem{theorem}{Theorem}[section]
\newtheorem{proposition}[theorem]{Proposition}

\newtheorem{corollary}[theorem]{Corollary}

\theoremstyle{definition}
\newtheorem{definition}[theorem]{Definition}
\newtheorem{remark}[theorem]{Remark}

\newenvironment{enum}
{\begin{enumerate}

}
{\end{enumerate}}

\newenvironment{enum*}
{\parskip=-4pt
\begin{enumerate*}

}
{\end{enumerate*}}

\DeclareMathOperator{\mingen}{mingen}
\DeclareMathOperator{\argmax}{argmax}
\DeclareMathOperator{\inid}{in}

\usepackage{lipsum}

\newcommand\blfootnote[1]{%
  \begingroup
  \renewcommand\thefootnote{}\footnote{#1}%
  \addtocounter{footnote}{-1}%
  \endgroup
}

\usepackage[pagebackref,breaklinks,colorlinks]{hyperref}

\usepackage[capitalize]{cleveref}
\crefname{section}{Sec.}{Secs.}
\Crefname{section}{Section}{Sections}
\Crefname{table}{Table}{Tables}
\crefname{table}{Tab.}{Tabs.}

\def\cvprPaperID{10401} 
\def\confName{CVPR}
\def\confYear{2023}

\begin{document}

\title{Abstract Visual Reasoning: An Algebraic Approach for Solving Raven's Progressive Matrices}

\author{Jingyi Xu$^{1*}$ \, Tushar Vaidya$^{2\circ*}$ \, Yufei Wu$^{2\circ*}$
\, Saket Chandra$^{1}$ \, Zhangsheng Lai$^{3\circ}$ \, Kai Fong Ernest Chong$^{1\dag}$
\\
$^1$Singapore University of Technology and Design
\\
$^2$Nanyang Technological University
\quad
$^3$Singapore Polytechnic
\\
{\tt\small jingyi\_xu@mymail.sutd.edu.sg}
\quad
{\tt\small tushar.vaidya@ntu.edu.sg} \quad {\tt\small yufei002@e.ntu.edu.sg} \\
{\tt\small lai\_zhangsheng@sp.edu.sg} \quad
{\tt\small \{saket\_chandra,ernest\_chong\}@sutd.edu.sg}}
\maketitle

\begin{abstract}
We introduce algebraic machine reasoning, a new reasoning framework that is well-suited for abstract reasoning. 
Effectively, algebraic machine reasoning reduces the difficult process of novel problem-solving to routine algebraic computation.
The fundamental algebraic objects of interest are the ideals of some suitably initialized polynomial ring. 
We shall explain how solving Raven's Progressive Matrices (RPMs) can be realized as computational problems in algebra, which combine various well-known algebraic subroutines that include: Computing the Gr\"{o}bner basis of an ideal, checking for ideal containment, etc.
Crucially, the additional algebraic structure satisfied by ideals allows for more operations on ideals beyond set-theoretic operations.  

Our algebraic machine reasoning framework is not only able to select the correct answer from a given answer set, but also able to generate the correct answer with only the question matrix given.
Experiments on the I-RAVEN dataset yield an overall $93.2\%$ accuracy, which significantly outperforms the current state-of-the-art accuracy of $77.0\%$ and exceeds human performance at $84.4\%$ accuracy. 

\blfootnote{\hspace*{-1em} $^*$ Equal contributions. $^\dag$ Corresponding author. \\ \indent $^\circ$ This work was done when the author was previously at SUTD.\\
\indent Code: \url{https://github.com/Xu-Jingyi/AlgebraicMR}}

%

\end{abstract}

\section{Introduction}
\label{sec:intro}

When we think of machine reasoning, nothing captures our imagination more than the possibility that machines would eventually surpass humans in intelligence tests and general reasoning tasks. Even for humans, to excel in IQ tests, such as the well-known Raven's progressive matrices (RPMs)~\cite{carpenter1990one}, is already a non-trivial feat.
A typical RPM instance is composed of a question matrix and an answer set; see Fig. \ref{fig:RPMexample}. 
A question matrix is a $3\times3$ grid of panels that satisfy certain hidden rules, where the first 8 panels are filled with geometric entities, and the 9-th panel is ``missing”. 
The goal is to infer the correct answer for this last panel from among the 8 panels in the given answer set.

\begin{figure}[tb]
  \centering
  \includegraphics[width=\linewidth]{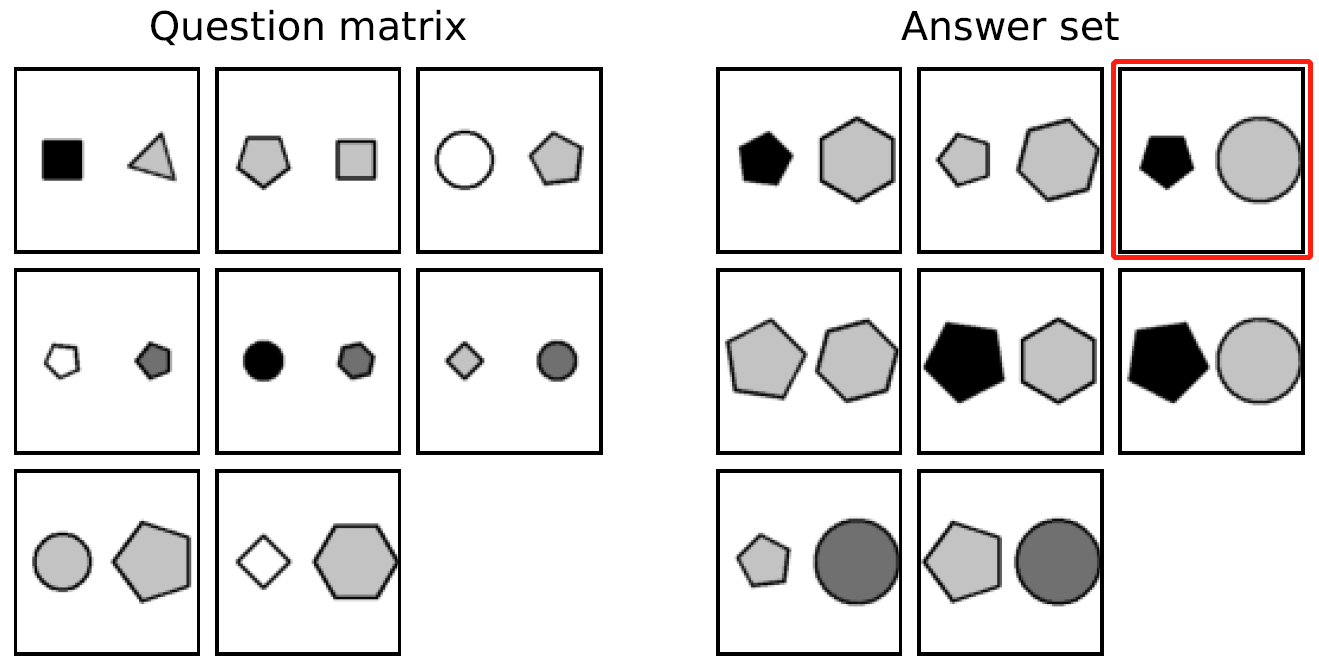}
   \caption{An example of RPM instance from the I-RAVEN dataset. The correct answer is marked with a red box.}
   \label{fig:RPMexample}
\end{figure}

The ability to solve RPMs is the quintessential display of what cognitive scientists call {fluid intelligence}. The word ``fluid'' alludes to the mental agility of discovering new relations and abstractions~\cite{perret2015children}, especially for solving \emph{novel} problems not encountered before. Thus, it is not surprising that abstract reasoning on \emph{novel} problems is widely hailed as the hallmark of human intelligence~\cite{carroll1993human}. 
\blfootnote{This work is supported by the National Research Foundation, Singapore under its AI Singapore Program (AISG Award No: AISG-RP-2019-015) and under its NRFF Program (NRFFAI1-2019-0005), and by Ministry of Education, Singapore, under its Tier 2 Research Fund (MOE-T2EP20221-0016).}

\begin{figure*}[!t]
  \centering
  \includegraphics[width=\linewidth]{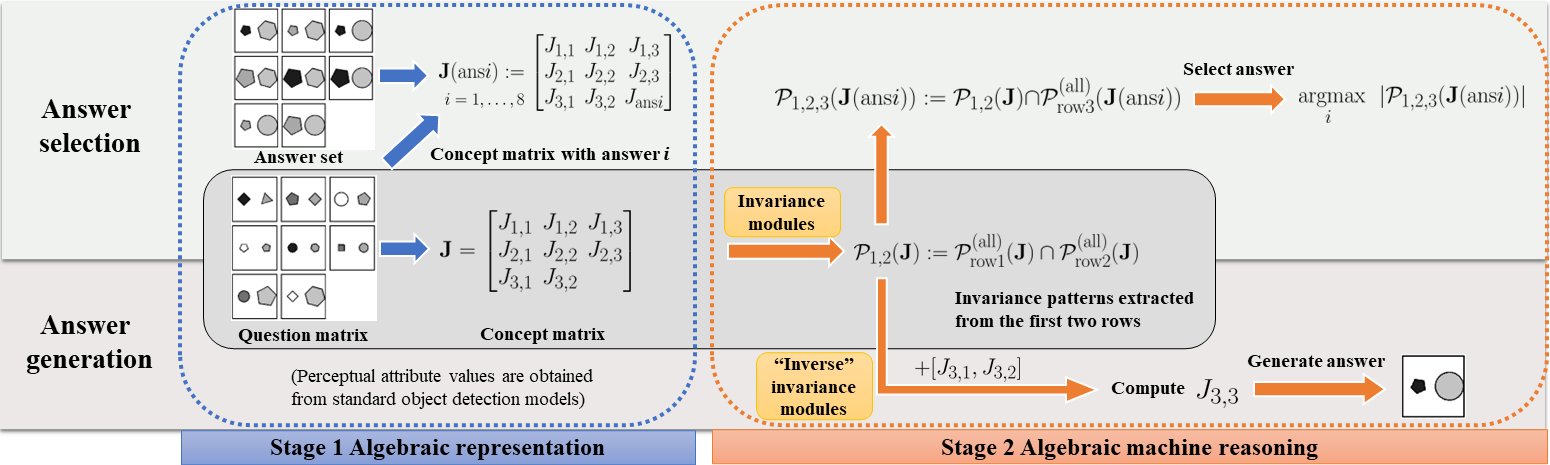}
   \caption{An overview of our algebraic machine reasoning framework, organized into 2 stages.}
   \label{fig:flowchart}
\end{figure*}

Although there has been much recent progress in machine reasoning \cite{jahrens2020solving, kim2020few, qu2021rnnlogic, rabe2021mathematical, santoro2018measuring, shegheva2018structural, van2019disentangled, wang2019satnet, zhang2019raven,zhang2019learning}, a common criticism~\cite{Davidson_Walker_2019,mitchell2021abstraction,mitchell2021ai} is that existing reasoning frameworks have focused on approaches involving extensive training, even when solving well-established reasoning tests such as RPMs. Perhaps most pertinently, as \cite{Davidson_Walker_2019} argues, reasoning tasks such as RPMs should not need task-specific performance optimization. After all, if a machine optimizes performance by training on task-specific data, then that task cannot possibly be novel to the machine.

To better emulate human reasoning, we propose what we call ``algebraic machine reasoning'', a new reasoning framework that is well-suited for abstract reasoning. 
Our framework solves RPMs \emph{without} needing to optimize for performance on task-specific data, analogous to how a gifted child solves RPMs without needing practice on RPMs. 
Our key starting point is to \emph{define} concepts as ideals of some suitably initialized polynomial ring. These ideals are treated as the ``actual objects of study'' in algebraic machine reasoning, which do not require any numerical values to be assigned to them. We shall elucidate how the RPM task can be realized as a computational problem in algebra involving ideals.

Our reasoning framework can be broadly divided into two stages: (1) algebraic representation, and (2) algebraic machine reasoning; see Fig. \ref{fig:flowchart}. In the first stage, we represent RPM panels as ideals, based on perceptual attribute values extracted from object detection models. In the second stage, we propose 4 invariance modules to extract patterns from the RPM question matrix.

To summarize, our main contributions are as follows:
\begin{itemize}
    \item We reduce ``solving the RPM task" to ``solving a computational problem in algebra". 
    Specifically, we present how the discovery of abstract patterns can be realized very concretely as algebraic computations known as primary decompositions of ideals.
    \item In our algebraic machine reasoning framework, we introduce 4 invariance modules for extracting patterns that are meaningful to humans.
    \item Our framework is not only able to select the correct answer from a given answer set, but also able to generate answers \textit{without needing any given answer set}.
    \item Experiments conducted on RAVEN and I-RAVEN datasets demonstrate that our reasoning framework significantly outperforms state-of-the-art methods.
\end{itemize}

\section{Related Work} \label{sec:RelatedWork}
\textbf{RPM solvers}.
There has been much recent interest in solving RPMs with deep-learning-based methods
\cite{santoro2018measuring, zhang2019learning, zhuo2020solving, lyu2022solving, jahrens2020solving, ZhengDistracting2019, zhuo2021effective, UnsupMohan2021, zhang2021abstract}. Most methods extract features from raw RPM images using nueral networks, and select answers by measuring panel similarities. 
Several works instead focus on generating correct answers without needing the answer set \cite{pekar2020generating, shi2021raven}.
To evaluate the reasoning capabilities of these methods, RPM-like datasets such as PGM \cite{santoro2018measuring} and RAVEN \cite{zhang2019raven} have been proposed. 
Subsequently, I-RAVEN \cite{hu2021stratified} and RAVEN-FAIR \cite{benny2021scale} are introduced to overcome a shortcut flaw in the answer set generation of RAVEN.

\textbf{Algebraic methods in AI}.\label{subsec:AlgebraicMethodsAI}
Using algebraic methods in AI is not new. Systems of polynomial equations are commonly seen in computer vision~\cite{AlgGeomComputerVision} and robotics~\cite{cox2015ideals}, which are solved algebraically via Gr\"{o}bner basis computations.
In statistical learning theory, methods in algebraic geometry~\cite{watanabe2009algebraic} and algebraic statistics \cite{drton2008lectures} are used to study singularities in statistical models~\cite{lin2011algebraic,watanabe2013widely,Watanabe2010:AsympEquivBayes,YAMAZAKI20031029}, to analyze generalization error in hierarchical models \cite{Watanabe2001:AlgebraicAnalysis,Watanabe2001:AlgGeomMethodsHierarchical}, to learn invariant subspaces of probability distributions \cite{kiraly2012algebraic,larsen2014fano}, and
to model Bayesian networks~\cite{duarte2021discrete,Sullivant2008:AlgGeomGaussianBayesian}.
A common theme in \mbox{these works is to study suitably defined algebraic varieties.} 
In deep learning, algebraic methods are used to study the expressivity of neural nets~\cite{Chong2020:ApproxCapabilitiesNN,kileel2019expressive,maragos2021tropical,zhang2018tropical}. In automated theorem proving, Gr\"{o}bner basis computations are used in proof-checking~\cite{stuber1998superposition}. 
Recently, a matrix formulation of first-order logic was applied to the RPM task \cite{zhang2022learning}, where relations are approximated by matrices and reasoning is framed as a bilevel optimization task to find best-fit matrix operators. As far as we know, methods from commutative algebra have not been used in machine reasoning.

\section{Proposed Algebraic Framework}
\label{Section:method}

In abstract reasoning, a key cognitive step is to ``discover patterns from observations'', which can be formulated concretely as ``finding invariances in observations''. 
In this section, we describe how algebraic objects known as ideals are used to represent RPM instances, how patterns are extracted from 
such algebraic representations, and how RPMs can be solved, both for answer selection and answer generation, as computational problems in algebra.

\subsection{Preliminaries}
Throughout, let $R = \mathbb{R}[x_1, \dots, x_n]$ be the ring of polynomials in variables $x_1, \dots, x_n$, with real coefficients. In particular, $R$ is closed under addition and multiplication of polynomials, i.e., for any $a, b\in R$, we have $a+b,ab \in R$.

\vspace*{-0.6em}
\subsubsection{Algebraic definitions} 
\vspace*{-0.3em}
\label{subsec:AlgebraicPrelim}
\noindent\textbf{Ideals in polynomial rings}. 
A subset $I \subseteq R$ is called an \textit{ideal} if there exist polynomials $g_1, \dots, g_k$ in $R$ such that
\vspace*{-0.7em}
\[I = \{f_1g_1 + \dots + f_kg_k| f_1, \dots, f_k \in R\}\\[-0.7em]\]
contains all polynomial combinations of $g_1, \dots, g_k$. We say that $\mathcal{G} = \{g_1, \dots, g_k\}$ is a \textit{generating set} for $I$, we call $g_1, \dots, g_k$ \textit{generators}, 
and we write either $I = \langle g_1, \dots, g_k\rangle$ or $I = \langle \mathcal{G}\rangle$. Note that generating sets of ideals are not unique. If $I$ has a generating set consisting only of monomials, then we say that $I$ is a \textit{monomial ideal}. (Recall that a \textit{monomial} is a polynomial with a single term.)
Given ideals $J_1 = \langle g_1, \dots, g_k\rangle$ and $J_2 = \langle h_1, \dots, h_{\ell}\rangle$, there are three basic operations (sums, products, intersections): 
\vspace*{-0.8em}
\begin{align*}
J_1 + J_2 &:= \langle g_1, \dots, g_k, h_1, \dots, h_{\ell}\rangle;\\
J_1J_2 &:= \langle \{g_ih_j | 1\leq i\leq k, 1\leq j\leq \ell\}\rangle; \\
J_1 \cap J_2 &:= \{r\in R: r\in J_1 \text{ and }r\in J_2\}.\\[-2.3em]
\end{align*}

Most algebraic computations involving ideals, especially ``advanced" operations (e.g. primary decompositions), require computing their Gr\"{o}bner bases as a key initial step. More generally, Gr\"{o}bner basis computation forms the backbone of most algorithms in algebra; see Appendix \ref{AppendixSubsec:grobnerBasis}.

\noindent\textbf{Primary decompositions}.
In commutative algebra, primary decompositions of ideals are a far-reaching generalization of the idea of prime factorization for integers. 
Its importance to algebraists cannot be overstated. 
Informally, every ideal $J$ has a decomposition \smash{$J = J_1 \cap \dots \cap J_s$} as an intersection of finitely many \textit{primary} ideals. This intersection is called a \textit{primary decomposition} of $J$, and each $\smash{J_j}$ is called a \textit{primary component} of the decomposition. 
In the special case when $J$ is a monomial ideal, there is an \textbf{unique} minimal primary decomposition with maximal monomial primary components \cite{BayerGalligoStillman1993}; 
We denote this unique set of primary components by $\text{pd}(J)$.
See Appendix \ref{AppendixSubsec:primaryIdeals,pd} for details.

\vspace*{-0.7em}

\subsubsection{Concepts as monomial ideals}
\label{subsubsec:ConceptsMonomialIdeals}
\vspace*{-0.3em}
\label{Subsubsection:concepts}

We define a \textit{concept} to be a monomial ideal of $R$. In particular, the zero ideal $\langle 0 \rangle \subseteq R$ is the concept ``null'', and could be interpreted as ``impossible'' or ``nothing'', while the ideal $\langle 1\rangle = R$ is the concept ``conceivable'', and could be interpreted as ``possible'' or ``everything''. 
Given a concept $J\subseteq R$, a monomial in $J$ is called an \textit{instance} of the concept. 
For example, $x_{\text{black}}x_{\text{square}}$ is an instance of $\langle x_{\text{square}}\rangle$ (the concept ``square''). 
For each $x_i$,  we say $\langle x_i\rangle \subseteq R$ is a \textit{primitive} concept, and $x_i$ is a \textit{primitive} instance.

\begin{theorem}\label{Theo:PrimaryDecompProperConcept}
There are infinitely many concepts in $R$, even though there are finitely many primitive concepts in $R$. Furthermore, if $J \subseteq R$ is a concept, then the following hold:
\begin{enum}
\item $J$ has infinitely many instances, unless $J = \langle 0 \rangle$.
\item $J$ has a unique minimal generating set consisting of \mbox{finitely many instances, which we denote by $\mingen(J)$.}
\item If $J \neq \langle 1\rangle$, then $J$ has a unique set of associated concepts $\{P_1, \dots, P_k\}$, together with a unique minimal primary decomposition $J = J_1 \cap \dots \cap J_k$, such that each $J_i$ is a concept contained in $P_i$, that is maximal among all possible primary components contained in $P_i$ that are concepts.
\end{enum}
\end{theorem}

See Appendix \ref{AppendixSubsec:ConceptsIntuition} for a proof of Theorem \ref{Theo:PrimaryDecompProperConcept} and for more details on why defining concepts as monomial ideals \mbox{captures the expressiveness of concepts in human reasoning.}

\subsection{Stage 1: Algebraic representation} \label{subsec:SymbolizationStep}
We shall use the RPM instance depicted in Fig \ref{fig:RPMexample} as our running example, to show the entire algebraic reasoning process: (1) algebraic representation; and (2) algebraic machine reasoning. In this subsection, we focus on the first stage. 
Recall that every RPM instance is composed of 16 panels filled with geometric entities. For our running example, each entity can be described using 4 attributes: ``color", ``size", ``type", and ``position". We also need one additional attribute to represent the ``number" of entities in the panel.

\vspace*{-0.6em}

\subsubsection{Attribute concepts}
\label{sec:attributeConcept}
\vspace*{-0.2em}

In human cognition, certain semantically similar concepts are naturally grouped to form a more general concept. For example, concepts such as ``red", ``green", ``blue", ``yellow", etc., can be grouped to form a new concept that represents ``color". Intuitively, we can think of ``color" as an attribute, and ``red", ``green", ``blue", ``yellow" as attribute values.

For our running example, the $5$ attributes are represented by $5$ concepts (monomial ideals). In general, all possible values for each attribute are encoded as generators for the concept representing that attribute. However, for ease of explanation, we shall consider only those attribute values that are involved in Fig. \ref{fig:RPMexample} to explain our example:
\vspace*{-0.4em}
\begin{align*}
\begin{split}
\mathcal{A}_{\text{num}} &:= \{ x_{\text{one}}, x_{\text{two}} \},\\
\mathcal{A}_{\text{pos}}&:= \{ x_{\text{left}}, x_{\text{right}} \},\\
\mathcal{A}_{\text{type}} &:= \{ x_{\text{triangle}}, x_{\text{square}}, x_{\text{pentagon}}, x_{\text{hexagon}}, x_{\text{circle}} \} , \\
\mathcal{A}_{\text{color}}  &:= \{x_{\text{white}}, x_{\text{gray}}, x_{\text{dgray}}, x_{\text{black}}\}, \\
\mathcal{A}_{\text{size}}  &:= \{ x_{\text{small}}, x_{\text{avg}}, x_{\text{large}}\}.\\[-0.6em]
\end{split}
\end{align*}
Let $\mathcal{L} := \{\text{num}, \text{pos}, \text{type}, \text{color}, \text{size}\}$ be the set of attribute labels, 
and let $\mathcal{A}_{\text{all}}:=\bigcup_{\ell\in\mathcal{L}}\mathcal{A}_\ell$. Initialize the ring $R:=\mathbb{R}[\mathcal{A}_{\text{all}}]$ of all polynomials on the variables in $\mathcal{A}_{\text{all}}$ with real coefficients.
For each $\ell \in \mathcal{L}$, let $J_{\ell}$ be the concept $\langle \mathcal{A}_{\ell}\rangle \subseteq R$.
These concepts, which we call \textit{attribute concepts}, are task-specific.
We assume humans tend to discover and organize complex patterns in terms of attributes. 
Thus for pattern extraction, we shall use the inductive bias that a concept representing a pattern is deemed meaningful if it is in some attribute concept.

\vspace*{-0.8em}
\subsubsection{Representation of RPM panels}
\label{Subsubsection:algebraicRepresentation}
\vspace*{-0.2em}

In order to encode the RPM images algebraically, we first need to train  perception modules to extract attribute information directly from raw images. One possible approach for perception, as used in our experiments, is to train 4 RetinaNet models (each with a ResNet-50 backbone) separately for all 4 attributes except ``number'', which can be directly inferred by counting the number of bounding boxes.

After extracting attribute values for entities, we can represent each panel as a concept. For example, the top-left panel of the RPM in Fig. \ref{fig:RPMexample} can be encoded as the concept
\vspace*{-0.8em}
\begin{equation*}
    J_{1,1}=\langle x_{\text{two}}x_{\text{left}}x_{\text{square}}x_{\text{black}}x_{\text{avg}}, x_{\text{two}}x_{\text{right}}x_{\text{triangle}}x_{\text{gray}}x_{\text{avg}} \rangle\\[-0.7em]
\end{equation*} 
in the polynomial ring $R$. Here, $J_{1,1}$ represents a panel with two entities, a black square of average size on the left, and a gray triangle of average size on the right. The indices in $J_{1,1}$ tell us that the panel is in row $1$, column $1$. Similarly, we can encode the remaining 7 panels of the question matrix as concepts $J_{1,2},J_{1,3},\dots,J_{3,2}$ and encode the 8 answer options as concepts $J_{\text{ans}1},\dots,J_{\text{ans}8}$. In general, every monomial generator of each concept describes an entity in the associated panel.

The list of 8 concepts $\mathbf{J}=[J_{1,1},\dots,J_{3,2}]$ shall be called a \textit{concept matrix}; this represents the 
RPM question matrix with a missing $9$-th panel. Let $\mathbf{J}_i:=[J_{i,1},J_{i,2},J_{i,3}]$ (for $i=1,2$) represent the $i$-th row in the question matrix.

\subsection{Stage 2: algebraic machine reasoning} \label{subsec:AlgebraicSolver}

Previously in Section \ref{subsec:SymbolizationStep}, we have already encoded the question matrix in an RPM instance as a concept matrix $\mathbf{J}=[J_{1,1},\dots,J_{3,2}]$. In this subsection, we will introduce the reasoning process of our algebraic framework. 

Our goal of extracting patterns for a single row of $\mathbf{J}$ can be mathematically formulated as ``finding invariance" across the concepts that represent the panels in this row. (The same process can be applied to columns.) This seemingly imprecise idea of ``finding invariance" can be realized very concretely via the computation of primary decompositions. Ideally, we want to extract patterns that are meaningful to humans. Hence we have designed 4 invariance modules to mimic human cognition in pattern recognition.

\vspace*{-0.4em}

\subsubsection{Prior knowledge}
\label{sec:priorknowledge}
\vspace*{-0.2em}

To use algebraic machine reasoning, we adopt: 
\begin{itemize}
    \item Inductive bias of attribute concepts (see Section \ref{sec:attributeConcept});
    \item Useful binary operations on numerical values; 
    \item Functions that map concepts to concepts.
\end{itemize}

There are numerous binary operations, such as $+, -, \times$, $\div, \min, \max$, etc., that can be applied to numerical values extracted from concepts. For the RPM task, we use $+, -$.  

In algebra, the study of functions between algebraic objects is a productive strategy for understanding the underlying algebraic structure. Analogously, we shall use maps on concepts to extract complex patterns.
For the RPM task, we need to cyclically order the values in $\mathcal{A}_\ell$ for each attribute $\ell\in\mathcal{L}$ before we can extract sequential information.
To encode the idea of ``next'', we introduce the function \smash{$f_{\text{next}}(J|\Delta)$} defined on concepts $J$, where $\Delta$ represents the step-size. 
Each variable $x\in\mathcal{A}_\ell$ that appears in a generator of $J$ is mapped to the $\Delta$-th variable after $x$, w.r.t. the cyclic order on $\mathcal{A}_\ell$.
For example, \smash{$f_{\text{next}}(\langle x_{\text{square}}x_{\text{gray}}x_{\text{avg}}\rangle|1)$} \smash{$= \langle x_{\text{pentagon}}x_{\text{dgray}}x_{\text{large}}\rangle$}, and \smash{$f_{\text{next}}(\langle x_{\text{square}}\rangle|-2) = \langle x_{\text{circle}}\rangle$}. 

\vspace*{-0.8em}
\subsubsection{Reasoning via primary decompositions}
\label{Sec:reasoningViaPD}
\vspace*{-0.2em}

Given concepts $J_1, \dots, J_k$ that share a common ``pattern'', how do we extract this pattern? Abstractly, a common pattern
can be treated as a concept $K$ that contains all of these concepts $J_1, \dots, J_k$. If there are several common patterns $K_1, \dots, K_r$, then each concept $J_i$ can be ``decomposed'' as $J_i = K_1 \cap \dots \cap K_r \cap K'_i$ for some ideal $K'_i$. 
Thus, we have the following algebraic problem: Given $J_1, \dots, J_k$, compute their common components $K_1, \dots, K_r$.

Recall that a concept $J$ has a unique minimal primary decomposition, since concepts are monomial ideals. 
Thus, to extract the common patterns of concepts $J_1,\dots,J_k$, we first have to compute $\text{pd}(J_1), \dots, \text{pd}(J_k)$, then extract the common primary components. The intersection of (any subset of) these common components would yield a new concept, which can be interpreted as a common pattern of the concepts $J_1, \dots, J_k$. 
As part of our inductive bias, we are only interested in those primary components that are contained in attribute concepts. See Appendix \ref{AppendixSubsec:primaryIdeals,pd} for further details.

\vspace*{-0.5em}

\subsubsection{Proposed invariance modules}
\label{Sec:module}
\vspace*{-0.1em}
Our 4 proposed invariance modules are:
(1) intra-invariance module, (2) inter-invariance module, (3) compositional invariance module, and (4) binary-operator invariance module. Intuitively, they check for 4 general types of invariances across a sequence of concepts $J_1,\dots,J_k$ (e.g. a row \smash{$\mathbf{J}_i=[J_{i,1}, J_{i,2}, J_{i,3}]$} for the RPM task). Such invariances apply not just to the RPM task, but could be applied to other RPM-like tasks, e.g. based on different prior knowledge, different grid layouts, etc. Full computational details for our running example can be found in Appendix \ref{AppendixSubsec:examples}.

\noindent\textbf{1. Intra-invariance module} extracts patterns where the set of values for some attribute within concept $J_i$ remains invariant over all $i$. First, we define \smash{$J_{+} := J_1 + \dots + J_k$} and \smash{$J_{\cap} := J_1 \cap \dots \cap J_k$}; see Section \ref{subsec:AlgebraicPrelim}. Intuitively, $J_+$ and \smash{$J_\cap$} are concepts that capture information about the entire sequence \smash{$J_1,\dots,J_k$} in two different ways.
Next, we compute the common primary components of \smash{$J_{+}$ and $J_{\cap}$} that are contained in attribute concepts.
Finally, we return the attributes associated to these common primary components:
\vspace*{-1.3em}
\begin{equation*}
\!\text{
\resizebox{1 \linewidth}{!}{
$\mathcal{P}_{\text{intra}}([J_1\dots J_k]):=\big\{\text{attr} \in \mathcal{L}\mid \exists I\in\text{pd}(J_+)\cap \text{pd}(J_{\cap}), I\subseteq \langle \mathcal{A}_{\text{attr}}\rangle\big\}$.}}\\[-0.45em]
\end{equation*}

\noindent\textbf{2. Inter-invariance module} extracts patterns arising from the set difference between $\text{pd}(J_\cap)$ and $\text{pd}(J_+)$. Thereafter, we check for the invariance of these extracted patterns across multiple sequences. The extracted set of patterns is:
\vspace*{-0.9em}
\begin{equation*}
\resizebox{1 \linewidth}{!}{
$\mathcal{P}_{\text{inter}}([J_1,\dots,J_k]):=\Bigg\{(\text{attr},\mathcal{I}) \Bigg| \ \begin{matrix}\mathcal{I}\subseteq \text{pd}(J_\cap)-\text{pd}(J_+),  \\ \text{attr} \in \mathcal{L}, I\subseteq \langle \mathcal{A}_{\text{attr}} \rangle \  \forall I \in \mathcal{I} \end{matrix} \Bigg\},$}\\[-0.9em]
\end{equation*}
where $\mathcal{I}$ is a set of concepts, and  ``$-$" refers to set difference. We omit $\text{pd}(J_+)$ so that we do not overcount the patterns already extracted in the previous module. Informally, for each pair $(\text{attr},\mathcal{I})$, the concepts in $\mathcal{I}$ can be interpreted as those ``primary" concepts that correspond to at least one of $J_1,\dots,J_k$, that do not correspond to all of $J_1,\dots,J_k$, and that are contained in $\langle \mathcal{A}_\text{attr} \rangle$.

\noindent\textbf{3. Compositional invariance module} extracts patterns arising from invariant attribute values in the following new sequence of concepts:
\vspace*{-0.7em}
\[[J'_1,\dots,J'_k]=[f^{k-1}(J_1), f^{k-2}(J_2), \dots, f(J_{k-1}), J_k],\\[-0.6em]\]
where $f$ is some given function. 
Intuitively, for such patterns, there are some attributes whose values are invariant in $[f(J_i), J_{i+1}]$ for all $i=1,\dots,k-1$.
By checking the intersection of primary components of the concepts in the new sequence, the extracted set of patterns is given by:
\vspace*{-0.8em}
\begin{equation*}
\resizebox{1 \linewidth}{!}{
$\mathcal{P}_{\text{comp}}([J_1,\dots,J_k]):=\Bigg\{(\text{attr},f) \Bigg| \ \begin{matrix}\exists I \in \bigcap_{i=1}^k \text{pd}(f^{k-i}(J_i)),  \\ \text{attr} \in \mathcal{L}, I \subseteq \langle \mathcal{A}_{\text{attr}} \rangle \end{matrix} \Bigg\}.$}\\[-0.9em]
\end{equation*}
The given function used for the RPM task is \smash{$f_{\text{next}}(\cdot|\Delta)$}, where $\Delta$ represents the number of steps; see Section \ref{sec:priorknowledge}.

\noindent\textbf{4. Binary-operator module} extracts numerical patterns, based on a given real-valued function $g$ on concepts, and a given set $\Lambda$ of binary operators. The extracted patterns are:
\vspace*{-0.9em}
\begin{equation*}
\resizebox{1 \linewidth}{!}{$
    \mathcal{P}_\text{binary}(\mathbf{J}_i):=\Bigg\{\overline{\boldsymbol{\oslash}} \Bigg| \ \begin{matrix} 
    \overline{\boldsymbol{\oslash}}=[\oslash_1,\dots,\oslash_{k-2}], \  \oslash_i \in \Lambda,
      \\ g(J_1) \oslash_1 \dots\oslash_{k-2} g(J_{k-1})=g(J_k) \end{matrix} \Bigg\}.$}\\[-1.0em]
\end{equation*}

\subsubsection{Extracting row-wise patterns}
\label{Sec:extractPattern}
\vspace*{-0.3em}
Given a concept matrix $\mathbf{J}=\smash{[J_{1,1}, \dots, J_{3,2}]}$, how do we extract the patterns from its $i$-th row?
We first begin by extracting the common position values among all $8$ panels:
\vspace*{-0.7em}
\begin{equation*}
\label{eq:commonPosition}
    \!\text{comPos}(\mathbf{J}):=\big\{ p\in \mathcal{A}_{\text{pos}} \mid \exists I \in \textstyle\bigcap_{J\in \mathbf{J}} \text{pd}(J), p\in I\big\}\\[-0.7em]
\end{equation*}
For each common position $\smash{p}\in \text{comPos}(\mathbf{J})$, we generate two new concept matrices $\smash{\bar{\mathbf{J}}^{(p)}}$ and $\smash{\hat{\mathbf{J}}^{(p)}}$, such that:
\begin{itemize}
    \item Each concept $\smash{\bar{J}^{(p)}_{\smash{i,j}}}$ in $\smash{\bar{\mathbf{J}}^{(p)}}$ is generated by the unique generator in $J_{i,j}$ that is divisible by $p$;
    \item Each concept ${\hat{J}^{(p)}_{\smash{i,j}}}$ in $\smash{\hat{\mathbf{J}}^{(p)}}$ is generated by all generators in \smash{$J_{i,j}$} that are not divisible by $p$. 
\end{itemize}
(Recall that generators of a concept are polynomials.)

Informally, we are splitting each panel in the RPM image into 2 panels, one that contains only the entity in the common position $p$, and the other that contains all remaining entities not in position $p$. This step allows us to reason about rules that involve only a portion of the panels.

Consequently, if \smash{$\text{comPos}(\mathbf{J})=\{p_1,\dots,p_k\}$}, then we can extend the single concept matrix into a list of concept matrices $[\mathbf{J}, \bar{\mathbf{J}}^{(p_1)}, \hat{\mathbf{J}}^{(p_1)},\dots, \bar{\mathbf{J}}^{(p_k)}, \hat{\mathbf{J}}^{(p_k)}]$.  

For each concept matrix $\check{\mathbf{J}}$ from the extended list, we consider its $i$-th row $\check{\mathbf{J}}_i=[\check{J}_{i,1},\check{J}_{i,2},\check{J}_{i,3}]$ (left-to-right) and extract patterns from $\check{\mathbf{J}}_i$ via the 4 modules from Section \ref{Sec:module}. Let $\mathcal{P}(\check{\mathbf{J}}_i)$ be the set of all such patterns, i.e.,
$
    \mathcal{P}(\check{\mathbf{J}}_i):=\mathcal{P}_{\text{intra}}(\check{\mathbf{J}}_i)\cup \mathcal{P}_{\text{inter}}(\check{\mathbf{J}}_i)\cup \mathcal{P}_{\text{comp}}(\check{\mathbf{J}}_i)\cup \mathcal{P}_{\text{binary}}(\check{\mathbf{J}}_i).
$

Finally, for row $i=1,2$, we define
\vspace*{-0.7em}
\begin{equation}
\label{eq:Pi}
    \mathcal{P}^\text{(all)}_i(\mathbf{J}):=\textstyle\bigcup_{\check{\mathbf{J}}} \big\{(K, \check{\mathbf{J}}) \mid K\in \mathcal{P}(\check{\mathbf{J}}_i)\big\},\\[-0.8em]
\end{equation}
where the union ranges over all concept matrices $\check{\mathbf{J}}$ in the extended list, i.e.
$\check{\mathbf{J}}\in[\mathbf{J}, \bar{\mathbf{J}}^{(p_1)}, \hat{\mathbf{J}}^{(p_1)},\dots, \bar{\mathbf{J}}^{(p_k)}, \hat{\mathbf{J}}^{(p_k)}]$.
Note that $\mathcal{P}^{\text{(all)}}_i(\mathbf{J})$ can be regarded as all the patterns extracted from the $i$-th row of the original concept matrix $\mathbf{J}$. If instead $\mathbf{J}=\smash{[J_{1,1}, \dots, J_{3,3}]}$ is a list containing 9 concepts, then we can define $\mathcal{P}^\text{(all)}_3(\mathbf{J})$ analogously. 

\vspace*{-0.2em}

\begin{algorithm}[ht]
\caption{Answer selection.}
\textbf{Inputs:} Concept matrix $\mathbf{J}=[J_{1,1}\dots J_{3,2}]$, and associated answer set $[ J_{\text{ans}1},\dots,J_{\text{ans}8}]$.
\begin{algorithmic}[1] 
\STATE Initialize $\text{comPattern}=[0,\dots,0]_{1\times8}$.
\STATE Compute $\mathcal{P}_{1,2}(\mathbf{J}):=\mathcal{P}_1^{\text{(all)}}(\mathbf{J})\cap\mathcal{P}_2^{\text{(all)}}(\mathbf{J})$. \hfill //\textit{ see }\eqref{eq:Pi}
\FOR{$i$ from 1 to 8}
\STATE $\mathbf{J} \gets [J_{1,1},\dots,J_{3,2}, J_{\text{ans}i}]$
\STATE Compute $ \mathcal{P}^{\text{(all)}}_3(\mathbf{J})$.
\STATE $\text{comPattern}[i]\leftarrow  | \mathcal{P}_{1,2}(\mathbf{J})\cap\mathcal{P}^{\text{(all)}}_3(\mathbf{J}) |$
\ENDFOR
\RETURN answer index $i=\argmax_{i'} \text{ comPattern}[i']$.
\end{algorithmic}
\label{alg:answerSelection}
\end{algorithm}

\vspace*{-1.0em}

\subsection{Solving RPMs}
\vspace*{-0.2em}
\subsubsection{Answer selection}
\vspace*{-0.4em}
In Section \ref{Sec:extractPattern}, we described how row-wise patterns can be extracted using the 4 invariance modules. 
Thus, a natural approach for answer selection is to determine which answer option, when inserted in place of the missing panel, would maximize the number of patterns that are common to all three rows. Consequently, answer selection is reduced to a simple optimization problem; see Algorithm \ref{alg:answerSelection}.

\vspace*{-0.8em}
\subsubsection{Answer generation} \label{subsec:GenerateAnswer}
\vspace*{-0.4em}
Since our algebraic machine reasoning framework is able to extract common patterns that are meaningful to humans, hidden in the raw RPM images, it provides a new way to generate answers without needing a given answer set. This is similar to a gifted human who is able to solve the RPM task, by first recognizing the patterns in the first two rows, then inferring what the missing panel should be. Intuitively, we are applying ``inverse" operations of the  $4$ invariance modules to generate the concept representing the missing panel; see Algorithm \ref{alg:answerGeneration} for an overview.

Briefly speaking, for a given RPM concept matrix $\mathbf{J}$, we first compute the common patterns among the first two rows via $\mathcal{P}_{1,2}(\mathbf{J}):=\mathcal{P}_1^{\text{(all)}}(\mathbf{J})\cap\mathcal{P}_2^{\text{(all)}}(\mathbf{J})$; see \eqref{eq:Pi}. Each element in $\mathcal{P}_{1,2}(\mathbf{J})$ is a pair $(K,\check{\mathbf{J}})$, where $K$ is a common pattern (for rows 1 and 2) specific to one attribute, and $\check{\mathbf{J}}$ is the corresponding concept matrix. (This represents the \textit{difficult} step of pattern discovery by a gifted human.)
Then, we go through all common patterns to compute the attribute values for the missing $9$th panel. (This represents a \textit{routine} consistency check of the discovered patterns; see Appendix \ref{AppendixSubsec:inverseInvarianceModule} for full algorithmic details, and \ref{AppendixSubsec:examples} for an example.)

In general, when integrating all the attribute values for $J_{3,3}$ derived from the patterns in $\mathcal{P}_{1,2}(\mathbf{J})$, it is possible that entities (i) have multiple possible values for a single attribute; or (ii) have missing attribute values. Case (i) occurs when there are multiple patterns extracted for a single attribute, while case (ii) occurs when there are no non-conflicting patterns for this attribute. For either case, we randomly select an attribute value from the possible values.

\vspace*{-0.4em}

\begin{algorithm}[ht]
\caption{Answer generation.}
\textbf{Inputs:} Concept matrix $\mathbf{J}=[J_{1,1}\dots J_{3,2}]$. 
\begin{algorithmic}[1] 
\FOR{$(K,\check{\mathbf{J}})\in\mathcal{P}_1^{\text{(all)}}(\mathbf{J})\cap\mathcal{P}_2^{\text{(all)}}(\mathbf{J})$} \hfill //\textit{ see }\eqref{eq:Pi}
\IF{ $[\check{J}_{3,1},\check{J}_{3,2}]$ does not conflict with pattern $K$} 
\STATE Compute attribute value for $\check{J}_{3,3}$ using pattern $K$.
\ENDIF
\ENDFOR
\STATE Collect all the above attribute values for $J_{3,3}$.
\WHILE{$\nexists$ unique value for some attribute of an entity} 
\STATE Randomly choose one valid attribute value.
\ENDWHILE
\STATE Generate ideal $J_{3,3}\subseteq R$.
\RETURN $J_{3,3}$ and the corresponding image.
\end{algorithmic}
\label{alg:answerGeneration}
\end{algorithm}

\vspace*{-0.8em}

\begin{table*}[ht]	
	\centering
		\begin{small}
			\begin{tabular}{llllllllll}
				\toprule
				& Method  &  \multicolumn{1}{c}{Avg. Acc.} & \multicolumn{1}{c}{Center} &  \multicolumn{1}{c}{2$\times$2G} &  \multicolumn{1}{c}{3$\times$3G} &  \multicolumn{1}{c}{O-IC} &  \multicolumn{1}{c}{O-IG} &  \multicolumn{1}{c}{L-R} &  \multicolumn{1}{c}{U-D} \\
				\midrule
				1 & LSTM~\cite{zhang2019raven} &  18.9 / 13.1 & 26.2 / 13.2 & 16.7 / 14.1 & 15.1 / 13.7 & 21.9 / 12.2 & 21.1 / 13.0 & 14.6 / 12.8 & 16.5 / 12.4 \\
				2 & WReN~\cite{santoro2018measuring} & 23.8 / 34.0 & 29.4 / 58.4 & 26.8 / 38.9 & 23.5 / 37.7 & 22.5 / 38.8 & 21.5 / 22.6 & 21.9 / 21.6 & 21.4 / 19.7 \\	
				3 & ResNet~\cite{zhang2019raven} &  40.3 / 53.4 & 44.7 / 52.8 & 29.3 / 41.9 & 27.9 / 44.3 & 46.2 / 63.2 & 35.8 / 53.1 & 51.2 / 58.8 & 47.4 / 60.2 \\
				4 & ResNet+DRT~\cite{zhang2019raven} &  40.4 / 59.6 & 46.5 / 58.1 & 28.8 / 46.5 & 27.3 / 50.4 & 46.0 / 69.1 & 34.2 / 60.1 & 50.1 / 65.8 & 49.8 / 67.1 \\
				5 & LEN~\cite{ZhengDistracting2019} & 41.4 / 72.9 & 56.4 / 80.2 & 31.7 / 57.5 & 29.7 / 62.1 & 52.1 / 84.4 & 31.7 / 71.5 & 44.2 / 73.5 & 44.2 / 81.2 \\
				6 & CoPINet~\cite{zhang2019learning} & 46.1 / 91.4 & 54.4 / 95.1 & 36.8 / 77.5 & 31.9 / 78.9 & 52.2 / 98.5 & 42.8 / 91.4 & 51.9 / 99.1 & 52.5 / 99.7\\
				7 & DCNet \cite{zhuo2021effective} & 49.4 / \textbf{93.6} & 57.8 / 97.8 & 34.1 / 81.7 & 35.5 / 86.7 & 57.0 / \textbf{99.0} & 42.9 / \textbf{91.5} & 58.5 / \textbf{99.8} & 60.0 / \textbf{99.8} \\	
				8 & NCD \cite{UnsupMohan2021} & 48.2 / 37.0 & 60.0 / 45.5 & 31.2 / 35.5 & 30.0 / 39.5 & 62.4 / 40.3 & 39.0 / 30.0 & 58.9 / 34.9 & 57.2 / 33.4 \\
				9 & SRAN \cite{hu2021stratified} & 60.8 / - & 78.2 / - & 50.1 / - & 42.4 / - & 68.2 / - & 46.3 / - & 70.1 / - & 70.3 / - \\
				10 & PrAE \cite{zhang2021abstract} &  77.0 / 65.0 &  90.5 / 76.5 &   85.4 / 78.6 &   45.6 / 28.6  &  63.5 / 48.1 &  60.7 / 42.6 & 96.3 / 90.1 & 97.4 / 90.9\\
			11 & Our Method   & \textbf{93.2} / 92.9 & \textbf{99.5} / \textbf{98.8} & \textbf{89.6} / \textbf{91.9} & \textbf{89.7} / \textbf{93.1} 
    & \textbf{99.6} / 98.2 & \textbf{74.7} / 70.1 & \textbf{99.7} / 99.2 & \textbf{99.5} / 99.1 \\
				\midrule
				& Human \cite{zhang2019raven}  & \multicolumn{1}{r}{- / 84.4} & \multicolumn{1}{r}{- / 95.5} & \multicolumn{1}{r}{- / 81.8} & \multicolumn{1}{r}{- / 79.6} & \multicolumn{1}{r}{- / 86.4} & \multicolumn{1}{r}{- / 81.8} & \multicolumn{1}{r}{- / 86.4} & \multicolumn{1}{r}{- / 81.8}\\
				\bottomrule %
			\end{tabular} %
		\end{small} %
	\centering %
 \caption{Performance on I-RAVEN/RAVEN. We report mean accuracy, and the accuracies for all configurations: \texttt{Center}, \texttt{2x2Grid}, \texttt{3x3Grid}, \texttt{Out-InCenter}, \texttt{Out-InGrid}, \texttt{Left-Right}, and \texttt{Up-Down}.}
 \label{tbl:accuracyother}
\end{table*}

\vspace*{-0.2em}
\section{Discussion}
\label{Section:Discussion}
\vspace*{-0.2em}

Algebraic machine reasoning provides a fundamentally new paradigm for machine reasoning beyond numerical computation.
Abstract notions in reasoning tasks are encoded very concretely as ideals, which are computable algebraic objects. We treat ideals as ``actual objects of study'', and we do not require numerical values to be assigned to them. This means our framework is capable of reasoning on more qualitative or abstract notions that do not naturally have associated numerical values.
Novel problem-solving, such as the discovery of new abstract patterns from observations, is realized concretely as computations on ideals (e.g. computing the primary decompositions of ideals). 
In particular, we are \emph{not} solving a system of polynomial equations, in contrast to existing applications of algebra in AI (cf. Section \ref{subsec:AlgebraicMethodsAI}). 
Variables (or primitive instances) are \emph{not} assigned values.
We do \emph{not} evaluate polynomials at input values.

Theory-wise, our proposed approach breaks new ground. We established a new connection between machine reasoning and commutative algebra, two areas that were completely unrelated previously. There is over a century's worth of very deep results in commutative algebra that have not been tapped. Could algebraic methods be the key to tackling the long-standing fundamental questions in machine reasoning? It was only much more recently in 2014 that L\'eon Bottou~\cite{bottou2014machine} suggested that humans should ``build reasoning capabilities from the ground up'', and he speculated that the missing ingredient could be an algebraic approach.

\vspace*{-0.1em}

Why use ideals to represent concepts? Why not use sets? Why not use symbolic expressions, e.g. polynomials? Intuitively, we think of a concept as an ``umbrella term'' consisting of multiple (potentially infinitely many) instances of the concept. 
Treating concepts as merely sets of instances is inadequate in capturing the expressiveness of human reasoning.
A set-theoretic representation system with finitely many ``primitive sets'' can only have finitely many possible sets in total. In contrast, we proved that we can construct infinitely many concepts from only \emph{finitely many} primitive concepts (Theorem \ref{Theo:PrimaryDecompProperConcept}). This agrees with our intuition that humans are able to express infinitely many concepts from only finitely many primitive concepts. The main reason is that the ``richer'' algebraic structure of ideals allows for significantly more operations on ideals, beyond set-theoretic operations. See Appendix \ref{AppendixSubsec:ConceptsIntuition} for further discussion.

Why is our algebraic method fundamentally different from logic-based methods, e.g. those based on logic programming?
At the heart of logic-based reasoning is the idea that reasoning can be realized concretely as the resolution (or inverse resolution) of logical expressions. Inherent in this idea is the notion of \emph{satisfiability}; cf. \cite{jaffar1987constraint}. Intuitively, we have a logical expression, usually expressed in a canonical normal form, and we want to assign truth values (true or false) to literals in the logical expression, so that the entire expression is satisfied (i.e. truth value is ``true''); see Appendix \ref{subsubsec:PDnoAnalogInLogic} for more discussion. In fact, much of the exciting progress in automated theorem proving \cite{abdelaziz2022learning, irving2016deepmath,lample2020deep,li2021isarstep, wu2021int, zombori2020prolog} is based on logic-based reasoning. 

In contrast, algebraic machine reasoning builds upon computational algebra and computer algebra systems. At the heart of our algebraic approach is the idea that reasoning can be realized concretely as solving computational problems in algebra. Crucially, there is no notion of satisfiability. We do not assign truth values (or numerical values) to concepts in \smash{$R = \Bbbk[x_1, \dots, x_n]$}. In particular, although primitive concepts \smash{$\langle x_1\rangle, \dots, \langle x_n\rangle$} in $R$ correspond to the variables \smash{$x_1, \dots, x_n$}, we do not assign values to primitive concepts. Instead, ideals are treated as the ``actual objects of study'', and we reduce ``solving a reasoning task'' to ``solving non-numerical computational problems involving ideals''.
Moreoever, our framework can discover new patterns beyond the actual rules of the RPM task; see Section \ref{subsec:failureCases}.

In the RPM task, we have attribute concepts representing ``position'', ``number'', ``type'', ``size'', and ``color''; these are concepts that categorize the primitive instances according to their semantics, into what humans would call attributes.
Intuitively, an attribute concept combines certain primitive concepts together in a manner that is ``meaningful'' to the task. For example, $\langle x_{\text{white}}, x_{\text{gray}}, x_{\text{black}} \rangle$ is ``more meaningful'' than $\langle x_{\text{white}}, x_{\text{circle}}, x_{\text{large}} \rangle$ as a ``simpler'' or ``generalized'' concept, since we would treat $x_{\text{white}}, x_{\text{gray}}, x_{\text{black}}$  as instances of a single broader ``color'' concept.

Notice that the primitive concepts correspond precisely to the prediction classes of our object detection models. Such prediction classes are already implicitly identified by the available data. Consequently, our method is limited by what our perception modules can perceive. For other tasks, e.g. where text data is available, entity extraction methods can be used to identify primitive concepts.
Note also that our method requires prior knowledge, since there is no training step for the reasoning module. This limitation can be mitigated if we replace user-defined functions on concepts with trainable functions optimized via deep learning.

In general, the identification of attribute concepts is task-specific, and the resulting reasoning performance would depend heavily on these identified attribute concepts. Effectively, our choice of attribute concepts would determine the \textit{inductive bias} of our reasoning framework: As we decompose a concept $J$ into ``simpler" concepts (i.e. primary components in $\text{pd}(J)$), only those ``simpler" concepts contained in attribute concepts are deemed ``meaningful''.
Concretely, let $J, J' \subsetneq R$ be concepts such that $\text{pd}(J) = \{J_1, \dots, J_k\}$ and $\text{pd}(J') = \{J'_1, \dots, J'_{\ell}\}$, i.e. $J, J'$ have minimal primary decompositions $J = J_1 \cap \dots \cap J_k$ and $J' = J'_1 \cap \dots \cap J'_{\ell}$, respectively. 
We can examine their primary components and extract out those primary components (between the two primary decompositions) that are contained in some common attribute concept. 
For example, if $A$ is an attribute concept of $R$ such that $J_1 \subseteq A$ and $J'_1 \subseteq A$, then $J$ and $J'$ share a ``common pattern'', represented by the attribute concept $A$.

\vspace*{-0.4em}

\section{Experiment results}
\vspace*{-0.3em}
To show the effectiveness of our framework, we conducted experiments on the RAVEN \cite{zhang2019raven} and I-RAVEN \mbox{datasets. In both datasets, RPMs are generated according to} 7 configurations. We trained our perception modules on 4200 images from I-RAVEN \cite{hu2021stratified} (600 from each configuration), and used them to predict attribute values of entities. The average accuracy of our perception modules is $96.24\%$. 
For both datasets, we tested on 2000 instances for each configuration.
Overall, our reasoning framework is fast (7 hours for 14000 instances on a 
16-core Gen11 Intel i7 CPU processor).
See Appendix \ref{AppendixSec:experiment} for full experiment details.

\vspace*{-0.1em}
\subsection{Comparison with other baselines}
Table \ref{tbl:accuracyother} compares the performance of our method with $10$ other baseline methods.
We use the accuracies on I-RAVEN reported in \cite{hu2021stratified, UnsupMohan2021} for methods 1-7, and
the accuracies on RAVEN reported in \cite{UnsupMohan2021, zhang2019raven} for methods 1-5.
All the other accuracies are obtained from the original papers.
As a reference, we also include the human performance on the RAVEN dataset (i.e. \emph{not} I-RAVEN) as reported in \cite{zhang2019raven}.

\subsection{Ambiguous instances and new patterns} \label{subsec:failureCases}
Although our method outperforms all baselines, some instances have multiple answer options that are assigned equal top scores by our framework. Most of these cases occur due to the discovery of (i) ``accidental" unintended rules (e.g. Fig. \ref{fig:ambiguity}); or (ii) new patterns beyond the actual rules in the dataset (e.g. Fig. \ref{fig:dis2}). Case (i) occurs because in the design of I-RAVEN, at most one rule is assigned to each attribute. 

Interestingly, case (ii) reveals that our framework is able to discover completely new patterns that are not originally designed as rules for I-RAVEN. In Fig. \ref{fig:dis2}, the new pattern discovered is arguably very natural to humans.

\begin{figure}[t!]
\centering
{\includegraphics[width=\linewidth]{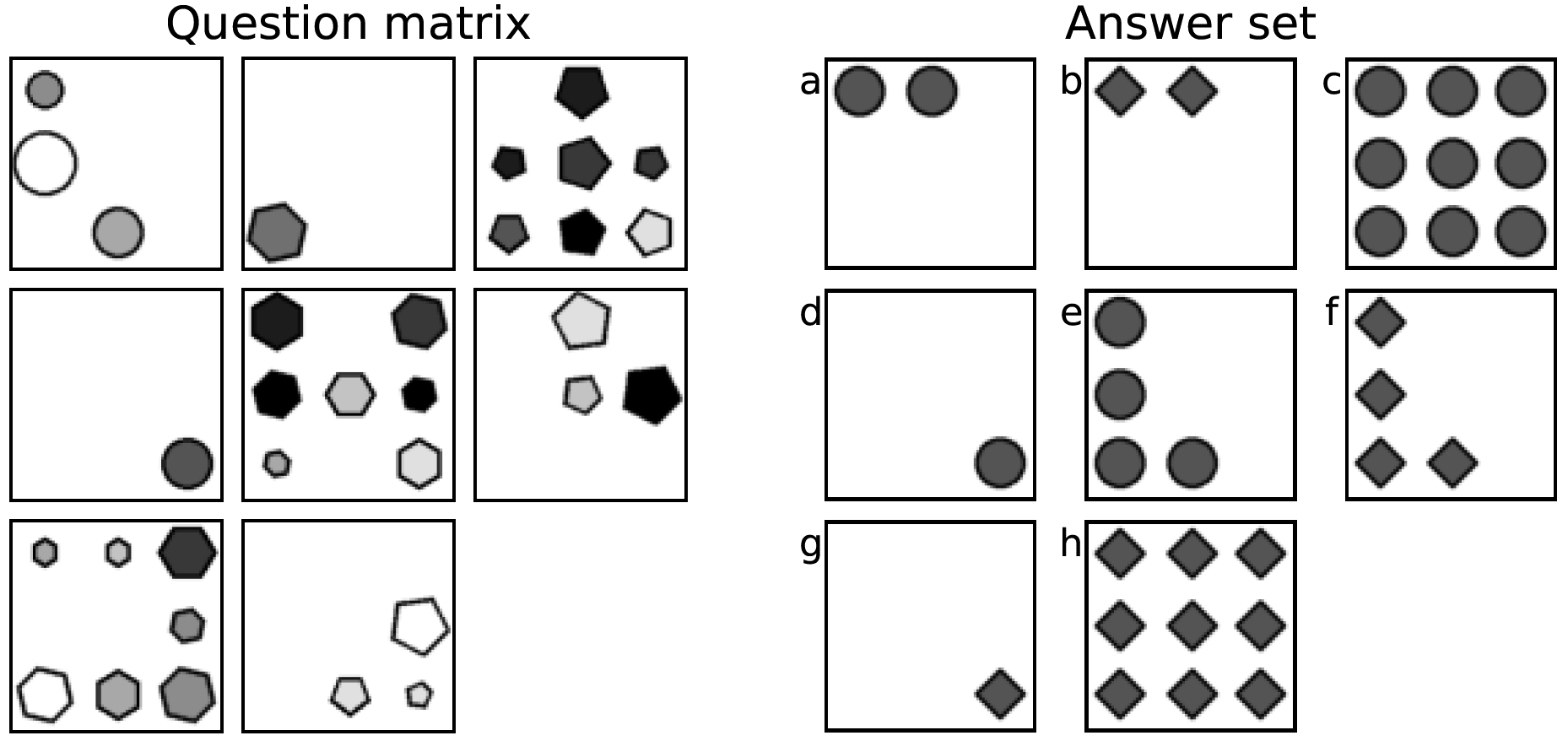}}
\caption{An example of an ambiguous RPM instance. The given answer is option \textbf{g}. For I-RAVEN, the type sequence (``circle”, ``hexagon”, ``pentagon”) in the first two rows follows a Progression rule with consecutively decreasing type indices, so $\textbf{g}$ could be a correct answer. (Remaining attribute values are determined by other patterns.) However, our framework assigns equal top scores to both options \textbf{d} and \textbf{g}, as a result of another inter-invariance pattern for type (the type set $\{$``circle”, ``hexagon”, ``pentagon”$\}$ is invariant across the rows). Thus, option \textbf{d} could also be correct.}
\label{fig:ambiguity}%
\end{figure}
\begin{figure}[t!]
\centering
{\includegraphics[width=\linewidth]{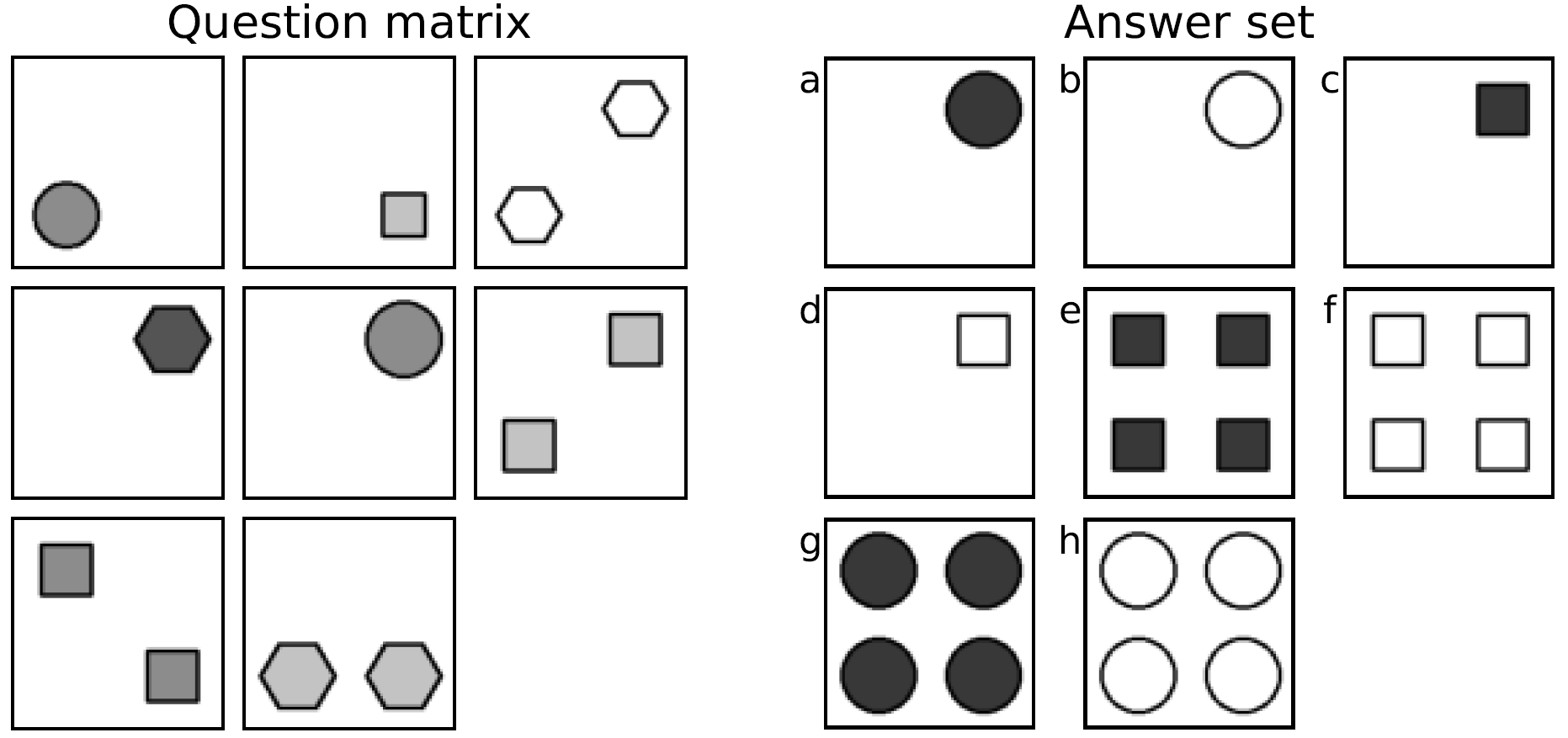} }
\caption{An example of an RPM instance with an unexpected new pattern. The given answer is option \textbf{h}. In each row, the number of entities in the first 2 panels sum up to the number of entities in the 3rd panel, so \textbf{h} could be correct.
However, our framework assigns equal top scores to both options \textbf{b} and \textbf{h}, as a result of a new inter-invariance pattern for number (informally, every panel has either 1 or 2 entities). Thus option \textbf{b} could also be correct.}
\label{fig:dis2}%
\end{figure}

\subsection{Evaluation of answer generation}
Every RPM instance is assumed to have a single correct answer from the given answer set. However, there are multiple other possible images that are also acceptable as correct answers. For example, images modified from the given correct answer, via random perturbations of those attributes that are not involved in any of the rules (e.g. entity angles in the I-RAVEN dataset), are also correct. All these distinct correct answers (images) can be encoded algebraically as the same concept, based on prior knowledge of which raw perceptual attributes are relevant for the RPM task. Hence, to evaluate the answer generation process proposed in Section \ref{subsec:GenerateAnswer}, we will directly evaluate the generated concepts. 

Let $J=\langle e_1,\dots,e_k \rangle$ and $J'=\langle e'_1,\dots,e'_{\ell} \rangle$ be concepts representing the ground truth answer and our generated answer, respectively. Here, each $e_i$ (or $e'_i$) is a monomial of the form $x_i^{(\text{pos})}x_i^{(\text{type})}x_i^{(\text{color})}x_i^{(\text{size})}$, and represents an entity described by $4$ attributes. Motivated by the well-known idea of Intersection over Union (IoU), we propose a new similarity measure between $J$ and $J'$. In order to define analogous notions of ``intersection" and ``union",
we first pair $e_i$ with $e'_j$ if $x_i^{(\text{pos})}= x'_j{}^{\!(\text{pos})}$ (i.e. same ``position'' values).
This pairing is well-defined, since the ``position" values of the entities in any panel are uniquely determined.
Hence we can group all entities in $J$ and $J'$ into 3 sets:
\vspace*{-0.7em}
\begin{align*}
    S_1&:=\{ (e_i,e'_j) \mid e_i\in J,e'_j\in J', x_i^{(\text{pos})} = x'_j{}^{\!(\text{pos})}\}; \\[-0.2em]
    S_2&:= \{ e_i\in J \mid \nexists e'_j \in J' \text{ such that }(e_i,e'_j)\in S_1 \};\\[-0.2em]
    S_3&:= \{ e'_j\in J' \mid \nexists e_i \in J \text{ such that }(e_i,e'_j)\in S_1\}.\\[-2.0em]
\end{align*}
We can interpret $S_1$ and $S_1\cup S_2\cup S_3$ as analogous notions of the ``intersection" and ``union" of $J$ and $J'$, respectively. Thus, we define our similarity measure as follows:
\vspace*{-0.6em}
\begin{align}
    \varphi(J,J')&:=\frac{\sum_{(e_i,e'_j)\in S_1}\phi(e_i,e'_j)}{|S_1|+|S_2|+|S_3|}; \\[-0.1em]
    \label{eq:similarity}\phi(e_i,e'_j)&:=\frac{1}{4}\smash{\sum_a}\mathbbm{1}(x_i^{(\text{a})}= x'_j{}^{\!(\text{a})});\\[-1.7em]
    \nonumber
\end{align}
where in \eqref{eq:similarity}, $a$ ranges over the 4 attributes in \{pos, type, color, size\}. Here, $\phi(e_i,e'_j)$ is the similarity score between $e_i$ and $e'_j$, measured by the proportion of common variables.

The overall average similarity score of the generated answers is $67.7\%$. Note that within a panel, some attribute values such as ``size", ``color" and ``position",  may be totally random for \texttt{2x2Grid}, \texttt{3x3Grid}, \texttt{Out-InGrid} (e.g. as shown in Fig. \ref{fig:ambiguity}). Hence, achieving high similarity scores for such cases would inherently require task-specific optimization and knowledge of how the data is generated. We assume neither. This could explain why our overall similarity score is lower than our answer selection accuracy.

For examples of generated images, see Appendix \ref{AppendixSubsec:generatedAnswers}.

\section{Conclusion}

Algebraic machine reasoning is a reasoning framework that is well-suited for abstract reasoning.
In its current form, we have used primary decompositions as a key algebraic operation to discover abstract patterns in the RPM task, via the invariance modules that we have specially designed to mimic human reasoning.
The idea that ``discovering common patterns'' can be realized concretely as ``computing primary decompositions'' is rather broad, and could potentially be applied to other inferential reasoning tasks.

More generally, our algebraic approach opens up new possibilities of tapping into the vast literature of commutative algebra and computational algebra. There are numerous algebraic operations on ideals (ideal quotients, radicals, saturation, etc.) and algebraic invariants (depth, height, etc.) that have not been explored in machine reasoning (or even in AI). Can we use them to tackle other reasoning tasks?

\begin{table*}[t!]
\centering
\begin{adjustbox}{width=\textwidth,center}
\parbox{\textwidth}{\centering\Large \textbf{Appendices for Abstract Visual Reasoning: An Algebraic Approach\\ for Solving Raven's Progressive Matrices}}
\end{adjustbox}
\end{table*}

\pagebreak
\appendix

\textbf{Outline.} As part of supplementary materials for this paper, we provide further details, organized into the following three appendices:
\setlist{nosep}
\begin{itemize}
\item Appendix \ref{AppendixSec:furtherAlgDetails} gives further algebraic details on algebraic machine reasoning.
\item Appendix \ref{AppendixSec:experiment} gives further experiment details for how we solved RPMs.
\item Appendix \ref{AppendixSec:discussion} gives further discussion on algebraic machine reasoning, including its potential societal impact and its differences from logic-based reasoning.
\end{itemize}

\section{Further Algebraic Details}
\label{AppendixSec:furtherAlgDetails}
This appendix serves as an in-depth elaboration of the algebraic ideas presented in Section \ref{Section:method}. 
For definitions and relevant terminology, we have kept this appendix as self-contained as possible. For proofs of ``standard'' algebraic results, we provide references to explicit proofs in textbooks wherever possible, and provide full proof details otherwise. For ease of reading, it would be helpful if the reader has at least some prior exposure to the basics of commutative algebra, e.g. at the level of undergraduate abstract algebra. For a gentle introduction to commutative algebra and computational algebra, we recommend \cite{cox2015ideals}. 
For a more detailed treatment of the subject, see \cite{Eisenbud:CommutativeAlgebraBook}.

Throughout the appendix, let $R := \Bbbk[x_1, \dots, x_n]$ be a polynomial ring on $n$ variables $x_1, \dots, x_n$, with coefficients in a field $\Bbbk$. In this paper, we assumed for simplicity that $\Bbbk = \mathbb{R}$; in fact, $\Bbbk$ can be chosen to be any field. In particular, $\Bbbk$ need not be an algebraically closed field.\footnote{When dealing with a system of polynomial equations in $n$ variables, where the coefficients of the polynomial equations are treated as values in $\Bbbk$, any solution to this system of polynomial equations is a point in $\Bbbk^n$. It is frequently assumed that $\Bbbk$ is algebraically closed, so that the analysis of the solution space is easier. Similarly, when considering the algebraic variety associated to such a system of polynomial equations, the analysis of the algebraic variety is easier when the base field $\Bbbk$ is algebraically closed. In our paper, we do not deal with or solve systems of polynomial equations, and we will not work with algebraic varieties. Hence, we will not assume that $\Bbbk$ is algebraically closed.} (Note that $\mathbb{C}$ is an algebraically closed field, while $\mathbb{R}$ is not.) This means that we allow $\Bbbk = \mathbb{R}$.  We also allow $\Bbbk$ to be a field with non-zero characteristic (e.g. a finite field $\mathbb{Z}/p\mathbb{Z}$ for some prime number $p$). The specific choice of field $\Bbbk$ will be irrelevant for the details discussed in this appendix, since none of the algebraic computations required for reasoning involve numerical computations on the coefficients of the terms of polynomials. For concreteness, the reader may choose to assume henceforth that $\Bbbk = \mathbb{R}$, without any loss of generality. 

For our algebraic notation, we use ``$\langle g_1, \dots, g_k\rangle$'' to denote an ideal generated by $g_1, \dots, g_k$, in contrast to some other authors who use the notation ``$(g_1, \dots, g_k)$'' instead to refer to ideals. For us, we reserve the use of round brackets to denote tuples, e.g. ``$(g_1, \dots, g_k)$'' is a $k$-tuple. All other algebraic notation we use is consistent with the present-day notation used in the commutative algebra literature.

The rest of Appendix \ref{AppendixSec:furtherAlgDetails} is organized as follows:
\setlist{nosep}
\begin{itemize}
\item Appendix \ref{AppendixSubsec:concepts,ideals,operation} gives further details and relevant results concerning concepts (i.e. monomial ideals), including algebraic operations on ideals, and with a focus on the special case of monomial ideals.
\item Appendix \ref{AppendixSubsec:grobnerBasis} gives a detailed treatment of Gr\"{o}bner basis theory, the key ``workhorse'' underlying most of the algorithms in computational algebra and commutative algebra.
\item Appendix \ref{AppendixSubsec:primaryIdeals,pd} gives more details on primary ideals and primary decompositions, as well as a discussion on \textit{how primary decompositions relate to the inductive bias for algebraic machine reasoning}.
\item Appendix \ref{AppendixSubsec:ConceptsIntuition} completes the proof of Theorem \ref{Theo:PrimaryDecompProperConcept} from Section \ref{Subsubsection:concepts}, and gives details on \textit{why defining concepts as monomial ideals captures the expressiveness of concepts in human reasoning.}
\end{itemize}

\subsection{Concepts, monomial ideals, ideal operations}
\label{AppendixSubsec:concepts,ideals,operation}
Let $\mathbb{N}$ be the set of non-negative integers. For each $\boldsymbol{\alpha} = (\alpha_1, \dots, \alpha_n) \in \mathbb{N}^n$, let $\mathbf{x}^{\boldsymbol{\alpha}}$ denote the monomial $x_1^{\alpha_1}x_2^{\alpha_2}\dots x_n^{\alpha_n}$ in $R$. 
By definition, a \textit{monomial} in $R$ is a polynomial of the form $\mathbf{x}^{\boldsymbol{\alpha}}$, i.e. the leading coefficient of a monomial is always $1$, the multiplicative identity of the field $\Bbbk$. 
The \textit{degree} of $\mathbf{x}^{\boldsymbol{\alpha}}$, which we denote by $\deg(\mathbf{x}^{\boldsymbol{\alpha}})$, is the sum $\alpha_1 + \dots + \alpha_n$. For example, $\deg(x_1^2x_2^3x_3^7) = 12$.

\subsubsection{Algebraic properties of concepts}
Let $\mathcal{M}_R$ be the set of all monomials in $R$. We begin with the observation that the polynomial ring $R$ is a $\Bbbk$-vector space with basis $\mathcal{M}_R$. This means that every polynomial $p$ in $R$ can be written as $p = \sum_{m\in \mathcal{M}_R}a_m\cdot m$, where each $a_m$ is a uniquely determined scalar in $\Bbbk$, called the \textit{coefficient} of the monomial $m$ in $p$. A monomial $m\in \mathcal{M}_R$ is called a \textit{monomial of $p$} if its coefficient $a_m$ is non-zero.

\begin{proposition}\label{prop:polynomialAndTerms}{\cite[Chap. 2.4, Lemma 3]{cox2015ideals}} Let $p$ be a polynomial of $R$, and let $J$ be a monomial ideal of $R$. Then $p$ is an element of $J$ if and only if every monomial $m$ of $p$ is an element of $J$.
\end{proposition}

Although $\mathcal{M}_R$ is infinite (since it contains monomials of all degrees), it is a well-known algebraic fact that every ideal $I$ of $R$ has a \emph{finite} generating set; this is famously known as Hilbert's basis theorem (see \cite[Chap. 2.5, Thm. 4]{cox2015ideals} or \cite[Thm. 1.2]{Eisenbud:CommutativeAlgebraBook} for a proof). For monomial ideals, we have the following stronger result:

\begin{proposition}\label{prop:Rmonomialbasis}{\cite[Lemma 1.2]{book:CombCA}}
Every monomial ideal of $R$ has a \textbf{unique} minimal generating set consisting of finitely many monomial generators.
\end{proposition}

To make sense of Proposition \ref{prop:Rmonomialbasis}, note that a generating set $\mathcal{G}$ for an ideal $J$ of $R$ is called \textit{minimal} if no strictly smaller subset $\mathcal{G}' \subsetneq \mathcal{G}$ generates $J$. This means if we are given a monomial ideal $J$ of $R$, and any subset $\mathcal{M} \subseteq \mathcal{M}_R$ of monomials that generates $J$, i.e. $J = \langle \mathcal{M} \rangle$, then Proposition \ref{prop:Rmonomialbasis} tells us that we can always find a minimal subset $\mathcal{M}' \subseteq \mathcal{M}$ (with respect to set inclusion) that generates the same ideal $J$, such that this minimal subset $\mathcal{M}'$ is uniquely determined, independent of our choice of $\mathcal{M}$. This unique minimal generating set of monomials shall be denoted by $\mingen(J)$. 

Given any monomial ideal $J$ of $R$, a monomial contained in $\mingen(J)$ is called a \textit{minimal monomial generator} of $J$. Since a concept is defined to be a monomial ideal of $R$, we can analogously define a \textit{minimal instance} of $J$ to be a minimal monomial generator of $J$. Thus, a concept is generated by all of its minimal instances.

Recall that there are three basic operations on ideals: sums, products, and intersections. Let $J_1 = \langle g_1, \dots, g_k\rangle$ and $J_2 = \langle h_1, \dots, h_{\ell}\rangle$ be any two ideals of $R$. Then the sum $J_1 + J_2$, product $J_1J_2$ and intersection $J_1\cap J_2$ are ideals defined as follows:
\begin{align*}
J_1 + J_2 &:= \langle g_1, \dots, g_k, h_1, \dots, h_{\ell}\rangle;\\
J_1J_2 &:= \langle \{g_ih_j | 1\leq i\leq k, 1\leq j\leq \ell\}\rangle;\\
J_1 \cap J_2 &:= \{r\in R: r\in J_1 \text{ and }r\in J_2\}.
\end{align*}
In particular, $J_1 \cap J_2$ is an ideal, since for any polynomials $r_1, \dots, r_s$ in $J_1 \cap J_2$, the polynomial combination $r_1p_1 + \dots + r_sq_s$ (for any polynomials $p_1, \dots, p_s$ in $R$) is by definition contained in both $J_1$ and $J_2$.

A useful property of these basis operations on ideals is that monomial ideals remain as monomial ideals under such operations:

\begin{proposition}\label{prop:MonIdealsPreserved}
If $J_1$ and $J_2$ are concepts of $R$, then $J_1 + J_2$, $J_1\cap J_2$ and $J_1J_2$ are also concepts of $R$. Furthermore, we always have $J_1J_2 \subseteq J_1 \cap J_2$.
\end{proposition}

\begin{proof}
Clearly, $J_1 + J_2$ and $J_1J_2$ are monomial ideals, since the generators specified in their definitions are monomials. Also, for each $1\leq i\leq k$, $1\leq j\leq \ell$, the element $g_ih_j$ is by definition contained in both $J_1$ and $J_2$, which means that every minimal monomial generator of $J_1J_2$ is an element of $J_1 \cap J_2$, thus we get $J_1J_2 \subseteq J_1\cap J_2$.

Finally, we shall prove that $J_1 \cap J_2$ is a monomial ideal. Let $\mathcal{M}' := \mathcal{M}_R \cap J_1 \cap J_2$, and define the monomial ideal $J' := \langle \mathcal{M}'\rangle \subseteq R$. By definition, $\mathcal{M}' \subseteq J_1 \cap J_2$, hence $J' \subseteq J_1 \cap J_2$. To show that the converse $J' \supseteq J_1 \cap J_2$ holds, consider an arbitrary element $p\in J_1 \cap J_2$, and write $p = \sum_{m\in \mathcal{M}_R}a_m\cdot m$, where each $a_m \in \Bbbk$ is the coefficient of the monomial $m$. Note that $p\in J_1$ and $p\in J_2$.
Since $J_1$ and $J_2$ are both monomial ideals, Proposition \ref{prop:polynomialAndTerms} implies that every monomial $m$ of $p$ (i.e. with coefficient $a_m \neq 0$) is an element of $J_1\cap J_2$, and so must also be contained in $\mathcal{M}' = \mathcal{M} \cap J_1 \cap J_2$. Thus, $p\in \langle \mathcal{M}'\rangle = J'$.
\end{proof}

\subsubsection{A hierarchy of different kinds of concepts}
\label{AppendixSubsubsec:HierarchyConcepts}
Recall that the first stage of algebraic machine reasoning is algebraic representation. For the RPM task, we extracted attribute information from raw RPM images, and encoded each RPM panel algebraically as a concept. These concepts that we have used for our algebraic representation of RPM panels are particularly nice: Their minimal monomial generators are squarefree. Recall that a monomial $\mathbf{x}^{\boldsymbol{\alpha}}$ in $R$ (for some $\boldsymbol{\alpha} = (\alpha_1, \dots, \alpha_n) \in \mathbb{N}^n$) is called \textit{squarefree} if every exponent $\alpha_i$ is either $0$ or $1$, or equivalently, if $\mathbf{x}^{\boldsymbol{\alpha}} \not\in \langle x_1^2, \dots, x_n^2\rangle$. (Informally, a monomial $m$ is squarefree if the variables appearing in $m$ do not have ``higher powers'', i.e. none of the variables have exponents $\geq 2$.) Monomial ideals generated by squarefree monomials are called \textit{squarefree monomial ideals}, and such ideals are heavily studied in commutative algebra for their rich properties \cite{book:CombCA,book:StanleyGreenBook}. Hence, we are motivated to define the following notion of ``basic'' concepts.

\begin{definition}
A concept $J$ in $R$ is called \textit{basic} if $J\neq R$, and if all of its minimal instances (i.e. minimal monomial generators) are squarefree.
\end{definition}
We shall later prove in Theorem \ref{thm:basicConceptDecomp} that basic concepts have a particularly nice structure for extracting patterns.

We shall also introduce the notion of ``simple'' concepts. Consider the ideal $J_1 := \langle x_{\text{black}}x_{\text{square}}, x_{\text{white}}x_{\text{circle}}\rangle$, which is a basic concept that represents ``black square or white circle''. Such a concept is contained in the intuitively simpler concept $J_2 := \langle x_{\text{black}}, x_{\text{white}}\rangle$ representing ``black or white''. Note that $J_2$ is generated by two primitive instances $x_{\text{black}}$ and $x_{\text{white}}$. To capture this intuition that ``black or white'' is simpler than ``black square or white circle'', we define a concept $J$ in $R$ to be \textit{simple} if $J$ is generated by a non-empty set of primitive instances in $R$, i.e. a non-empty subset of $\{x_1, \dots, x_n\}$. Equivalently, a concept is \textit{simple} if it is the sum of primitive concepts. Notice that all simple concepts are basic, since a primitive instance is squarefree. As we shall see later in Appendix \ref{AppendixSubsubsec:inductiveBias}, every basic concept can be ``decomposed'' into multiple simple concepts.

To summarize, we have the following hierarchy of different kinds of concepts:
\[\Big\{\text{concepts}\Big\} \supseteq \Big\{\parbox{3.8em}{\centering basic\\ concepts}\Big\} \supseteq \Big\{\parbox{3.8em}{\centering simple\\ concepts}\Big\} \supseteq \Big\{\parbox{3.8em}{\centering primitive\\ concepts}\Big\}.\]

\subsection{Gr\"{o}bner basis theory}
\label{AppendixSubsec:grobnerBasis}

In this subsection, we review basic terminology and several elementary results in Gr\"{o}bner basis theory. For a quick introduction to what a Gr\"{o}bner basis of an ideal is, see \cite{Sturmfels2005:WhatIsGrobnerBasis}. 

Why is Gr\"{o}bner basis theory important? In essence, if we want to \emph{compute} the answer to a computational problem involving ideals, we typically need, as a key initial step, to compute the Gr\"{o}bner bases of the input ideals. Subsequent computational steps would usually involve working directly with the computed Gr\"{o}bner bases, rather than the originally given generating sets for the input ideals. To avoid the reader thinking that Gr\"{o}bner basis computations are for ``more advanced'' operations on ideals, we highlight here that computing the intersections of ideals already relies on Gr\"{o}bner basis computations: Given ideals $J_1, J_2$ of $R$ with generating sets $\mathcal{G}_1, \mathcal{G}_2$ respectively, it is already not trivial to compute a generating set for the intersection $J_1\cap J_2$.

\subsubsection{What is a Gr\"{o}bner basis?}
Recall that $R$ is a $\Bbbk$-vector space with basis $\mathcal{M}_R$, where $\mathcal{M}_R$ denotes the set of monomials in $R$. The starting point in Gr\"{o}bner basis theory is the sorting of the set $\mathcal{M}_R$.

\begin{definition}
A \textit{monomial order} on $R$ is a well-order on $\mathcal{M}_R$, denoted by $<$, such that for all monomials $m_1, m_2, m$ in $\mathcal{M}_R$ satisfying $m\neq 1$, the following condition holds:
\begin{center}
If $m_1 > m_2$, then $mm_1 > mm_2 > m_2$. 
\end{center}    
\end{definition}

Given a monomial order $<$, and any polynomial $p$ in $R$, we can always write \mbox{$p = c_1m_1 + \dots + c_km_k$} for some non-zero scalars $c_1, \dots, c_k$ in $\Bbbk$, and some monomials $m_1, \dots, m_k$ in $\mathcal{M}_R$ that are sorted in descending order $m_1 > \dots > m_k$ with respect to the monomial order $<$. These scalars and monomials are uniquely determined given $p$ and $<$. Recall that the monomials $m_1, \dots, m_k$ are precisely all the monomials \emph{of} $p$, and each scalar $c_i$ is the coefficient of $m_i$ in $p$. The \textit{initial monomial} of $p$ (with respect to $<$), denoted by $\inid_{<}(p)$, is the largest monomial $m_1$. By default, we define $\inid_{<}(0) = 0$.

Suppose $\mathcal{A} = \{p_1, \dots, p_k\}$ is a finite set of polynomials in $R$. Write $p = \big(\sum_{i=1}^k f_ip_i\big) + q$ where $f_1, \dots, f_k, q$ are polynomials in $R$, such that $\inid_{<}(p) \geq \inid_{<}(f_ip_i)$ for all $1\leq i\leq k$, and such that none of the monomials of $q$ are divisible by any monomial contained in $\{\inid_{<}(p_1), \dots, \inid_{<}(p_k)\}$. Such an expression always exists and is called a \textit{standard expression} of $p$ with respect to $\mathcal{A}$. Any such polynomial $q$ is called a \textit{remainder} of $p$ with respect to $\mathcal{A}$. If $0$ is a remainder of $p$ with respect to $\mathcal{A}$, then we say that $p$ \textit{reduces to zero} with respect to $\mathcal{A}$. Note that standard expressions and remainders of $p$ (with respect to $\mathcal{A}$) are in general not unique. 

Let $f,g\in R$ be polynomials that are not both the zero polynomial. The \textit{$\mathcal{S}$-pair} of $f$ and $g$ (with respect to $<$) is the polynomial
\begin{equation*}
\mathcal{S}(f,g) := \frac{c_g\inid_{<}(g)}{\gcd(\inid_{<}(f), \inid_{<}(g))} f - \frac{c_f\inid_{<}(f)}{\gcd(\inid_{<}(f), \inid_{<}(g))} g,
\end{equation*}
where $\gcd(\inid_{<}(f), \inid_{<}(g))$ denotes the greatest common divisor of $\inid_{<}(f)$ and $\inid_{<}(g)$, and $c_f$ (resp. $c_g$) is the coefficient of $\inid_{<}(f)$ (resp. $\inid_{<}(g)$) in $f$ (resp. $g$). 
We use the convention that $\gcd(f,0) = f$ for all non-zero polynomials $f\in R$. By default, define $\mathcal{S}(0,0) = 0$. 
It is straightforward to check that if $\gcd(\inid_{<}(f), \inid_{<}(g)) = 1$, then $\mathcal{S}(f,g)$ always reduces to zero with respect to $\{f, g\}$.

For any ideal $I$ of $R$, the monomial ideal $\inid_{<}(I) := \langle \{\inid_{<}(f)| f\in I\}\rangle$, generated by the initial monomials of all elements in $I$, is called the \textit{initial ideal} of $I$ (with respect to $<$). A \textit{Gr\"{o}bner basis} for $I$ (with respect to $<$) is a \textbf{finite} set $\mathcal{G}$ of elements in $I$, such that $\inid_{<}(I)$ is generated by the set $\{\inid_{<}(f)| f\in \mathcal{G}\}$. Gr\"{o}bner bases for $I$ (with respect to any monomial order) always exist, and every Gr\"{o}bner basis for $I$ must generate $I$. If $\mathcal{G}$ is a Gr\"{o}bner basis for $I$, then the remainder of $p$ with respect to $\mathcal{G}$ is unique. In particular, $p\in I$ if and only if $p$ reduces to zero with respect to $\mathcal{G}$.

Consider any finite generating set $\mathcal{A}$ for the ideal $I$. \textit{Buchberger's criterion} says that $\mathcal{A}$ is a Gr\"{o}bner basis for $I$ (with respect to some given monomial order $<$) if and only if $\mathcal{S}(f,g)$ reduces to zero for every $f,g\in \mathcal{A}$. This criterion yields \textit{Buchberger's algorithm}, which is an algorithm to compute a Gr\"{o}bner basis $\mathcal{G}$ for any ideal $I$ of $S$, given a finite generating set $\mathcal{A}$ for $I$ as input. Buchberger's algorithm is described as follows: Starting with the input generating set $\mathcal{A}$, compute the remainders of the $\mathcal{S}$-pairs $\mathcal{S}(f,g)$ with respect to $\mathcal{A}$, for all generators $f, g \in \mathcal{A}$; if any such remainder $q$ is non-zero, then insert the polynomial $q$ into $\mathcal{A}$ and repeat the process. This process must terminate after finitely many steps, and the final $\mathcal{A}$ obtained after the process terminates is a Gr\"{o}bner basis for $I$, which contains the initially given generating set $\mathcal{A}$. Note that by definition, all $\mathcal{S}$-pairs of pairs of elements in the final Gr\"{o}bner basis $\mathcal{A}$ must reduce to zero with respect to $\mathcal{A}$.

\subsubsection{Computational subroutines in algebra}
\label{subsubsec:ComputationalSubroutines}
Computational problems in algebra can be solved using computer algebra systems. 
Well-known computer algebra systems (e.g. Maple, Magma, Mathematica, MATLAB; cf. \cite[Appendix C]{cox2015ideals}) are frequently used to solve systems of polynomial equations, while more specialized systems (e.g. CoCoA~\cite{CoCoA}, GAP~\cite{GAP}, Macaulay2~\cite{M2}, SageMath~\cite{sagemath}, Singular~\cite{Singular}) are more commonly used by algebraists to solve computational problems where ideals are the ``actual objects of study'', i.e. not involving the assignment of numerical values to variables.

Gr\"{o}bner bases are always defined with respect to some given monomial order $<$. In many computer algebra systems (e.g. CoCoA~\cite{CoCoA}, Macaulay2~\cite{M2}, SageMath~\cite{sagemath}, Singular~\cite{Singular}), the default monomial order for Gr\"{o}bner basis computations is the graded reverse-lexicographic order, also known as the degree reverse-lexicographic order.

\begin{definition}
The \textit{graded reverse-lexicographic order} (grevlex order) on $R$, which we denote by $<_{r\ell}$, is a monomial order on $R$, such that for any distinct monomials $\mathbf{x}^{\boldsymbol{\alpha}}$, $\mathbf{x}^{\boldsymbol{\beta}}$ in $R$, where $\boldsymbol{\alpha} = (\alpha_1, \dots, \alpha_n)$ and $\boldsymbol{\beta} = (\beta_1, \dots, \beta_n)$, we have $\mathbf{x}^{\boldsymbol{\alpha}} >_{r\ell} \mathbf{x}^{\boldsymbol{\beta}}$ if and only if either $\deg(\mathbf{x}^{\boldsymbol{\alpha}}) > \deg(\mathbf{x}^{\boldsymbol{\beta}})$; or $\deg(\mathbf{x}^{\boldsymbol{\alpha}}) = \deg(\mathbf{x}^{\boldsymbol{\beta}})$, and $\alpha_i < \beta_i$ for the largest $1\leq i\leq n$ such that $\alpha_i \neq \beta_i$.
\end{definition}

For example, if $n = 3$ (i.e. $R = \Bbbk[x_1, x_2, x_3]$), then the degree $3$ monomials in $R$, arranged in the grevlex order, is given as follows:
\begin{align*}
&x_1^3 >_{r\ell} x_1^2x_2 >_{r\ell} x_1x_2^2 >_{r\ell} x_2^3 >_{r\ell} x_1^2x_3 >_{r\ell} x_1x_2x_3\\
&\ \ \ \ >_{r\ell} x_2^2x_3 >_{r\ell} x_1x_3^2 >_{r\ell} x_2x_3^2 >_{r\ell} x_3^3.
\end{align*}
Notice that in our definition for the grevlex order, the degree $1$ monomials (which are precisely the $n$ variables in $R$) are arranged as $x_1 >_{r\ell} x_2 >_{r\ell} > \dots >_{r\ell} x_n$. Hence, by permuting these variables, there are $n!$ distinct grevlex orders that can be defined on a polynomial ring with $n$ variables. In general, we would need to specify a linear order on the $n$ variables in $R$, then define the grevlex order on top of the specified linear order for these degree $1$ monomials. For this paper, we assume implicitly and without loss of generality that $x_1 > x_2 > \dots > x_n$ for all monomial orders $<$ on $R$.

\textbf{Ideal membership.} Let $p\in R$ be a polynomial, let $I$ be an ideal of $R$, and suppose that $\mathcal{G}$ is a Gr\"{o}bner basis for $I$. How do we decide if $p$ is contained in $I$? We can compute the remainder of $p$ with respect to $\mathcal{G}$. This remainder is zero if and only if $p$ is contained in $I$; see \cite[Chap. 2.8]{cox2015ideals}.

\textbf{Ideal containment.} Let $J_1, J_2$ be ideals of $R$, and suppose $\mathcal{G}_1, \mathcal{G}_2$ are Gr\"{o}bner bases for $J_1, J_2$, respectively. How do we decide if $J_1 \subseteq J_2$? Suppose $\mathcal{G}_1 = \{g_1, \dots, g_k\}$. Then $J_1 \subseteq J_2$ if and only if $g_i$ reduces to zero with respect to $\mathcal{G}_2$ for all $1\leq i\leq k$.

\textbf{Intersection of ideals.} Let $J_1, J_2$ be ideals of $R$, and suppose that $\mathcal{G}_1 = \{g_1, \dots, g_k\}, \mathcal{G}_2 = \{h_1, \dots, h_\ell\}$ are Gr\"{o}bner bases for $J_1, J_2$, respectively. What is a possible generating set $\mathcal{G}$ for $J_1\cap J_2$? We can compute $\mathcal{G}$ algorithmically as follows: First, construct an ``extended'' polynomial ring $R' := \Bbbk[x_1, \dots, x_n, t]$ that has an additional variable $t$, so that $R$ is a subring of $R'$, and consider any monomial order $<$ such that any monomial divisible by $t$ is larger than all monomials not divisible by $t$. Construct a new ideal $K := \langle tg_1, \dots, tg_k, (1-t)h_1, \dots, (1-t)h_\ell\rangle$ contained in the new ring $R'$, compute a Gr\"{o}bner basis $\mathcal{G}'$ for $K$ with respect to the described monomial order $<$, then compute the subset $\mathcal{G}\subseteq \mathcal{G}'$ comprising all those elements $g' \in \mathcal{G}'$ such that $\inid_<(g')$ is not divisible by $t$. This final set $\mathcal{G}$ is not just a generating set for $J_1 \cap J_2$; it is in fact a Gr\"{o}bner basis for $J_1\cap J_2$. For proof details on why this algorithm works, see \cite[Chap. 4.3]{cox2015ideals}.

For a more detailed treatment of Gr\"{o}bner basis theory, including its rich connections to algorithms in computational algebra and the numerous computational problems in commutative algebra, see \cite[Chap. 15]{Eisenbud:CommutativeAlgebraBook}.

\subsection{Primary ideals and primary decompositions}
\label{AppendixSubsec:primaryIdeals,pd}

The notion of {primary decompositions} arises as a far-reaching generalization of the idea of prime factorization in integers, and has deep implications in number theory, group theory, and algebraic geometry. For example, the prime factorization of integers, the classification of finite abelian groups, and the decomposition of algebraic varieties into their irreducible components, can be interpreted as three seemingly different decomposition theorems that are three special cases of the celebrated Lasker--Noether theorem; see \cite[Thm. 3.10]{Eisenbud:CommutativeAlgebraBook} for a statement of this theorem in its full generality. 

In our context, the Lasker--Noether theorem tells us that every ideal of $R$ can be decomposed into an intersection of finitely many ``primary'' ideals. Such a decomposition is called a ``{primary decomposition}'', and there are several known algorithms for computing the primary decompositions of ideals in polynomial rings; see \cite{EisenbudHunekeVasconcelos1992:PrimaryDecompositions,GianniTragerZacharias1988:PrimaryDecomposition,ShimoyamaYokoyama1996:PrimaryDecompositions} for state-of-the-art algorithms. In our paper, the computation of primary decompositions of concepts (i.e. monomial ideals) is a key ingredient for our algebraic machine reasoning framework. 

\subsubsection{What exactly is a primary decomposition?}
To rigorously define primary decompositions and related notions, we first need to introduce more technical algebraic terminology.

\noindent\textbf{Prime ideals and primary ideals}. Let $J$ be an ideal of $R$. We say that $J$ is \textit{proper} if $J$ is not the entire polynomial ring $R$, i.e. if $J\neq \langle 1\rangle$.
We say that $J$ is a \textit{prime} ideal if $J$ is proper ideal that satisfies the following condition:
\begin{center}
If $J_1, J_2$ are ideals of $R$ such that $J_1J_2 \subseteq J$, then either $J_1 \subseteq J$ or $J_2 \subseteq J$ (or both).
\end{center}
A prime ideal $P$ of $R$ is called an \textit{associated prime} of $J$, if there exists an element $r\in R$ satisfying $r\not\in J$, such that $P = \{p\in R| pr\in J\}$. We say the ideal $J$ is \textit{primary} if $J$ has exactly one associated prime; if this unique associated prime is $P$, then we say that $J$ is a \textit{$P$-primary} ideal. 
There is another useful equivalent definition for primary ideals; cf. \cite[Prop. 3.9]{Eisenbud:CommutativeAlgebraBook}.  An ideal $J$ of $R$ is \textit{primary} if it is a proper ideal that satisfies the following condition:
\begin{center}
If $x, y$ are elements of $R$ such that $xy\in J$, then either $x\in J$, or $y^k \in J$ for some $k\in \mathbb{Z}^+$.
\end{center}

\noindent\textbf{Primary decompositions and associated primes}.
We now state (a special case of) the Lasker--Noether theorem:

\begin{theorem}[{\cite[Thm. 3.10]{Eisenbud:CommutativeAlgebraBook}}]\label{thm:PrimaryDecomp}
Every proper ideal $J$ of $R$ can be written as an intersection $J = J_1 \cap \dots \cap J_k$ for some finitely many primary ideals $J_1, \dots, J_k$ of $R$. Furthermore, if each $J_i$ is a $P_i$-primary ideal for a prime ideal $P_i$ of $R$, then every associated prime of $J$ must be one of the prime ideals $P_1, \dots, P_k$.
\end{theorem}

A \textit{primary decomposition} of $J \subseteq R$ is a decomposition of $J$ as an intersection of finitely many primary ideals of $R$. Theorem \ref{thm:PrimaryDecomp} says that primary decompositions for proper ideals of $R$ always exist. Note that in general, an ideal $J$ could have multiple primary decompositions.

Consider a primary decomposition $J = J_1 \cap \dots \cap J_k$ of a proper ideal $J$. The primary ideals $J_1, \dots, J_k$ are called the \textit{primary components} of this primary decomposition.\footnote{When dealing with primary components, it is implicitly assumed that any primary component is referenced with respect to some given primary decomposition. Since primary decompositions are not unique, merely mentioning that $J_i$ is a primary component of $J$ without having any context of the primary decomposition that $J_i$ is a primary component of, would be ambiguous.} If each $J_i$ is a $P_i$-primary ideal for some prime ideal $P_i$, then we say that $J_i$ is a $P_i$-primary component (of this primary decomposition). If the number $k$ of primary components in the primary decomposition $J = J_1 \cap \dots \cap J_k$ is the smallest among all possible primary decompositions of $J$, then we say that $J = J_1 \cap \dots \cap J_k$ is a \textit{minimal} primary decomposition. In this case, the associated primes $P_1, \dots, P_k$ must all be distinct (see \cite[Thm. 3.10c]{Eisenbud:CommutativeAlgebraBook}). Note that minimal primary decompositions are in general not unique.

The \textit{radical} of an ideal $J \subseteq R$ is the ideal defined by
\[\sqrt{J} := \{r\in R| r^k \in J\text{ for some }k\in \mathbb{Z}^+\}.\]
The next proposition relates radicals to the associated primes of primary components in a primary decomposition.

\begin{proposition}\label{prop:minPrimaryDecomp}
Let $J$ be a proper ideal of $R$, and let $J = J_1 \cap \dots \cap J_k$ be a minimal primary decomposition of $J$. Then for each $1\leq i\leq k$, the radical $\sqrt{J_i}$ is a prime ideal, and the primary component $J_i$ is a $\sqrt{J_i}$-primary ideal. Furthermore, $\sqrt{J_1}, \dots, \sqrt{J_k}$ are all the distinct associated primes of $J$.
\end{proposition}
\begin{proof}
Let $J'_1, J'_2$ be ideals of $\sqrt{J_i}$ such that $J'_1J'_2 \subseteq \sqrt{J_i}$. Let $j_1\in J'_1, j_2 \in J'_2$ be arbitrarily chosen non-zero elements. By the definition of a radical, we have $(j_1j_2)^t \in J$ for some $t\in \mathbb{Z}^+$, so since $J_i$ is primary, we infer that either $j_1^t \in J_i$, or $j_2^{tk} \in J_i$ for some $k\in \mathbb{Z}^+$; this implies that either $j_1 \in \sqrt{J_i}$ or $j_2 \in \sqrt{J_i}$, i.e. $\sqrt{J_i}$ is a prime ideal. Thus, by \cite[Cor. 2.12]{Eisenbud:CommutativeAlgebraBook}, $\sqrt{J_i}$ is the smallest prime ideal containing $J_i$, hence \cite[Prop. 3.9]{Eisenbud:CommutativeAlgebraBook} implies that $\sqrt{J_i}$ is the unique associated prime of $J_i$. Finally, \cite[Thm. 3.9c]{Eisenbud:CommutativeAlgebraBook} completes the proof.
\end{proof}

\begin{remark}\label{rem:uniquenessAss}
A useful consequence of Proposition \ref{prop:minPrimaryDecomp} is that although an ideal $J$ of $R$ could have multiple different minimial primary decompositions $J = J_1 \cap \dots \cap J_k$, the set $\{\sqrt{J_1}, \dots, \sqrt{J_k}\}$ of prime ideals corresponding to any minimal primary decomposition is always the same; this set is called the \textit{set of associated primes} of $J$.
\end{remark}

\begin{remark}\label{rem:PrimaryDecompCAS}
In computer algebra systems, a primary decomposition is typically represented as a list of primary ideals. Although there are known algorithms for computing a minimal primary decomposition\footnote{Existing implementations of algorithms to compute primary decompositions always return minimal primary decompositions. For a fixed ideal $J$, it is possible to compute more than one minimal primary decomposition, e.g. in Macaulay2; see \cite{PrimaryDecompositionSource,hocsten2002monomial}.} of an arbitrary proper ideal of a polynomial ring (see, e.g., \cite{PrimaryDecompositionSource} for several implementations in Macaulay2), we caution the reader that the computation of primary decompositions may not be supported in some computer algebra systems. 
\end{remark}

\noindent\textbf{Unique minimal primary decompositions}.
In the special case that $J$ is a proper monomial ideal of $R$, the minimal primary decompositions of $J$ are particularly nice:

\begin{theorem}[{\cite{BayerGalligoStillman1993}; cf. \cite[Ex. 3.11]{Eisenbud:CommutativeAlgebraBook}}]\label{thm:UniquePrimDecompMonomial}
Let $J$ be a proper monomial ideal of $R$. Then there exists a minimal primary decomposition $J = J_1 \cap \dots \cap J_k$ satisfying the following:
\setlist{nosep}
\begin{itemize}
\item Every primary component $J_i$ is a monomial ideal.
\item Each primary component $J_i$ is a $P_i$-primary ideal, where the prime ideal $P_i = \sqrt{J_i}$ is generated by a subset of $\{x_1, \dots, x_n\}$.
\item Each primary component $J_i$ is maximal among all possible monomial $P_i$-primary components.
\end{itemize}
Furthermore, any minimal primary primary decomposition of $J$ satisfying these properties is uniquely determined up to re-orderings of the primary components.
\end{theorem}

Succinctly, every proper monomial ideal $J$ of $R$ has a \textbf{unique} minimal primary decomposition with maximal monomial primary components, whose corresponding uniquely determined set $\{J_1, \dots, J_k\}$ of primary components is denoted by $(J)$.

In our language of algebraic machine reasoning (see, in particular, Appendix \ref{AppendixSubsubsec:HierarchyConcepts}), Theorem \ref{thm:UniquePrimDecompMonomial}, combined with Proposition \ref{prop:minPrimaryDecomp}, yields the following corollary.

\begin{corollary}\label{cor:PrimaryDecompProperConcept}
Let $J$ be a concept of $R$ such that $J\neq \langle 1\rangle$, i.e. $J$ is not the ``conceivable'' concept. Then there exist finitely many distinct simple concepts $P_1, \dots, P_k$ of $R$ that are uniquely determined by $J$, and there exists a primary decomposition $J =  J_1 \cap \dots \cap J_k$ of the concept $J$ as an intersection of $k$ concepts $J_1, \dots, J_k$, such that every radical $\sqrt{J_i}$ is the simple concept $P_i$. 
Furthermore, any such decomposition for which $J_i$ is maximal over all possible concepts with radical $\sqrt{J_i} = P_i$, is uniquely determined. 
\end{corollary}

\noindent\textbf{Primary decompositions for basic concepts}. As we have mentioned in Appendix \ref{AppendixSubsubsec:HierarchyConcepts}, we encoded RPM panels algebraically as basic concepts, i.e. whose minimal instances are squarefree. The following theorem tells us that basic concepts are particularly nice, because they have unique minimal primary decompositions whose primary components are simple concepts.

\begin{theorem}\label{thm:basicConceptDecomp}
Let $J$ be a basic concept of $R$. Then there exists a primary decomposition $J = J_1 \cap \dots \cap J_k$ of the concept $J$ as an intersection of finitely many distinct simple concepts $J_1, \dots, J_k$. Furthermore, such a decomposition is unique up to re-orderings of the simple concepts $J_1, \dots, J_k$.
\end{theorem}

\begin{proof}
Suppose $\mingen(J) = \{m_1, \dots, m_k\}$. Since $J\neq \langle 1\rangle$, none of the monomial generators $m_1, \dots, m_k$ can be $1$, so each of them has degree $\geq 1$. If all of $m_1, \dots, m_k$ have degrees exactly $1$, then $J$ is by definition a simple concept, and we are trivially done. Thus, we can assume henceforth that 
there is some minimal monomial generator $m_j$ such that $\deg(m_j) \geq 2$. This means we can write $m_j = m'm''$ for some monomials $m', m''$ in $\mathcal{M}_R$, each with degree $\geq 1$. Since $J$ is a basic concept, all minimal monomial generators of $J$ must be squarefree, so $m_j$ is a squarefree monomial, which implies that $\gcd(m', m'') = 1$. Hence, by defining $\hat{\mathcal{G}} := \mingen(J)\backslash \{m_j\}$, we can then decompose $J$ as an intersection
\begin{equation}
\label{eqn:primaryDecompositionBasicConcept}
J = \langle \hat{\mathcal{G}} \cup \{m'\}\rangle \cap \langle \hat{\mathcal{G}} \cup \{m''\}\rangle,
\end{equation}
where each of the two ideals $\langle \hat{\mathcal{G}} \cup \{m'\}\rangle$ and $\langle \hat{\mathcal{G}} \cup \{m''\}\rangle$ in this decomposition is a squarefree monomial ideal.

Consequently, for every squarefree monomial ideal $J'$ in any decomposition of $J$ as an intersection of squarefree monomial ideals, we can iteratively decompose $J'$ further as an intersection of two strictly larger squarefree monomial ideals whenever $J'$ contains a minimal monomial generator $m$ with degree $\geq 2$. By iteratively repeating this process, we would then get a decomposition $J = J_1 \cap \dots \cap J_k$, where $J_1, \dots, J_k$ are squarefree monomial ideals whose minimal monomial generators all have degree $1$. By definition, each such $J_i$ is precisely a simple concept, and $\sqrt{J_i} = J_i$, therefore our assertion follows from Corollary \ref{cor:PrimaryDecompProperConcept}.
\end{proof}

\subsubsection{Inductive bias for algebraic machine reasoning}
\label{AppendixSubsubsec:inductiveBias}

Theorem \ref{thm:basicConceptDecomp} has an important consequence. It tells us that any basic concept $J$ of $R$, no matter how ``complicated'' it may seem, can always be decomposed as an intersection of simple concepts. Simple concepts are easy to interpret, since by definition they are sums of primitive concepts. 

Among all possible simple concepts (i.e. whose generating sets range over all possible non-empty subsets of $\{x_1, \dots, x_n\}$, there are certain ``distinguished'' simple concepts that are important for reasoning. Such ``distinguished'' simple concepts are task-specific, and we call them \textit{attribute concepts}. Informally, an attribute concept is a concept representing a certain attribute of interest. 

In the RPM task, recall that we have attribute concepts representing ``position'', ``number'', ``type'', ``size'', and ``color''; these are simple concepts that categorize the primitive instances according to their semantics, into what humans would call attributes. For example, the ``type'' attribute concept $J_{\text{type}} = \langle x_{\text{triangle}}, x_{\text{square}}, x_{\text{pentagon}}, x_{\text{hexagon}}, x_{\text{circle}} \rangle$ is a simple concept generated by $x_{\text{triangle}}, x_{\text{square}}, x_{\text{pentagon}}, x_{\text{hexagon}}, x_{\text{circle}}$, which represent all possible instances of the atribute ``type'' in the RPM task.

The minimal instances of an attribute concept can be interpreted as ``attribute values''. This means that a simple concept contained in an attribute concept would be generated by ``attribute values'' of a single ``attribute'', and hence would inherit the same semantics as the corresponding attribute concept. For example, the simple concept $J := \langle x_{\text{triangle}}, x_{\text{pentagon}}, x_{\text{circle}} \rangle$ is contained in $J_{\text{type}}$ and inherits the semantics of the attribute ``type''. Notice that $J_{\text{type}}$ is a concept that is meaningful to humans (representing possible types of geometric entities), hence $J$ would also be a concept that is similarly meaningful to humans (representing some possible types of geometric entities). After all, the minimal instances of $J$ are all instances of the same attribute ``type''.

Note that the identification of such attribute concepts is task-specific, and the resulting reasoning performance (e.g. in terms of the extracted patterns) would depend heavily on these identified attribute concepts. In particular, our choice of which simple concepts are attribute concepts would determine the \textit{inductive bias} of our reasoning framework: After decomposing a basic concept into multiple simple concepts, only those simple concepts contained in attribute concepts are deemed ``meaningful'', and we extract such simple concepts as concepts representing patterns that are meaningful to humans. Roughly speaking, when extracting common patterns with this inductive bias, our reasoning framework is ``biased'' towards only those ``simpler'' concepts that are contained in attribute concepts.

What about primitive concepts? Typically, they are self-evident and easy to define. This is because the precise formulation of any reasoning task would first require an identification of (the semantics of) the possible data attribute values, which then naturally informs us what our primitive concepts should be. For the RPM task, we had $4$ perception modules for the $4$ attributes ``color'', ``size'', ``type'', ``position'', so all the possible attribute values across these $4$ attributes would naturally form primitive concepts.

In general, for visual reasoning tasks, the possible object classes that can be predicted by perception models would naturally take the role of primitive concepts for our reasoning framework.

\subsection{Algebraic intuition for concepts}
\label{AppendixSubsec:ConceptsIntuition}

\subsubsection{Proof of Theorem \ref{Theo:PrimaryDecompProperConcept}}
For convenience to the reader, we re-state Theorem \ref{Theo:PrimaryDecompProperConcept}:\\[0.4em]
\noindent\textbf{Theorem 3.1.}\textit{
There are infinitely many concepts in $R$, even though there are finitely many primitive concepts in $R$. Furthermore, if $J \subseteq R$ is a concept, then the following hold:
\begin{enum}
\item $J$ has infinitely many instances, unless $J = \langle 0 \rangle$.\label{thm31:parti}
\item $J$ has a unique minimal generating set consisting of \mbox{finitely many instances, which we denote by $\mingen(J)$.}\label{thm31:partii}
\item If $J \neq \langle 1\rangle$, then $J$ has a unique set of associated concepts $\{P_1, \dots, P_k\}$, together with a unique minimal primary decomposition $J = J_1 \cap \dots \cap J_k$, such that each $J_i$ is a concept contained in $P_i$, that is maximal among all possible primary components contained in $P_i$ that are concepts.\label{thm31:partiii}
\end{enum}
}

\begin{proof}
Since $R = \Bbbk[x_1, \dots, x_n]$, there are a total of $n$ primitive concepts in $R$, i.e. $\langle x_1\rangle, \dots, \langle x_n\rangle$. Since the set $\mathcal{M}_R$ of all monomials in $R$ is infinite, and since each $m\in \mathcal{M}_R$ generates a concept $\langle m\rangle$, there are infinitely many concepts in $R$. To show part \ref{thm31:parti}, notice that if $J\neq \langle 0\rangle$, then $J$ must contain at least one monomial $m\in \mathcal{M}_R$, thus by the definition of an ideal, $J$ contains $mx_1^k$ for all $k\in \mathbb{N}$ and so has infinitely many instances. For part \ref{thm31:partii}, see Proposition \ref{prop:Rmonomialbasis}. Finally, part \ref{thm31:partiii} follows directly from Corollary \ref{cor:PrimaryDecompProperConcept} and Proposition \ref{prop:minPrimaryDecomp}.
\end{proof}

\subsubsection{Compatibility with human reasoning and human understanding of concepts}
In cognitive science, \textit{concepts} are the most fundamental constructs for understanding human cognition~\cite{margolis1999concepts,murphy2004big}. There are numerous competing theories~\cite{carey2000origin,kamp1995prototype,medin1984concepts,malt1984correlated,locke1948essay,gopnik1994theory,pinker2005faculty} on what ``concepts'' are (or should be). Even the seemingly innocuous question of whether the underlying principles (upon which these different theories are based on) are compatible or not, has been fiercely debated, and still remains unresolved today~\cite{taylor2017conceptual}.

Faced with such circumstances as our backdrop, we do not seek to add to the debate by presenting a new theory on concepts (in the sense of constructs for human cognition). Rather, we seek to explain how our definition of concepts as monomial ideals adequately captures some of the richness of various aspects of concepts in human cognition. To us, our algebraic approach is grounded in the intuition that concepts are \emph{computable}, and they represent abstractions that derive \emph{computable meanings} from its relations to other concepts.

\textbf{Compositionality.} Concepts are compositional~\cite{kamp1995prototype,lake2017ingredients}. A concept could be composed of multiple ``simpler'' concepts. Distinct concepts could share a common concept. In philosophy and in linguistics, this compositional structure for concepts is widely accepted as an important aspect of the human experience in learning new concepts~\cite{Grady1990:Compositionality}. For the RPM task, the algebraic representation we used to encode RPM panels is compositional. We can compose a concept in $R$ by means of sums, products and intersections of concepts in $R$. We can decompose a concept as an intersection of multiple concepts (via primary decompositions). As we have discussed in Section \ref{Section:Discussion}, and elaborated in Appendix \ref{AppendixSubsubsec:inductiveBias}, the discovery of abstract patterns is based on primary decompositions; for this algebraic approach to work, we inherently need the ``compositional'' property of ideals; cf. Appendix \ref{subsubsec:PDnoAnalogInLogic}.

\textbf{Essence of concepts.} In some theories of concepts~\cite{kamp1995prototype,locke1948essay,gopnik1994theory},
there is an imprecise notion of the ``essence'' of a concept. We could view this ``essence'' as the definition of the concept (e.g. consider the famous example that the concept ``bachelor'' is synonymous with ``unmarried man''). However, more recent theories of concepts have espoused the belief that many concepts that are meaningful to humans (e.g. ``knowledge'', ``truth'') do not lend themselves easily to precise definitions in terms of other concepts~\cite{gettier2020justified}. This lack of definitional structure in some concepts has led to the idea that concepts can be defined in terms of its constituent \textit{features}. For example, ``feathers'' and ``wings'' can be viewed as features of the concept ``bird''.

Careful distinction has to be made between \textit{defining features} and \textit{irrelevant features}. For example, ``singing and chirping'' could be a feature of the concept ``bird'', but it is not a defining feature: A bird does not need to be able to sing or chirp to be considered a bird. Only the defining features of the concept would be considered part of the ``essence'' of that concept. However, it remains a challenge as to how the defining features of concepts can be extracted in a principled manner; see \cite{landau1982will}.

Is there an analogous notion of ``essence'' for concepts defined as monomial ideals?
By definition, we can construct a concept in $R$ by specifying its generating set; this is very much analogous to specifying features of the concept. Given a generating set $\mathcal{G}$ for a concept $J$ in $R$, we can compute the set $\mingen(J)$ of minimal monomial generators of $J$. By the uniqueness of $\mingen(J)$ (see Proposition \ref{prop:Rmonomialbasis}), this means we can interpret the generators in $\mingen(J)$ as being analogous to the defining features of the concept. Generators in $\mathcal{G}$ that are not in $\mingen(J)$ are redundant for generating $J$, and could be interpreted as irrelevant features. Consequently, we can view $\mingen(J)$ as the ``essence'' of concept $J$: Effectively, $\mingen(J)$ is a ``distinguished'' finite set of instances that determines all the infinitely many possible instances in $J$; cf. Theorem \ref{Theo:PrimaryDecompProperConcept} in Section \ref{Subsubsection:concepts}

In contrast, there is no analogous notion of $\mingen(J)$ in a logic-based or set-theoretic setting: For a given logical expression, there is \textit{a priori} no ``distinguished'' subset of the Boolean-type variables (appearing in the logical expression), or a subset of the clauses, that would represent the entire logical expression. For any given infinite set, if we do not have any further information about its elements, then there is no natural way to assign to it a ``distinguished'' finite subset that still represents the entire infinite set.

\textbf{Concepts with partial definitions.} In the classical theory of concepts, a concept is characterized in terms of necessary and sufficient conditions for an instance to be a member of that concept~\cite{carey2000origin}. A key criticism of this classical theory is that humans are still able to reason with concepts that have \textit{partial definitions}; this has been the starting point of neoclassical theories of concepts since the 1980s~\cite{margolis1999concepts}. 

In algebraic machine reasoning, our notion of concepts is well-equipped to handle partial definitions. Consider the attribute value ``$x_{\text{square}}$'' from the RPM task. Note that the concept $\langle x_{\text{square}}\rangle$ represents ``square'', which can be defined mathematically as a geometric entity in a plane with four equal sides and four right angles. Yet, we do not require this full definition to solve the RPM task. All we need is a \emph{partial} definition based on the condition that a square has four sides; we do not explicitly use other necessary conditions such as ``equals sides" and ``equal angles". The necessary condition ``four sides'' is encoded as follows: (i) the relation $f_{\text{next}}(\langle x_{\text{triangle}}\rangle) = \langle x_{\text{square}}\rangle$; (ii) the relation $f_{\text{next}}(\langle x_{\text{square}}\rangle) = \langle x_{\text{pentagon}}\rangle$; and implicitly, (iii) a perception model that distinguishes the object class associated to $x_{\text{square}}$ from other object classes associated to the remaining variables in $\mathcal{A}_{\text{type}}$. For example, even if we do not know the defining feature that a square should have four right angles, and allow a rhombus to be an instance of concept $\langle x_{\text{square}} \rangle$, we could still solve the RPM task, provided that interior angles are not a relevant feature for solving (which is the case for both the RAVEN and I-RAVEN datasets).

At the heart of the versatility of our algebraic approach, especially in dealing with partial definitions, is the fact that we are free to impose any semantics we want on the primitive concepts.
In particular, the primitive concepts in $R$ do not necessarily need to match the traditional notion of ``primitive concepts''. In the classical theory of concepts (in cognitive science), primitive concepts are primarily based on human sensory perception~\cite{locke1948essay}. In contrast, what we define as primitive concepts in $R$ is up to us, and would effectively depend on the specific reasoning task we wish to solve, as well as the perception models we have; cf. Appendix \ref{AppendixSubsubsec:inductiveBias}.

Furthermore, in contrast to logic-based approaches, we do not need to ``assign'' truth values to all primitive concepts. For example, to determine if an instance $m$ is contained in the concept $J = \langle x_{\text{white}}x_{\text{square}} \rangle$, we do not need to know if the entity represented by $m$ is small or large. The ``position'' and ``size'' attribute values could be unknown. All we need for determining that $m\in J$ is to verify the \emph{necessary} conditions, that the entity represented by $m$ is white and is a square. We could still reason about all instances in $J$, even if some of these instances have certain unknown attribute values. Consequently, our reasoning framework is able to deal with \textit{partial information}, and still make inferences about instances for which some of the truth values (true or false), on whether they are contained in each primitive concept, could be unknown. Intuitively, we do not have to assign truth values to all variables (representing statements about instance membership in primitive concepts), to be able to reason using algebraic machine reasoning.

\begin{table*}[!tb]
\centering
\begin{tabular}{c|l}
\toprule
\multicolumn{1}{c|}{Attribute value set} & \multicolumn{1}{c}{Variable labels} \\
\midrule
$\mathcal{A}_{\text{num}}:=\{x_0,\dots,x_8\}$ &  $x_{\text{one}}, x_{\text{two}},\dots,x_{\text{nine}}$                  \\
\midrule
\multirow{5}{*}{$\mathcal{A}_{\text{pos}}:=\{x_9,\dots,x_{31}\}$} & $x_{(0.16,0.16,0.33)}, x_{(0.16,0.5,0.33)}, x_{(0.16,0.83,0.33)},x_{(0.5, 0.16, 0.33)},x_{(0.5, 0.5, 0.33)},$ \\
 & $x_{(0.5,0.83,0.33)},x_{(0.83,0.16,0.33)},x_{(0.83,0.5,0.33)},x_{(0.83,0.83,0.33)},$ \\
 & $x_{(0.25,0.25,0.5)},x_{(0.25,0.75,0.5)},x_{(0.75,0.25,0.5)},x_{(0.75,0.75,0.5)},$ \\
 & $x_{(0.42,0.42,0.15)},x_{(0.42,0.58,0.15)},x_{(0.58,0.42,0.15)},x_{(0.58,0.58,0.15)},$ \\
 & $x_{(0.5,0.25,0.5)},x_{(0.5,0.75,0.5)},x_{(0.25,0.5,0.5)},x_{(0.75,0.5,0.5)},x_{(0.5,0.5,1.0)},x_\text{dummy}$ \\
 \midrule
$\mathcal{A}_{\text{type}}:=\{x_{32},\dots,x_{36}\}$ &  $x_{\text{triangle}}, x_{\text{square}}, x_{\text{pentagon}}, x_{\text{hexagon}}, x_{\text{circle}} $  \\
\midrule
$\mathcal{A}_{\text{color}}:=\{x_{37},\dots,x_{46}\}$ &  $x_{\#255}, x_{\#224}, x_{\#196}, x_{\#168}, x_{\#140}, x_{\#112}, x_{\#84}, x_{\#56}, x_{\#28}, x_{\#0}$   \\

\midrule
\multirow{4}{*}{$\mathcal{A}_{\text{size}}:=\{x_{47},\dots,x_{68}\}$} & $x_{(0.6,0.15)}, x_{(0.7,0.15)}, x_{(0.8,0.15)}, x_{(0.9,0.15)},$                    \\
 & $x_{(0.4,0.33)},x_{(0.5,0.33)},x_{(0.6,0.33)}, x_{(0.7,0.33)}, x_{(0.8,0.33)}, x_{(0.9,0.33)},$  \\
 & $x_{(0.4,0.5)},x_{(0.5,0.5)},x_{(0.6,0.5)}, x_{(0.7,0.5)}, x_{(0.8,0.5)}, x_{(0.9,0.5)},$  \\
 & $x_{(0.4,1)},x_{(0.5,1)},x_{(0.6,1)}, x_{(0.7,1)}, x_{(0.8,1)}, x_{(0.9,1)},$  \\
\bottomrule
\end{tabular}
\caption{Five sets of variables that correspond to the attribute values of five attributes.}
\label{table:attributeValues}
\end{table*}

\begin{table*}[!tb]
\centering
\begin{tabular}{c|l|l}
\toprule
\multicolumn{1}{c|}{Attribute} & \multicolumn{1}{c|}{Sub-sequences} & \multicolumn{1}{c}{Remarks} \\
\midrule
 & $\begin{bmatrix} x_{(0.16,0.16,0.33)}, x_{(0.16,0.5,0.33)}, x_{(0.16,0.83,0.33)},x_{(0.5, 0.16, 0.33)},x_{(0.5, 0.5, 0.33)},\\ x_{(0.5,0.83,0.33)},x_{(0.83,0.16,0.33)},x_{(0.83,0.5,0.33)},x_{(0.83,0.83,0.33)} \end{bmatrix}$ & \raisebox{0.2em}{\rdelim\}{5}{*}[\parbox{29mm}{Splitting based on\\width (the last entry)}]}
 \\[1.4em]
\raisebox{-0.4em}{position} & $\begin{bmatrix}{x_{(0.25,0.25,0.5)},x_{(0.25,0.75,0.5)},x_{(0.75,0.25,0.5)},x_{(0.75,0.75,0.5)}}\end{bmatrix}$ &  \\[0.8em]
 & $\begin{bmatrix}{x_{(0.42,0.42,0.15)},x_{(0.42,0.58,0.15)},x_{(0.58,0.42,0.15)},x_{(0.58,0.58,0.15)}}\end{bmatrix}$ &  \\[0.8em]
 & $x_{(0.5,0.25,0.5)},x_{(0.5,0.75,0.5)},x_{(0.25,0.5,0.5)},x_{(0.75,0.5,0.5)},x_{(0.5,0.5,1.0)},x_\text{dummy}$ & $f_\text{next}(\langle x \rangle,\Delta)=\langle x_\text{dummy} \rangle$\\
\midrule
\multirow{6}{*}{size} & $[x_{(0.6,0.15)}, x_{(0.7,0.15)}, x_{(0.8,0.15)}, x_{(0.9,0.15)}]$ & \multirow{6}{*} {$\begin{matrix}\text{Splitting based on} \\ \text{width (the last entry)} \end{matrix}$}   \\[0.8em]
 & $[x_{(0.4,0.33)},x_{(0.5,0.33)},x_{(0.6,0.33)}, x_{(0.7,0.33)}, x_{(0.8,0.33)}, x_{(0.9,0.33)}]$ & \\[0.8em]
 & $[x_{(0.4,0.5)},x_{(0.5,0.5)},x_{(0.6,0.5)}, x_{(0.7,0.5)}, x_{(0.8,0.5)}, x_{(0.9,0.5)}]$ &  \\[0.8em]
 & $[x_{(0.4,1)},x_{(0.5,1)},x_{(0.6,1)}, x_{(0.7,1)}, x_{(0.8,1)}, x_{(0.9,1)}]$  \\
\bottomrule
\end{tabular}
\caption{The sub-sequences for attributes ``position" and ``size", within which we cyclically order the variables to capture the sequential information of the attribute values.}
\label{table:subsequences}
\end{table*}

\section{Further experiment details}
\label{AppendixSec:experiment}
In this section, we provide further implementation details of our algebraic machine reasoning framework and more experimental results, organized as follows:
\begin{itemize}
    \item Appendix \ref{AppendixSubsec:algebraicEncoding} gives further details on the algebraic representation stage, including the design of the attribute concepts and the performance of our object detection modules.
    \item Appendix \ref{AppendixSubsec:inverseInvarianceModule} describes the details on how to ``inversely" apply the 4 invariance modules to generate answers for RPM instances.
    \item Appendix \ref{AppendixSubsec:examples} gives examples of full computational details of the 4 invariance modules and the answer generation process.
    \item Appendix \ref{AppendixSubsec:ablation} gives details on an ablation study of our algebraic machine reasoning framework to analyze the effectiveness of each invariance module in the answer selection task.
    \item Appendix \ref{AppendixSubsec:generatedAnswers} provides examples of generated answers for each RPM configuration in I-RAVEN \cite{hu2021stratified}.
\end{itemize}

\subsection{Algebraic representation for RPM task}
\label{AppendixSubsec:algebraicEncoding}
\subsubsection{Semantics of algebraic objects}
\textbf{Attribute concepts}.
An instance of the RAVEN \cite{zhang2019raven} (or I-RAVEN \cite{hu2021stratified}) dataset is given as a pair of files: an NPZ file that stores the pixel values for a raw image, and an XML file that stores the given answer and information about the attribute values of entities in each of the 16 panels. As described in Section \ref{sec:attributeConcept}, we first need to define five disjoint sets of variables (representing the five attributes) to encode the RPM instances algebraically. In thi paper, we only considered those attribute values that are involved in our running example for ease of explanation. However, in order to use the same algebraic representation set-up across all RPM instances, we shall define the sets of variables that are composed of all possible attribute values. 

In Table \ref{table:attributeValues}, we provide the list of all 69 variables (denoted by $x_0,\dots,x_{68}$), grouped into our five attribute sets.  We shall also assign labels to these 69 variables; these labels are character strings meant for easier human interpretation of the meaning of each variable. We shall use the indices and variable labels inter-changeably, e.g. $x_0=x_\text{one}$. Note that the last variable $x_\text{dummy}$ in $\mathcal{A}_\text{pos}$ is introduced for the purpose of subsequent use (in conjunction with the function $f_\text{next}$) in the compositional invariance module, which will be explained later. For the rest of the primitive instances, we briefly describe below how they should be interpreted:
\begin{itemize}
    \item Instances in $\mathcal{A}_\text{num}$ represent the number of geometric entities in a given panel.
    \item Instances in $\mathcal{A}_\text{pos}$ represent the positions of sub-panels corresponding to geometric entities, in the form of 3-tuples $(x_\text{center},y_\text{center},\text{width})$. (The position values provided in the XML files are in the form of 4-tuples, i.e. $(x_\text{center}, y_\text{center}, \text{width}, \text{height})$. We omit the 4th entry since the width and height of each sub-panel are always equal.)
    \item Instances in $\mathcal{A}_\text{type}$ represent the types (or ``shapes") of geometric entities.
    \item Instances in $\mathcal{A}_\text{color}$ represent the grayscale pixel values of geometric entities, where $x_{\#255}$ represents ``white" and $x_{\#0}$ represents ``black".
    \item Instances in $\mathcal{A}_\text{size}$ represent the absolute sizes of sub-panels, in the form of 2-tuples (relative size, width). (The size values provided in the XML files are the relative size of entities (i.e. $0.4,\dots,0.9$) with respect to the size of the sub-panel.)
\end{itemize}

Let $\mathcal{L}:=\{\text{num}, \text{pos}, \text{type}, \text{color}, \text{size}\}$ denote the set of attribute labels. Then for each attribute label $\ell\in\mathcal{L}$, we can define the attribute concept $J_\ell=\langle \mathcal{A}_\ell\rangle$.

\textbf{Function $f_\text{next}$ in the compositional invariance module}.
We have described in Section \ref{sec:priorknowledge} that we cyclically order the values in $\mathcal{A}_\ell$ for each attribute label $\ell\in\mathcal{L}$, and define a function $f_\text{next}(J|\Delta)$ to encode the idea of ``next". Here $\Delta$ represents the step size to map any variable $x\in\mathcal{A}_\ell$ that appears in a generator of concept $J$ to the $\Delta$-th variable after $x$ ($|\Delta|$-th variable before $x$ for negative $\Delta$), w.r.t. to the cyclic order on $\mathcal{A}_\ell$. The cyclic order within each $\mathcal{A}_\ell$ is used to guarantee that $f_{\text{next}}(\langle x \rangle | \Delta)$ always lies in the same attribute concept as $\langle x \rangle$. 

For attributes ``number", ``type" and ``color", the variables can be  naturally ordered following the sequential order as indicated in Table \ref{table:attributeValues}. For attributes ``position" and ``size", we split the sequences of variables into multiple sub-sequences. Within each sub-sequence, we define a cyclic order that follows human intuition for the sequential information that is inherent in the semantics of the variable labels; see Table \ref{table:subsequences}. For example, the ``next" position of $x_\text{(0.75, 0.75, 0.5)}$ (representing the bottom-right subpanel) should cycle back to $x_\text{(0.25, 0.25, 0.5)}$ (representing the top-left subpanel), the ``next" size of $x_{(0.9,0.5)}$ should cycle back to $x_{(0.4,0.5)}$, etc. Note that the last few variables $x_{(0.5, 0.25, 0.5)},\dots,x_{(0.5,0.5,1.0)}$ of ``position" are not in any sub-sequence, as there is no obvious sequential information on these position variables. Hence we introduce a new variable $x_\text{dummy}$ so that for any position variable $x\in\{x_{(0.5, 0.25, 0.5)},\dots,x_{(0.5,0.5,1.0)}\}$,
\begin{align*}
    &f_\text{next}(\langle x\rangle|\text{any step }\Delta)=\langle x_\text{dummy} \rangle, \\
    &f_\text{next}(\langle x_\text{dummy}\rangle|\text{any step }\Delta)=\langle x_\text{dummy} \rangle.
\end{align*}

\textbf{Function $g$ in the binary-operator invariance module}. In Section \ref{Sec:module}, we introduced the binary-operator module to extract numerical patterns, based on a given real-valued function $g$ on concepts, and a given set $\Lambda$ of binary operators. In our experiments on RAVEN and I-RAVEN, we only considered attributes ``number", ``color" and ``size" for this invariance module, as the remaining two attributes ``type" and ``position" do not have obvious numerical meanings. For attribute ``number", we define the function $g_\text{num}$ that maps each concept $J$ to the number of generators in the unique minimal generating set of $J$:
\begin{equation*}
    g_\text{num}(J):= |\mingen(J)|.
\end{equation*}
For attribute ``color", we define the function $g_\text{color}$ to extract the variable index within the color value sequence in Table \ref{table:attributeValues}. Note that we only apply $g_\text{color}$ to those concepts $J$ for which the color variables are invariant across all generators in $\mingen(J)$. Similarly for attribute ``size", we define the function $g_\text{size}$ to extract the variable index within the size value sub-sequences in Table \ref{table:subsequences}. 

For example, consider the concept below:
\begin{equation*}
    J=\left\langle \begin{matrix}x_\text{two}x_{(0.5,0.25,0.5)}x_\text{square}x_\text{\#255}x_{(0.6,0.5)}, \\ x_\text{two}x_{(0.5,0.75,0.5)}x_\text{circle}x_\text{\#255}x_{(0.8,0.5)} \end{matrix} \right\rangle.
\end{equation*}
We can compute 
\begin{itemize}
    \item $g_\text{num}(J) = |\mingen(J)| = 2$, which indicates that $J$ has two minimal monomial generators;
    \item $g_\text{color}(J) = 0$, which indicates that the color variable (invariant across all generators in $\mingen(J)$) is the first element in the color value sequence;
    \item $g_\text{size}(J) = \texttt{None}$, as the size variables are different across the two generators in $\mingen(J)$.
\end{itemize}

\subsubsection{Object detection modules}
\label{AppendixSubsec:objectDetection}

\begin{table}[tb]	
\centering
\begin{tabular}{lllll}
\toprule
   Class   & GTs   & Dets  & Recall& AP \\ \midrule
 triangle& 120978& 120984& 1.000 & 1.000\\ 
 square  & 95967 & 95967 & 1.000 & 1.000\\ 
 pentagon& 121722& 121722& 1.000 & 1.000\\ 
 hexagon & 95017 & 95017 & 1.000 & 1.000\\ 
 circle  & 118578& 118578& 1.000 & 1.000\\
    \bottomrule %
\end{tabular} %
\centering %
 \caption{Object detection results for attribute ``type", evaluated on the validation set of I-RAVEN. ``GTs" refers to the number of ground truth objects, ``Dets" refers to the number of detected objects, and ``AP" refers to the average precision score.}
 \label{table:detectionType}
\end{table}

\begin{table}[htb]	
\centering
\begin{tabular}{lllll}
\toprule
   Class   & GTs   & Dets  & Recall& AP \\ \midrule
 \#255   & 114701 & 114743 & 1.000  & 1.000 \\
 \#224   & 48561  & 48570  & 1.000  & 1.000 \\
 \#196   & 49626  & 49627  & 1.000  & 1.000 \\
 \#168   & 50013  & 50016  & 1.000  & 1.000 \\
 \#140   & 52634  & 52639  & 1.000  & 1.000 \\
 \#112   & 51144  & 51144  & 1.000  & 1.000 \\
 \#84    & 48390  & 48396  & 1.000  & 1.000 \\
 \#56    & 48566  & 48568  & 1.000  & 1.000 \\
 \#28    & 44464  & 44468  & 1.000  & 1.000 \\
 \#0     & 44163  & 44167  & 1.000  & 1.000 \\
    \bottomrule %
\end{tabular} %
\centering %
 \caption{Object detection results for attribute ``color", evaluated on the validation set of I-RAVEN.}
 \label{table:detectionColor}
\end{table}

\begin{table}[tb]	
\centering
\begin{tabular}{lllll}
\toprule
   Class   & GTs   & Dets  & Recall& AP \\ \midrule
 (0.16, 0.16, 0.33) & 16646 & 16646 & 1.000  & 1.000 \\
 (0.16, 0.5, 0.33)  & 16803 & 16803 & 1.000  & 1.000 \\
 (0.16, 0.83, 0.33) & 16531 & 16559 & 1.000  & 1.000 \\
 (0.5, 0.16, 0.33)  & 16652 & 16652 & 1.000  & 1.000 \\
 (0.5, 0.5, 0.33)   & 18455 & 18445 & 0.999  & 0.999 \\
 (0.5, 0.83, 0.33)  & 18354 & 18348 & 1.000  & 1.000 \\
 (0.83, 0.16, 0.33) & 32000 & 32000 & 1.000  & 1.000 \\
 (0.83, 0.5, 0.33)  & 32000 & 32000 & 1.000  & 1.000 \\
 (0.83, 0.83, 0.33) & 32000 & 32000 & 1.000  & 1.000 \\
 (0.25, 0.25, 0.5)  & 16761 & 16761 & 1.000  & 1.000 \\
 (0.25, 0.75, 0.5)  & 16851 & 16851 & 1.000  & 1.000 \\
 (0.75, 0.25, 0.5)  & 18420 & 18404 & 0.999  & 0.999 \\
 (0.75, 0.75, 0.5)  & 96000 & 96000 & 1.000  & 1.000 \\
 (0.42, 0.42, 0.15) & 16753 & 16753 & 1.000  & 1.000 \\
 (0.42, 0.58, 0.15) & 16611 & 16611 & 1.000  & 1.000 \\
 (0.58, 0.42, 0.15) & 18128 & 18128 & 1.000  & 1.000 \\
 (0.58, 0.58, 0.15) & 17992 & 17992 & 1.000  & 1.000 \\
 (0.5, 0.25, 0.5)   & 18081 & 18042 & 0.998  & 0.998 \\
 (0.5, 0.75, 0.5)   & 18135 & 18104 & 0.998  & 0.998 \\
 (0.25, 0.5, 0.5)   & 48648 & 48688 & 1.000  & 1.000 \\
 (0.75, 0.5, 0.5)   & 32000 & 32000 & 1.000  & 1.000 \\
 (0.5, 0.5, 1.0)    & 18441 & 18417 & 0.999  & 0.999 \\
    \bottomrule %
\end{tabular} %
\centering %
 \caption{Object detection results for attribute ``position", evaluated on the validation set of I-RAVEN.}
 \label{table:detectionPosition}
\end{table}

\begin{table}[tb]	
\centering
\begin{tabular}{lllll}
\toprule
   Class   & GTs   & Dets  & Recall& AP \\ \midrule
 (0.6, 0.15) & 23552 & 19372 & 0.823  & 0.823 \\
(0.7, 0.15) & 14447 & 11715 & 0.811  & 0.811 \\
(0.8, 0.15) & 14811 & 14805 & 1.000  & 1.000 \\
(0.9, 0.15) & 20084 & 20084 & 1.000  & 1.000 \\
(0.4, 0.33) & 32152 & 32152 & 1.000  & 1.000 \\
(0.5, 0.33) & 30635 & 30635 & 1.000  & 1.000 \\
(0.6, 0.33) & 30070 & 31111 & 1.000  & 1.000 \\
(0.7, 0.33) & 29539 & 29539 & 1.000  & 1.000 \\
(0.8, 0.33) & 30528 & 30528 & 1.000  & 1.000 \\
(0.9, 0.33) & 29332 & 29361 & 1.000  & 1.000 \\
(0.4, 0.5)  & 35748 & 35748 & 1.000  & 1.000 \\
(0.5, 0.5)  & 34319 & 34319 & 1.000  & 1.000 \\
(0.6, 0.5)  & 32247 & 32247 & 1.000  & 1.000 \\
(0.7, 0.5)  & 33439 & 33439 & 1.000  & 1.000 \\
(0.8, 0.5)  & 33002 & 33002 & 1.000  & 1.000 \\
(0.9, 0.5)  & 32357 & 32365 & 1.000  & 1.000 \\
(0.4, 1.0)  & 5455  & 5455  & 1.000  & 1.000 \\
(0.5, 1.0)  & 5829  & 5829  & 1.000  & 1.000 \\
(0.6, 1.0)  & 5336  & 5336  & 1.000  & 1.000 \\
(0.7, 1.0)  & 26793 & 26793 & 1.000  & 1.000 \\
(0.8, 1.0)  & 25638 & 25636 & 1.000  & 1.000 \\
(0.9, 1.0)  & 26949 & 26949 & 1.000  & 1.000 \\
    \bottomrule %
\end{tabular} %
\centering %
 \caption{Object detection results for attribute ``size", evaluated on the validation set of I-RAVEN.}
 \label{table:detectionSize}
\end{table}

\begin{table}[htb]	
\centering
\begin{tabular}{lll}
\toprule
    Configuration  &  RAVEN & I-RAVEN \\
    \midrule
    Center & 32000 & 32000\\
    2$\times$2Grid & 32000  & 32000 \\
    3$\times$3Grid & 31996 & 31986 \\
    O-IC & 23290 & 23617 \\
    O-IG & 31988 & 31987 \\
    L-R & 31999 & 31998 \\
    U-D & 32000 & 32000 \\
    \midrule
    Avg. proportion & $96.10\%$ & $96.24\%$ \\
    \bottomrule %
\end{tabular} %
\centering %
 \caption{The perceptual results of attribute values extracted via object detection modules, evaluated on RAVEN and I-RAVEN, in terms of the number of correctly detected panels (out of $2,000\times 16$ panels) for each configuration and the overall proportions of correctly detected panels.}
 \label{table:overallDetectionResults}
\end{table}

In this subsubsection, we explain how we trained the object detection modules to extract the attribute values from the raw RPM images. As mentioned in \ref{Subsubsection:algebraicRepresentation}, we trained 4 standard RetinaNet models (each with a ResNet-50 backbone) separately for all attributes except attribute ``number", which can be directly inferred by counting the number of predicted bounding boxes within a panel. 

The RAVEN/I-RAVEN dataset typically contains 70,000 RPM instances that are evenly distributed among 7 configurations. Within each configuration, the 10,000 RPM instances are randomly split into 6,000 training instances, 2,000 validation instances, and 2,000 test instances. For our experiments, we used only $10\%$ of the training set, i.e. we trained our object detection modules on 4,200 RPM instances (600 from each configuration).
Every RPM instance can be split into 16 grayscale images of size $160\times 160$, corresponding to 8 panels in the question matrix and 8 answer options. Each image has one or more geometric entities, which we shall detect. The ground truth labels and bounding box of each geometric entity can be obtained from the given XML files. 

We used the \texttt{MMDetection} package \cite{chen2019mmdetection} for standard training, with their implementation of RetinaNet models with a ResNet-50 backbone. All hyper-parameters are set to the default values, with the following exceptions: (i) during training, the initial learning rate is 0.003 for all attributes; and (ii) during the inference stage, we use a confidence threshold of 0.99 for ``type", and 0.95 for the other 3 attributes. All 4 modules are trained over 12 epochs. The evaluation results on the validation set with 14,000 instances (2,000 for each configuration) are shown in Tables \ref{table:detectionType}-\ref{table:detectionSize}. 

According to the tables above, the object detection module for attribute ``type" has only 6 errors out of 552,262 objects, which significantly outperforms the object detection modules for the other 3 attributes. Hence the module for ``type" is used as an ``anchor" to combine the detection results obtained from all 4 modules. In particular, we first obtain the list of detected objects, with a confidence score above 0.99, from the ``type" detection module. The prediction result of each object can be reformulated as a pair (type label, bounding box). For every detected ``type" object, we then obtain the lists of detection results from the other 3 detection modules with a confidence score above 0.8; within each list, we select the (label, bounding box) pair with the highest Intersection Over Union (IoU) score w.r.t. the anchor bounding box from ``type" module. These 3 selected pairs are assumed to be related to the entity corresponding to the pair generated by the ``type" module. Thereafter, we can collect 4 closely related (label, bounding box) pairs, each representing a detected object returned by one detection module, and we can combine the labels from these pairs as the attribute values extracted from a single entity. Finally, we applied the 4 object detection modules trained on I-RAVEN to both RAVEN and I-RAVEN datasets, to extract the attribute values from the raw RPM test images. The evaluation results are given in Table \ref{table:overallDetectionResults}.

\subsection{Full algorithmic details for ``inverse" invariance modules and answer generation}
\label{AppendixSubsec:inverseInvarianceModule}
In this subsection, we provide the details on how to generate answers for RPM instances based on the common patterns extracted from the first two rows and the first two panels in the 3rd row. 
Informally, for every non-conflicting pattern pair $(K,\check{\mathbf{J}})\in \mathcal{P}_{1,2}(\mathbf{J}):=\mathcal{P}_{1}^{(\text{all})}(\mathbf{J})\cap \mathcal{P}_{2}^{(\text{all})}(\mathbf{J})$, we shall ``inversely" apply the corresponding invariance module to the 3rd row $\check{\mathbf{J}}_3:=[\check{J}_{3,1},\check{J}_{3,2}]$. 
Recall the following:
\begin{itemize}
    \item $\mathcal{P}_i^{(\text{all})}(\mathbf{J})$ represents the set of pattern pairs extracted from the $i$-th row of concept matrix $\mathbf{J}$;
    \item $\check{\mathbf{J}}$ represents a concept matrix from the extended list $[\mathbf{J}, \bar{\mathbf{J}}^{(p_1)}, \hat{\mathbf{J}}^{(p_1)},\dots, \bar{\mathbf{J}}^{(p_k)}, \hat{\mathbf{J}}^{(p_k)}]$;
    \item $K$ represents a common pattern specific to one attribute extracted from the first two rows of $\check{\mathbf{J}}$.
\end{itemize}
We begin by introducing the algorithmic details for each ``inverse" invariance module.

\subsubsection{``Inverse" intra-invariance module}

\begin{algorithm}[H]
\caption{``Inverse" intra-invariance module.}
\textbf{Inputs:} A pattern $\ell$ and the 3rd row $\check{\mathbf{J}}_3:=[\check{J}_{3,1},\check{J}_{3,2}]$.
\begin{algorithmic}[1] 
\STATE $I_\text{out}\gets\langle 0 \rangle$.
\STATE $\mathcal{I}\gets\text{pd}(\check{J}_{3,1}+\check{J}_{3,2})\cap \text{pd}(\check{J}_{3,1}\cap\check{J}_{3,2})$.
\FOR{$I\in\mathcal{I}$}
\IF{$I\subseteq\langle \mathcal{A}_\ell  \rangle$ }
\STATE $I_\text{out}\gets I$.
\ENDIF
\ENDFOR
\RETURN $I_\text{out}$.
\end{algorithmic}
\label{alg:inverseIntra}
\end{algorithm}

Given a sequence of concepts $J_1,\dots,J_k$, we have denoted the sum and intersection of these concepts by $J_+:=J_1+\dots+J_k$ and $J_\cap:=J_1\cap\dots\cap J_k$. Recall that this intra-invariance module extracts patterns where for some attribute, the corresponding set of attribute values for entities in $J_i$ remains invariant over all $i$. The set of such extracted patterns is given by:
\begin{equation*}
\!\text{
\resizebox{1 \linewidth}{!}{
$\mathcal{P}_{\text{intra}}([J_1\dots J_k]):=\big\{\text{attr} \in \mathcal{L}\mid \exists I\in\text{pd}(J_+)\cap \text{pd}(J_{\cap}), I\subseteq \langle \mathcal{A}_{\text{attr}}\rangle\big\}$.}}\\[-0.45em]
\end{equation*}
Hence for a pattern pair $(K,\check{\mathbf{J}})$ extracted via this invariance module, the common pattern $K$ is by definition an attribute label $\ell \in\mathcal{L}$. The set of values for attribute $\ell$ should remain invariant across $\check{J}_{3,1}$, $\check{J}_{3,2}$ and $\check{J}_{3,3}$. See Algorithm \ref{alg:inverseIntra} for the pseudo-code to compute the ideal generated by the corresponding attribute variable for $\check{J}_{3,3}$. 

Note that Algorithm \ref{alg:inverseIntra} can automatically filter out those patterns that conflict with the 3rd row $\check{\mathbf{J}}_3$, without needing any manual checking. Specifically, the output of Algorithm \ref{alg:inverseIntra} would be a zero ideal ($I_\text{out}=\langle 0 \rangle$) if the input pattern $\ell$ conflicts with the input 3rd row $\check{\mathbf{J}}_3$. For the remaining 3 ``inverse" invariance modules, we have the same mechanism to automatically filter the conflicting patterns, i.e. the presence of conflicting patterns would yield an output $I_\text{out}=\langle 0 \rangle$.

\subsubsection{``Inverse" inter-invariance module}

\begin{algorithm}[H]
\caption{``Inverse" inter-invariance module.}
\textbf{Inputs:} A pattern $(\ell, \mathcal{I})$ and the 3rd row $\check{\mathbf{J}}_3:=[\check{J}_{3,1},\check{J}_{3,2}]$.
\begin{algorithmic}[1] 
\STATE  $I_\text{out}\gets\langle 0 \rangle$.
\STATE Compute $\mathcal{I}'$ in \eqref{eq:Inot3}.
\IF{$\mathcal{I}-\mathcal{I}'\neq\varnothing$ and $\mathcal{I}\cap\mathcal{I}'\neq\varnothing$}
\STATE $I_\text{out}\gets\sum_{I\in\mathcal{I}-\mathcal{I}'}I$.
\ENDIF
\RETURN $I_\text{out}$.
\end{algorithmic}
\label{alg:inverseInter}
\end{algorithm}

Recall that the set of patterns extracted via inter-invariance module is given by:
\begin{equation*}
\resizebox{1 \linewidth}{!}{
$\mathcal{P}_{\text{inter}}([J_1,\dots,J_k]):=\Bigg\{(\text{attr},\mathcal{I}) \Bigg| \ \begin{matrix}\mathcal{I}\subseteq \text{pd}(J_\cap)-\text{pd}(J_+),  \\ \text{attr} \in \mathcal{L}, I\subseteq \langle \mathcal{A}_{\text{attr}} \rangle \  \forall I \in \mathcal{I} \end{matrix} \Bigg\},$}\\[-0.5em]
\end{equation*}
For a pattern pair $(K,\check{\mathbf{J}})$ extracted via this invariance module, the common pattern $K=(\ell, \mathcal{I})$ is a pair, whose first entry $\ell$ is an attribute label, and whose second entry $\mathcal{I}$ is a set of ideals contained in the attribute concept $\langle\mathcal{A}_\ell\rangle$. Naturally, we shall infer the attribute values for $\check{J}_{3,3}$ based on the set difference $\mathcal{I}-\mathcal{I}'$, where $\mathcal{I}'$ is defined by:
\begin{align}
\label{eq:Inot3}
\resizebox{1 \linewidth}{!}{
$\nonumber \mathcal{I}'=\left\{ I \big| I\in \Big(\text{pd}(\check{J}_{3,1}\cap\check{J}_{3,2}) - \text{pd}(\check{J}_{3,1}+\check{J}_{3,2})\Big), I\subseteq \langle \mathcal{A}_\ell \rangle \right\}. $}\\
\end{align}
Note that this $\mathcal{I}'$ is defined in the manner that is analogous to our definition for $\mathcal{P}_\text{inter}$ (restricted to the 3rd row of $\check{\mathbf{J}}$ and to attribute $\ell$).
See Algorithm \ref{alg:inverseInter} for a summary.

\subsubsection{``Inverse" compositional invariance module}

\begin{algorithm}[H]
\caption{``Inverse" compositional invariance module.}
\textbf{Inputs:} A pattern $(\ell, f_\text{next}|_{\Delta=\Delta^*})$ and the 3rd row $\check{\mathbf{J}}_3:=[\check{J}_{3,1},\check{J}_{3,2}]$.
\begin{algorithmic}[1] 
\STATE $I_\text{out}\gets\langle 0 \rangle$.
\STATE $\mathcal{I}\gets\text{pd}\big(f_\text{next}^2(\check{J}_{3,1}|\Delta^*)\big)\cap \text{pd}\big(f_\text{next}(\check{J}_{3,2}|\Delta^*)\big)$.
\FOR{$I\in\mathcal{I}$}
\IF{$I\in\langle \mathcal{A}_\ell  \rangle$}
\STATE $I_\text{out}\gets I$.
\ENDIF
\ENDFOR
\RETURN $I_\text{out}$.
\end{algorithmic}
\label{alg:inverseCompositional}
\end{algorithm}

Given a function $f$ that maps concepts to concepts, the compositional invariance module extracts the patterns arising from invariant attribute values in the new sequence of concepts:
\begin{equation*}
    [J'_1,\dots,J'_k]=[f^{k-1}(J_1), f^{k-2}(J_2), \dots, f(J_{k-1}), J_k].
\end{equation*}
The set of such patterns is given by: 
\begin{equation*}
\resizebox{1 \linewidth}{!}{
$\mathcal{P}_{\text{comp}}([J_1,\dots,J_k]):=\Bigg\{(\text{attr},f) \Bigg| \ \begin{matrix}\exists I \in \bigcap_{i=1}^k \text{pd}(f^{k-i}(J_i)),  \\ \text{attr} \in \mathcal{L}, I \subseteq \langle \mathcal{A}_{\text{attr}} \rangle \end{matrix} \Bigg\}.$}\\[-0.5em]
\end{equation*}
Each pattern $K=(\ell,f)$ can be interpreted as follows: The values of attribute $\ell$ are invariant across the new sequence of concepts. In other words, for any pair of successive concepts $[J_i,J_{i+1}]$, $i=1,\dots,k-1$, in the original sequence, the values of attribute $\ell$ should be invariant in $[f(J_i), J_{i+1}]$. This property provides a natural way to compute the attribute values of a concept $J_{i+1}$, given the attribute values of the previous concept $J_i$ in the sequence; see Algorithm \ref{alg:inverseCompositional}. Note that for the RPM task, we introduce the $f_\text{next}|_\Delta$ function for our compositional invariance module to extract the sequential information within a row of the concept matrix.

\subsubsection{``Inverse" binary-operator invariance module}

\begin{algorithm}[H]
\caption{``Inverse" binary-operator invariance module.}
\textbf{Inputs:} A pattern $(\overline{\boldsymbol{\oslash}},g_\ell,\Lambda)$, and the 3rd row $\check{\mathbf{J}}_3:=[\check{J}_{3,1},\check{J}_{3,2}]$.
\begin{algorithmic}[1] 
\STATE $I_\text{out}\gets\langle 0 \rangle$.
\STATE $t\gets g_\ell(\check{J}_{3,1}) \oslash_1 g_\ell(\check{J}_{3,2})$. \hfill //\textit{ expected $g_\ell(\check{J}_{3,3})$ value}
\STATE $x\gets t$-th variable in the corresponding attribute sequence (or subsequence) for attribute $\ell$.
\STATE  $I_\text{out}\gets \langle x \rangle$.
\RETURN $I_\text{out}$.
\end{algorithmic}
\label{alg:inverseBinary}
\end{algorithm}

Given a real-valued function $g$ on concepts and a set $\Lambda$ of binary operators, the binary-operator invariance module extracts patterns arising from the sequence of numbers $g(J_1),\dots,g(J_k)$, described as follows:
\begin{equation*}
\resizebox{1 \linewidth}{!}{$
    \mathcal{P}_\text{binary}(\mathbf{J}_i):=\Bigg\{(\overline{\boldsymbol{\oslash}},g,\Lambda) \Bigg| \ \begin{matrix} 
    \overline{\boldsymbol{\oslash}}=[\oslash_1,\dots,\oslash_{k-2}], \  \oslash_i \in \Lambda,
      \\ g(J_1) \oslash_1 \dots\oslash_{k-2} g(J_{k-1})=g(J_k) \end{matrix} \Bigg\}.$}\\[-0.7em]
\end{equation*}
In other words, we can represent each pattern as an equation, obtained by inserting binary operators chosen from $\Lambda$ between the numbers in the sequence.

For a new given sequence of concepts $J'_1,\dots,J'_{k-1}$ and a pattern $K=([\oslash_1,\dots,\oslash_{k-2}], g, \Lambda)$, the value of $g(J'_k)$ we expect from this pattern can be directly computed as $g(J'_k)=g(J'_1) \oslash_1 \dots\oslash_{k-2} g(J'_{k-1})$. Then we can infer the information about concept $J'_k$ based on the meaning of function $g$. 

As described in Section \ref{AppendixSubsec:algebraicEncoding}, we introduced the function $g_\text{num}(J)$ to extract the number of minimal monomial generators of $J$ (which also represents the number of entities in the corresponding panel), and introduced the functions $g_\text{color}$ (resp. $g_\text{size}$) to extract the variable index within the ``color" value sequence (resp. within the ``size" value sub-sequence). Hence we can directly infer the corresponding attribute variables for $J'_k$ from the values of $g_\ell(J'_k)$, for $\ell\in[``\text{num}", ``\text{color}", ``\text{size}"]$. See Algorithm \ref{alg:inverseBinary} for a summary.

\begin{table*}[!ht]	
\centering
\begin{tabular}{lllllllll}
\toprule
    Methods  &  \multicolumn{1}{c}{Avg. Acc.} & \multicolumn{1}{c}{Center} &  \multicolumn{1}{c}{2$\times$2G} &  \multicolumn{1}{c}{3$\times$3G} &  \multicolumn{1}{c}{O-IC} &  \multicolumn{1}{c}{O-IG} &  \multicolumn{1}{c}{L-R} &  \multicolumn{1}{c}{U-D} \\
    \midrule
    With all 4 invariance modules &  93.2  & 99.5  & 89.6  & 89.7  
    & 99.6  & 74.7 & 99.7  & 99.5  \\
    W/o intra-invariance module & 63.7 & 63.5 & 71.6 & 73.4 & 60.2 & 49.8 & 63.2 & 64.0  \\
    W/o inter-invariance module & 59.0 & 64.3 & 56.8 & 59.2 & 59.2 & 43.5 & 64.5 & 65.8 \\ 
    W/o compositional invariance module & 69.5 & 73.0 & 65.9 & 64.3 & 78.1 & 61.5 &72.9  & 71.2 \\
    W/o binary operator invariance module & 74.6 & 75.8 & 69.3 & 67.1 & 86.2 & 70.2 & 77.3 & 76.7 \\
    \bottomrule %
\end{tabular} %
\centering %
 \caption{Detailed ablation study results of our algebraic machine reasoning framework, evaluated on the I-RAVEN dataset, in terms of the overall answer-selection accuracy and the individual accuracies for all seven RPM configurations.}
 \label{table:ablation}
\end{table*}

\begin{algorithm}[H]
\caption{Detailed answer generation process.}
\textbf{Inputs: A concept matrix $\mathbf{J}=[J_{1,1},\dots,J_{3,2}]$.} 
\begin{algorithmic}[1] 
\STATE Compute $\text{comPos}(\mathbf{J})=[p_1,\dots,p_k]$.
\STATE $I_\text{out}\gets \langle 0 \rangle$.
\FOR{$p_i\in[p_1,\dots,p_k]$}
\STATEx \ \ \ \textit{ // initialize with fixed ``number" and ``position"}
\STATE $I^{(i)}\gets\langle x_{k-1}x_{p_i} \rangle$.  
\FOR{$K\in \mathcal{P}(\bar{\mathbf{J}}^{(p_i)}_1)\cap\mathcal{P}(\bar{\mathbf{J}}^{(p_i)}_2)$}
\STATE $I^{*}\gets$ Output ideal generated based on one of the Algorithms \ref{alg:inverseIntra}-\ref{alg:inverseBinary}, given inputs $K$ and $\bar{\mathbf{J}}^{(p_i)}_3$.
\IF{$I^*\neq \langle 0 \rangle$}
\STATE $I^{(i)}\gets I^{(i)}\cap I^{*}$
\ENDIF
\ENDFOR
\STATE $m\gets$ \mbox{generator randomly selected from $\mingen(I^{(i)})$.}
\STATE $I_\text{out}\gets I_\text{out}+\langle m \rangle$.
\ENDFOR
\STATE $\mathcal{G}\gets \varnothing$
\FOR{ $m\in\mingen(I_\text{out})$}
\STATE $m'\gets m$
\WHILE{$\exists \ell\in \mathcal{L}$ s.t. $x$ does not divide $m$  $\forall x\in\mathcal{A}_\ell$}
\STATE $x'\gets$ a variable randomly selected from $\mathcal{A}_\ell$. 
\STATE $m'\gets m'x'$.
\ENDWHILE
\STATE $\mathcal{G}\gets \mathcal{G}\cup\{ m' \}$.
\ENDFOR
\STATE $I_\text{out}\gets\langle \mathcal{G} \rangle$.
\RETURN $I_\text{out}$.
\end{algorithmic}
\label{alg:answerGenerationDetailed}
\end{algorithm}

\subsubsection{Overall process of answer generation}

In the RAVEN/I-RAVEN dataset, the values for some attribute within a panel may contain randomness by design. For example, there do not exist rules for both ``number" and ``position" to avoid conflicts, thus the entity positions in a panel from some configurations (such as \texttt{2$\times$2Grid}, \texttt{3$\times$3Grid}, and \texttt{Out-InGrid}) could be totally random once the number of entities is fixed. The values of other attributes could also be random; see Fig. \ref{fig:randomness} for example. It is difficult to extract patterns from those concept matrices with such random attribute values, and effectively impossible to generate an answer that matches the given correct answer. Hence, for any given RPM instance, we shall simplify the answer generation process by focusing on each single position. 

\begin{figure}[ht]
\centering
{\includegraphics[width=\linewidth]{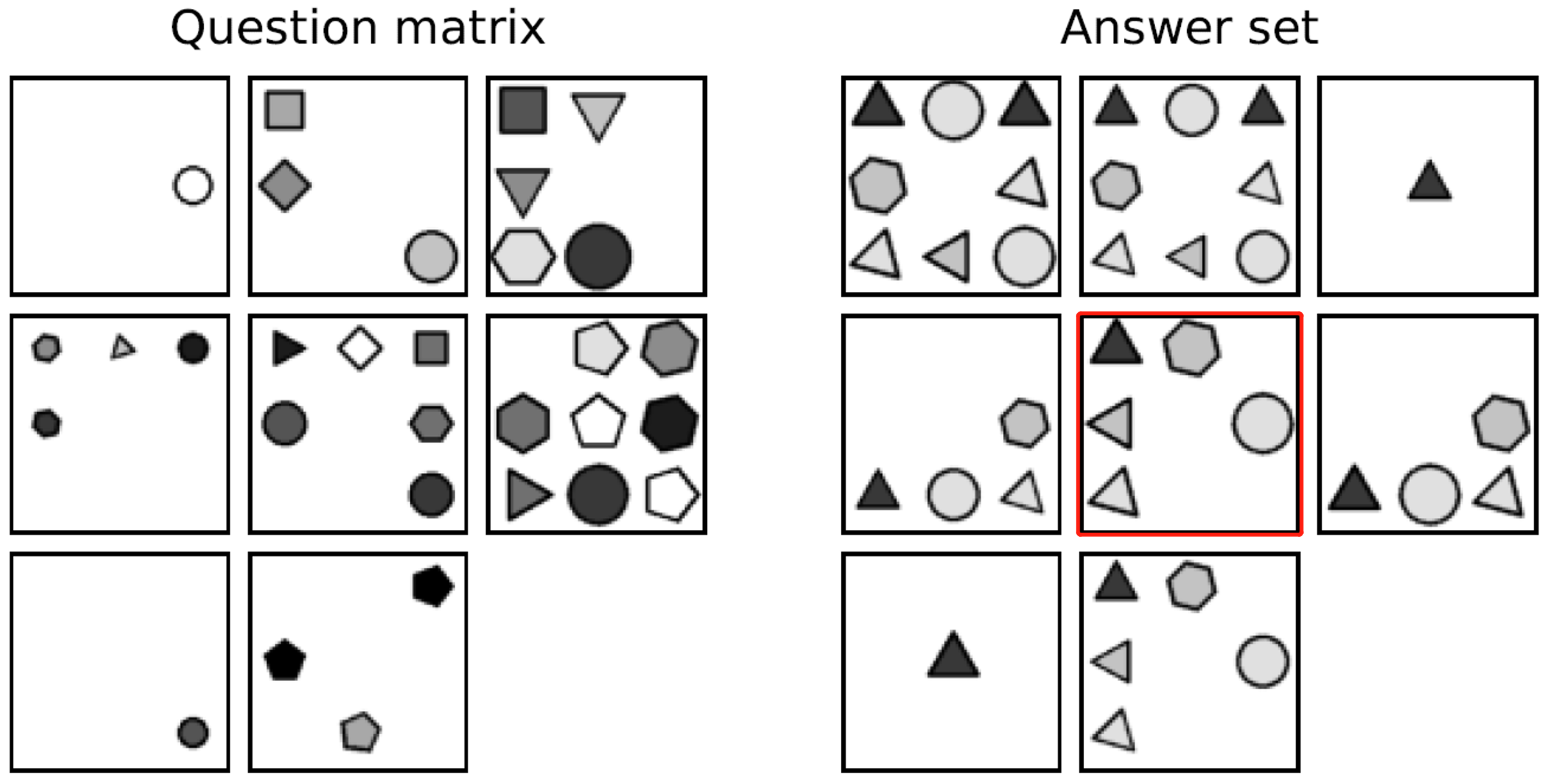}}
\caption{An example of a RPM instance with random values for attributes ``position", ``type" and ``color", within each panel. The correct answer is marked with a red box.}
\label{fig:randomness}%
\end{figure}

Recall that for a concept matrix $\mathbf{J}$ representing a given RPM instance, we extract the common patterns from each concept matrix $\check{\mathbf{J}}$ in an extended list of concept matrices, i.e. $[\mathbf{J},\bar{\mathbf{J}}^{(p_1)}, \hat{\mathbf{J}}^{(p_1)},\dots,\bar{\mathbf{J}}^{(p_k)}, \hat{\mathbf{J}}^{(p_k)}]$; see Section \ref{Sec:extractPattern}. Here $p_1,\dots,p_k$ are the common positions we extracted from the 8 panels of the question matrix corresponding to $\mathbf{J}$, and all $\bar{\mathbf{J}}^{(p_t)}, \hat{\mathbf{J}}^{(p_t)}$ are defined as follows:

\begin{itemize}
    \item Each concept $\smash{\bar{J}^{(p_t)}_{\smash{i,j}}}$ in $\smash{\bar{\mathbf{J}}^{(p_t)}}$ is generated by the unique generator in $J_{i,j}$ that is divisible by $p_t$;
    \item Each concept ${\hat{J}^{(p_t)}_{\smash{i,j}}}$ in $\smash{\hat{\mathbf{J}}^{(p_t)}}$ is generated by all generators in $J_{\smash{i,j}}$ that are not divisible by $p_t$. 
\end{itemize}
In other words, for a common position $p_t$, each panel in the original question matrix is split into two panels, with one panel containing only the entity in position $p_t$ (represented by the corresponding concept in $\smash{\bar{\mathbf{J}}^{(p_t)}}$), and the other panel containing all the remaining entities (represented by the corresponding concept in $\smash{\hat{\mathbf{J}}^{(p_t)}}$).
In the answer generation task, we shall consider only all $\bar{\mathbf{J}}^{(p_i)}$ as discussed before. Intuitively, we can separately extract the patterns restricted to each single common position, across the panels in the original question matrix. Since the attributes ``number" and ``position" are inherently fixed within each $\bar{\mathbf{J}}^{(p_i)}$, we only need to extract patterns for the remaining 3 attributes. See Algorithm \ref{alg:answerGenerationDetailed} for details. Note that in Line 6 of Algorithm \ref{alg:answerGenerationDetailed}, when generating the output ideal $I^*$, the exact algorithm (from Algorithm \ref{alg:inverseIntra}-\ref{alg:inverseBinary}) to apply to the inputs $K$ and $\bar{\mathbf{J}}_3^{(p_i)}$ is uniquely determined by the syntax of $K$.

\begin{figure}[tb]
  \centering
  \includegraphics[width=1\linewidth]{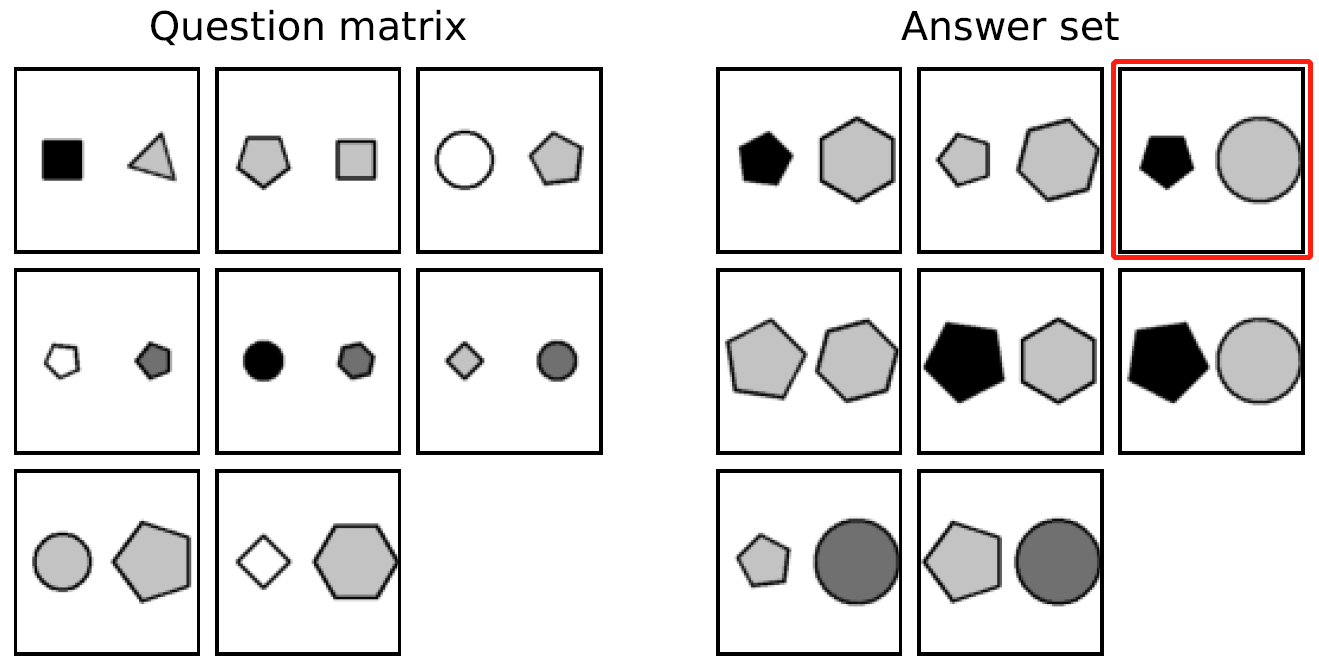}
   \caption{An example of RPM instance from the I-RAVEN dataset. The correct answer is marked with a red box. (This is the same as Fig. \ref{fig:RPMexample}; we put it here for convenience.)}
   \label{fig:RPMexample2}
\end{figure}

\subsection{Examples for invariance modules}
\label{AppendixSubsec:examples}
Consider the RPM instance depicted in Fig. \ref{fig:RPMexample2}.  In this subsection, we provide specific computational details on this RPM instance to show how the invariance modules are used to extract patterns, and generate the correct answer.

\textbf{Example 1 (Intra-invariance module)}. Consider the first row in Fig. \ref{fig:RPMexample2}. By extracting attribute values for entities in each panel via the object detection modules, we can obtain the algebraic representation given as follows:
\begin{align*}
    \mathbf{J}_1&:=[J_{1,1},J_{1,2},J_{1,3}], \\
    J_{1,1}&=\langle x_{\text{two}}x_{\text{left}}x_{\text{square}}x_{\text{black}}x_{\text{avg}}, x_{\text{two}}x_{\text{right}}x_{\text{triangle}}x_{\text{gray}}x_{\text{avg}}\rangle , \\
    J_{1,2}&=\langle x_{\text{two}}x_{\text{left}}x_{\text{pentagon}}x_{\text{gray}}x_{\text{avg}}, x_{\text{two}}x_{\text{right}}x_{\text{square}}x_{\text{gray}}x_{\text{avg}}\rangle , \\
    J_{1,3}&=\langle x_{\text{two}}x_{\text{left}}x_{\text{circle}}x_{\text{white}}x_{\text{avg}}, x_{\text{two}}x_{\text{right}}x_{\text{pentagon}}x_{\text{gray}}x_{\text{avg}}\rangle . 
\end{align*}

We first compute $\text{pd}(J_+)$ and $ \text{pd}(J_{\cap})$. 
\begin{align*}
    \text{pd}(J_+) &= \{ {\color{blue}\langle x_{\text{two}} \rangle}, {\color{blue}\langle x_{\text{avg}} \rangle} , {\color{blue}\langle x_{\text{left}},x_{\text{right}} \rangle} , {\color{blue}\langle x_{\text{white}},x_{\text{gray}} ,x_{\text{black}}\rangle} , \\
     & \qquad {\color{blue}\langle x_{\text{triangle}},x_{\text{square}},x_{\text{pentagon}},x_{\text{circle}} \rangle} , \langle x_{\text{left}},x_{\text{gray}} \rangle , \\
     & \qquad \langle x_{\text{square}},x_{\text{circle}},x_{\text{gray}} \rangle , \langle x_{\text{square}},x_{\text{white}},x_{\text{gray}} \rangle \\ 
     & \qquad \langle x_{\text{circle}},x_{\text{gray}},x_{\text{black}} \rangle , \\
     & \qquad \langle x_{\text{left}},x_{\text{triangle}},x_{\text{square}},x_{\text{pentagon}} \rangle \\
     & \qquad \langle x_{\text{right}},x_{\text{square}},x_{\text{pentagon}},x_{\text{circle}} \rangle \\
     & \qquad \langle x_{\text{right}},x_{\text{square}},x_{\text{pentagon}},x_{\text{white}} \rangle \\
     & \qquad \langle x_{\text{right}},x_{\text{pentagon}},x_{\text{circle}},x_{\text{black}} \rangle \\
     & \qquad \langle x_{\text{right}},x_{\text{pentagon}},x_{\text{white}},x_{\text{black}} \rangle \\
     & \qquad \langle x_{\text{triangle}},x_{\text{square}},x_{\text{pentagon}},x_{\text{white}} \rangle
     \} ,\\
    \text{pd}(J_{\cap}) &= \{ {\color{blue}{\langle x_{\text{two}} \rangle}}, {\color{blue}{\langle x_{\text{avg}} \rangle}} , {\color{blue}{\langle x_{\text{left}},x_{\text{right}} \rangle}} , {\color{blue}{\langle x_{\text{gray}} \rangle}}, \\
    & \qquad {\color{blue}{\langle x_{\text{triangle}},x_{\text{square}} \rangle}} , {\color{blue}{\langle x_{\text{square}},x_{\text{pentagon}} \rangle}} , \\
    & \qquad  {\color{blue}{\langle x_{\text{pentagon}},x_{\text{circle}} \rangle}} , \langle x_{\text{left}},x_{\text{triangle}} \rangle , \langle x_{\text{left}},x_{\text{square}} \rangle , \\
    & \qquad  \langle x_{\text{left}},x_{\text{pentagon}} \rangle , \langle x_{\text{right}},x_{\text{square}} \rangle , 
    \langle x_{\text{right}},x_{\text{pentagon}} \rangle , \\
    & \qquad \langle x_{\text{right}},x_{\text{circle}} \rangle , 
     \langle x_{\text{right}},x_{\text{white}} \rangle , 
    \langle x_{\text{right}},x_{\text{black}} \rangle , \\
    & \qquad  \langle x_{\text{triangle}},x_{\text{black}} \rangle ,  \langle x_{\text{pentagon}},x_{\text{white}} \rangle
    \}.
\end{align*}
For those primary components that are contained in attribute concepts, we have highlighted them in blue. These are the concepts that are deemed ``meaningful" for the RPM task.
Note that most of the primary components are not contained in any attribute concept.

Next, we compute the set intersection $\text{pd}(J_+)\cap \text{pd}(J_{\cap})=\{\langle x_\text{two}\rangle, \langle x_{\text{avg}}\rangle, \langle x_{\text{left}},x_{\text{right}} \rangle \}$, which can be interpreted to mean that all panels have two average-sized entities, one on the left side and one on the right side.

Recall that the output of the intra-invariance module is:
\begin{equation*}
\!\text{
\resizebox{1 \linewidth}{!}{
$\mathcal{P}_{\text{intra}}([J_1\dots J_k]):=\big\{\text{attr} \in \mathcal{L}\mid \exists I\in\text{pd}(J_+)\cap \text{pd}(J_{\cap}), I\subseteq \langle \mathcal{A}_{\text{attr}}\rangle\big\}$.}}\\[-0.45em]
\end{equation*}
Hence, we have $$\mathcal{P}_\text{intra}\left(\mathbf{J}_1\right)=\{\text{``size"},\text{``number"}, \text{``position"}\}.$$

\textbf{Example 2 (Inter-invariance module)}. Consider the second row $\mathbf{J}_2$ in Fig. \ref{fig:RPMexample2}. Similarly, we first obtain the algebraic representation of $\mathbf{J}_2$:
\begin{align*}
    \mathbf{J}_2&:=[J_{2,1},J_{2,2},J_{2,3}], \\
    J_{2,1}&=\langle x_{\text{two}}x_{\text{left}}x_{\text{pentagon}}x_{\text{white}}x_{\text{small}}, x_{\text{two}}x_{\text{right}}x_{\text{pentagon}}x_{\text{dgray}}x_{\text{small}}\rangle , \\
    J_{2,2}&=\langle x_{\text{two}}x_{\text{left}}x_{\text{circle}}x_{\text{black}}x_{\text{small}}, x_{\text{two}}x_{\text{right}}x_{\text{hexagon}}x_{\text{dgray}}x_{\text{small}}\rangle , \\
    J_{2,3}&=\langle x_{\text{two}}x_{\text{left}}x_{\text{square}}x_{\text{gray}}x_{\text{small}}, x_{\text{two}}x_{\text{right}}x_{\text{circle}}x_{\text{dgray}}x_{\text{small}}\rangle .
\end{align*}

Next, we compute $\text{pd}(J_+), \text{pd}(J_{\cap})$ as follows:
\begin{align*}
    \text{pd}(J_+) &= \{ {\color{blue}{\langle x_{\text{two}} \rangle}} , {\color{blue}{\langle x_{\text{small}} \rangle}} , {\color{blue}{\langle x_{\text{left}},x_{\text{right}} \rangle}} , \\
    & \qquad {\color{blue}{\langle x_{\text{white}},x_{\text{gray}},x_{\text{dgray}},x_{\text{black}}  \rangle}} , \\
    & \qquad {\color{blue}{\langle x_{\text{square}},x_{\text{pentagon}},x_{\text{hexagon}},x_{\text{circle}} \rangle}} , \\
    & \qquad \langle x_{\text{left}},x_{\text{dgray}} \rangle , \langle x_{\text{left}},x_{\text{pentagon}},x_{\text{hexagon}},x_{\text{circle}} \rangle , \\
    & \qquad \langle x_{\text{right}},x_{\text{square}},x_{\text{pentagon}},x_{\text{circle}} \rangle , \\
    & \qquad \langle x_{\text{right}},x_{\text{square}},x_{\text{pentagon}},x_{\text{black}} \rangle , \\
    & \qquad \langle x_{\text{right}},x_{\text{square}},x_{\text{circle}},x_{\text{white}} \rangle , \\
    & \qquad \langle x_{\text{right}},x_{\text{square}},x_{\text{white}},x_{\text{black}} \rangle , \\
    & \qquad \langle x_{\text{right}},x_{\text{pentagon}},x_{\text{circle}},x_{\text{gray}} \rangle , \\
    & \qquad \langle x_{\text{right}},x_{\text{pentagon}},x_{\text{gray}},x_{\text{black}} \rangle , \\
    & \qquad \langle x_{\text{right}},x_{\text{circle}},x_{\text{white}},x_{\text{gray}} \rangle , \\
    & \qquad \langle x_{\text{right}},x_{\text{white}},x_{\text{gray}},x_{\text{black}} \rangle , \\
    & \qquad \langle x_{\text{square}},x_{\text{pentagon}},x_{\text{circle}},x_{\text{dgray}} \rangle , \\
    & \qquad \langle x_{\text{square}},x_{\text{pentagon}},x_{\text{dgray}},x_{\text{black}} \rangle , \\
    & \qquad \langle x_{\text{square}},x_{\text{circle}},x_{\text{white}},x_{\text{dgray}} \rangle , \\
    & \qquad \langle x_{\text{square}},x_{\text{white}},x_{\text{dgray}},x_{\text{black}} \rangle , \\
    & \qquad \langle x_{\text{pentagon}},x_{\text{hexagon}},x_{\text{circle}},x_{\text{gray}} \rangle , \\
    & \qquad \langle x_{\text{pentagon}},x_{\text{circle}},x_{\text{gray}},x_{\text{dgray}} \rangle , \\
    & \qquad \langle x_{\text{pentagon}},x_{\text{gray}},x_{\text{dgray}},x_{\text{black}} \rangle , \\
    & \qquad \langle x_{\text{circle}},x_{\text{white}},x_{\text{gray}},x_{\text{dgray}} \rangle ,   \} ,
    \end{align*}
    \begin{align*}
    \text{pd}(J_{\cap}) &= \{ {\color{blue}{\langle x_{\text{two}} \rangle}} , {\color{blue}{\langle x_{\text{small}} \rangle}} , {\color{blue}{\langle x_{\text{left}},x_{\text{right}} \rangle}} , {\color{blue}{\langle x_{\text{pentagon}} \rangle}} , \\
    & \qquad {\color{blue}{\langle x_{\text{square}},x_{\text{circle}} \rangle}} ,  {\color{blue}{\langle x_{\text{hexagon}},x_{\text{circle}} \rangle}} , \\
    &  \qquad {\color{blue}{\langle x_{\text{white}},x_{\text{dgray}} \rangle}} , {\color{blue}{\langle x_{\text{gray}},x_{\text{dgray}} \rangle}} , {\color{blue}{\langle x_{\text{dgray}},x_{\text{black}} \rangle}} , \\
    & \qquad \langle x_{\text{left}},x_{\text{hexagon}} \rangle ,  \langle x_{\text{left}},x_{\text{circle}} \rangle ,  \langle x_{\text{left}},x_{\text{dgray}} \rangle , \\
    & \qquad \langle x_{\text{right}},x_{\text{square}} \rangle , \langle x_{\text{right}},x_{\text{circle}} \rangle , \langle x_{\text{right}},x_{\text{white}} \rangle , \\
    & \qquad \langle x_{\text{right}},x_{\text{gray}} \rangle , \langle x_{\text{right}},x_{\text{black}} \rangle , \langle x_{\text{square}},x_{\text{dgray}} \rangle , \\
    & \qquad \langle x_{\text{hexagon}},x_{\text{black}} \rangle ,  \langle x_{\text{circle}},x_{\text{gray}} \rangle , \langle x_{\text{circle}},x_{\text{dgray}} \rangle     \}.
\end{align*}
Similarly, we have indicated in blue those primary components that are contained in attribute concepts.

Recall that the output of the inter-invariance module is:
\begin{equation*}
\resizebox{1 \linewidth}{!}{
$\mathcal{P}_{\text{inter}}([J_1,\dots,J_k]):=\Bigg\{(\text{attr},\mathcal{I}) \Bigg| \ \begin{matrix}\mathcal{I}\subseteq \text{pd}(J_\cap)-\text{pd}(J_+),  \\ \text{attr} \in \mathcal{L}, I\subseteq \langle \mathcal{A}_{\text{attr}} \rangle \  \forall I \in \mathcal{I} \end{matrix} \Bigg\},$}\\
\end{equation*}
Hence $\mathcal{P}_{\text{inter}}(\mathbf{J}_2)$ is the following set with two pairs:\\[0.8em]
\resizebox{1 \linewidth}{!}{$\bigg\{
\begin{matrix}
\big(\text{``color"},\{\langle x_{\text{white}}, x_{\text{dgray}}\rangle, \langle x_{\text{gray}}, x_{\text{dgray}}\rangle, \langle x_{\text{black}}, x_{\text{dgray}}\rangle\}\big), \\ \big(\text{``type"},\{\langle x_{\text{pentagon}}\rangle, \langle x_{\text{square}}, x_{\text{circle}}\rangle, \langle x_{\text{hexagon}}, x_{\text{circle}}\rangle  \}\big) 
\end{matrix}
\bigg\}$.}\\[0.8em]
The first pair encodes the information that the three extracted ``color" concepts correspond to ``white or dark gray", ``gray or dark gray", and ``black or dark gray". The second pair encodes the information that the three extracted ``type" concepts correspond to ``pentagon", ``square or circle", and ``hexagon or circle".

\textbf{Example 3 (Compositional invariance module)}. Consider the right subpanels of the first row $\bar{\mathbf{J}}^{(\text{right})}_1$ in Fig. \ref{fig:RPMexample2}, and $\Delta=1$. We can generate the new row:
\begin{align*}
\begin{split}
\bar{J}'_{1,1} &=  f_{\text{next}}^2(\bar{J}^{(\text{right})}_{1,1}|1) = \langle x_{\text{right}}x_\text{one}x_{\text{pentagon}}x_{\text{black}}x_{\text{small}} \rangle,\\[-0.1em]
\bar{J}'_{1,2} &=  f_{\text{next}}(\bar{J}^{(\text{right})}_{1,2}|1) = \langle x_{\text{left}}x_\text{two}x_{\text{pentagon}}x_{\text{dgray}}x_{\text{large}} \rangle,\\[-0.1em]
\bar{J}'_{1,3} &=  \bar{J}_{1,3}^{(\text{right})} = \langle x_{\text{right}}x_\text{one}x_{\text{pentagon}}x_{\text{gray}}x_{\text{avg}} \rangle .\\[-0.5em]
\end{split}
\end{align*}

Recall that the output of the compositional invariance module is:
\begin{equation*}
\resizebox{1 \linewidth}{!}{
$\mathcal{P}_{\text{comp}}([J_1,\dots,J_k]):=\Bigg\{(\text{attr},f) \Bigg| \ \begin{matrix}\exists I \in \bigcap_{i=1}^k \text{pd}(f^{k-i}(J_i)),  \\ \text{attr} \in \mathcal{L}, I \subseteq \langle \mathcal{A}_{\text{attr}} \rangle \end{matrix} \Bigg\}.$}\\
\end{equation*}
Hence, we have $\mathcal{P}_{\text{comp}}(\bar{\mathbf{J}}^{(\text{right})}_1)=\{ (\text{``type"},f_\text{next}|_{\Delta=1}) \}$, which means that the ``type" value (i.e. $x_\text{pentagon}$) remains invariant in the new row generated by $f_\text{next}$ with step $\Delta=1$.

\textbf{Example 4 (Binary-operator module)}. Recall that the output of the binary-operator module is:
\begin{equation*}
\resizebox{1 \linewidth}{!}{$
    \mathcal{P}_\text{binary}(\mathbf{J}_i):=\Bigg\{\overline{\boldsymbol{\oslash}} \Bigg| \ \begin{matrix} 
    \overline{\boldsymbol{\oslash}}=[\oslash_1,\dots,\oslash_{k-2}], \  \oslash_i \in \Lambda,
      \\ g(J_1) \oslash_1 \dots\oslash_{k-2} g(J_{k-1})=g(J_k) \end{matrix} \Bigg\}.$}\\
\end{equation*}
Let $g$ be the real-valued function on concepts given by $J \mapsto |\mingen(J)|$. Note that for both rows $\mathbf{J}_1$ and $\mathbf{J}_2$ in Fig. \ref{fig:RPMexample2}, every panel \smash{$J_{i,j}$} has two entities, which can be computed as $g(J_{i,j}) = 2$. Hence, if $\Lambda=\{+,-\}$, then $\mathcal{P}_\text{binary}(\mathbf{J}_1)$ and $\mathcal{P}_\text{binary}(\mathbf{J}_2)$ are both empty.

\textbf{Example 5 (Answer generation)}. For the example depicted in Fig. \ref{fig:RPMexample2}, we have $\text{comPos}(\mathbf{J})=\{x_{\text{left}}, x_{\text{right}}\}$. Note that $\bar{\textbf{J}}^{(x_{\text{left}})}=\hat{\textbf{J}}^{(x_{\text{right}})}$. (The panel with only the left entity is identical to the panel without the right entity.) Hence we shall consider only the list of concept matrices $[\textbf{J}, \bar{\textbf{J}}^{(x_{\text{left}})}, \bar{\textbf{J}}^{(x_{\text{right}})}]$. 

By applying the invariance modules iteratively, we can then compute $\mathcal{P}_{1,2}(\mathbf{J}):=\mathcal{P}_1^{\text{(all)}}(\textbf{J})\cap\mathcal{P}_2^{\text{(all)}}(\textbf{J})$, whose 13 elements are labeled as $\mathfrak{P}_1,\dots,\mathfrak{P}_{13}$: \\[0.8em]
\noindent \resizebox{1 \linewidth}{!}{
$\begin{matrix}    \mathfrak{P}_1=\big(\text{``num"}, \textbf{J}\big) & \mathfrak{P}_2=\smash{\big(\text{``num"},\bar{\textbf{J}}^{(x_{\text{left}})}\big)} & \mathfrak{P}_3=\smash{\big(\text{``num"}, \bar{\textbf{J}}^{(x_{\text{right}})}\big)}
\end{matrix}$} 

\noindent \resizebox{1 \linewidth}{!}{
$\begin{matrix}
\mathfrak{P}_4=\smash{\big(\text{``pos"}, \textbf{J}}\big) &
\mathfrak{P}_5=\smash{\big(\text{``pos"}, \bar{\textbf{J}}^{(x_{\text{left}})}\big)} & \mathfrak{P}_6=\smash{\big(\text{``pos"}, \bar{\textbf{J}}^{(x_{\text{right}})}\big)}
\end{matrix}$}

\noindent \resizebox{0.85 \linewidth}{!}{
$\begin{matrix}
\noindent \mathfrak{P}_7=\smash{\big((\text{``type"}, \{\langle x_\text{square}\rangle,\langle x_\text{pentagon}\rangle, \langle x_\text{circle} \rangle \}), \bar{\textbf{J}}^{(x_{\text{left}})}\big)} 
\end{matrix}$}

\noindent \resizebox{0.95 \linewidth}{!}{
$\begin{matrix}
\mathfrak{P}_8=\smash{\big((\text{``type"}, f_\text{next}|_{\Delta=1}), \bar{\textbf{J}}^{(x_{\text{right}})}\big)} & \mathfrak{P}_{9}=\smash{\big(\text{``color"}, \bar{\textbf{J}}^{(x_{\text{right}})}\big)} 
\end{matrix}$}

\noindent \resizebox{0.85 \linewidth}{!}{
$\begin{matrix}
\mathfrak{P}_{10}=\smash{\big((\text{``color"}, \{\langle x_\text{white} \rangle,\langle x_\text{gray} \rangle, \langle x_\text{black} \rangle \}), \bar{\textbf{J}}^{(x_{\text{left}})}\big)} 
\end{matrix}$}

\noindent \resizebox{1 \linewidth}{!}{
$\begin{matrix}
\mathfrak{P}_{11}=\big(\text{``size"}, \textbf{J}\big) &
\mathfrak{P}_{12}=\big(\text{``size"}, \bar{\textbf{J}}^{(x_{\text{left}})}\big) & \mathfrak{P}_{13}=\big(\text{``size"}, \bar{\textbf{J}}^{(x_{\text{right}})}\big)

\end{matrix}$} \\[-0.3em]

We can also generate the algebraic representations of the first two panels in the $3$rd row $\mathbf{J}_3:=[J_{3,1}, J_{3,2}]$:
\begin{align*}
    J_{3,1}&=\langle x_{\text{two}}x_{\text{left}}x_{\text{circle}}x_{\text{avg}}x_{\text{gray}}, x_{\text{two}}x_{\text{right}}x_{\text{pentagon}}x_{\text{gray}}x_{\text{large}} \rangle; \\
    J_{3,2}&=\langle x_{\text{two}}x_{\text{left}}x_{\text{square}}x_{\text{avg}}x_{\text{white}}, x_{\text{two}}x_{\text{right}}x_{\text{hexagon}}x_{\text{gray}}x_{\text{large}} \rangle.\\[-1.3em]
\end{align*}

Among these 13 patterns, only $\mathfrak{P}_{11}$ (the ``size" values of entities are invariant across each row in $\mathbf{J}$) \textit{conflicts} with $\mathbf{J}_3$, as the ``size" values of the two entities in $J_{3,1}$ (or $J_{3,2}$) are different.
Each remaining non-conflicting pattern gives information about the corresponding attribute value for $J_{3,3}$. 
$\mathfrak{P}_{1}$ to $\mathfrak{P}_{6}$: two entities in $J_{3,3}$, with one entity on the left side and the other entity on the right side;
$\mathfrak{P}_{7}\Rightarrow x_{\text{pentagon}}$ for the left entity;
$\mathfrak{P}_{8}\Rightarrow x_{\text{circle}}$ for the right entity;
$\mathfrak{P}_{9}\Rightarrow x_{\text{gray}}$ for the right entity; 
$\mathfrak{P}_{10}\Rightarrow x_{\text{black}}$ for the left entity; $\mathfrak{P}_{12}\Rightarrow x_{\text{avg}}$ for the left entity; and $\mathfrak{P}_{13}\Rightarrow x_{\text{large}}$ for the right entity.
Consequently, we can collect all the attribute values to form our generated answer concept
\begin{equation*}
    \resizebox{1 \linewidth}{!}{$J'_{3,3}=\langle x_{\text{two}}x_{\text{left}}x_{\text{pentagon}}x_{\text{avg}}x_{\text{black}}, x_{\text{two}}x_{\text{right}}x_{\text{circle}}x_{\text{large}}x_{\text{gray}} \rangle$.\\}
\end{equation*}

\begin{figure*}[!ht]
  \centering
  \includegraphics[width=0.88\linewidth]{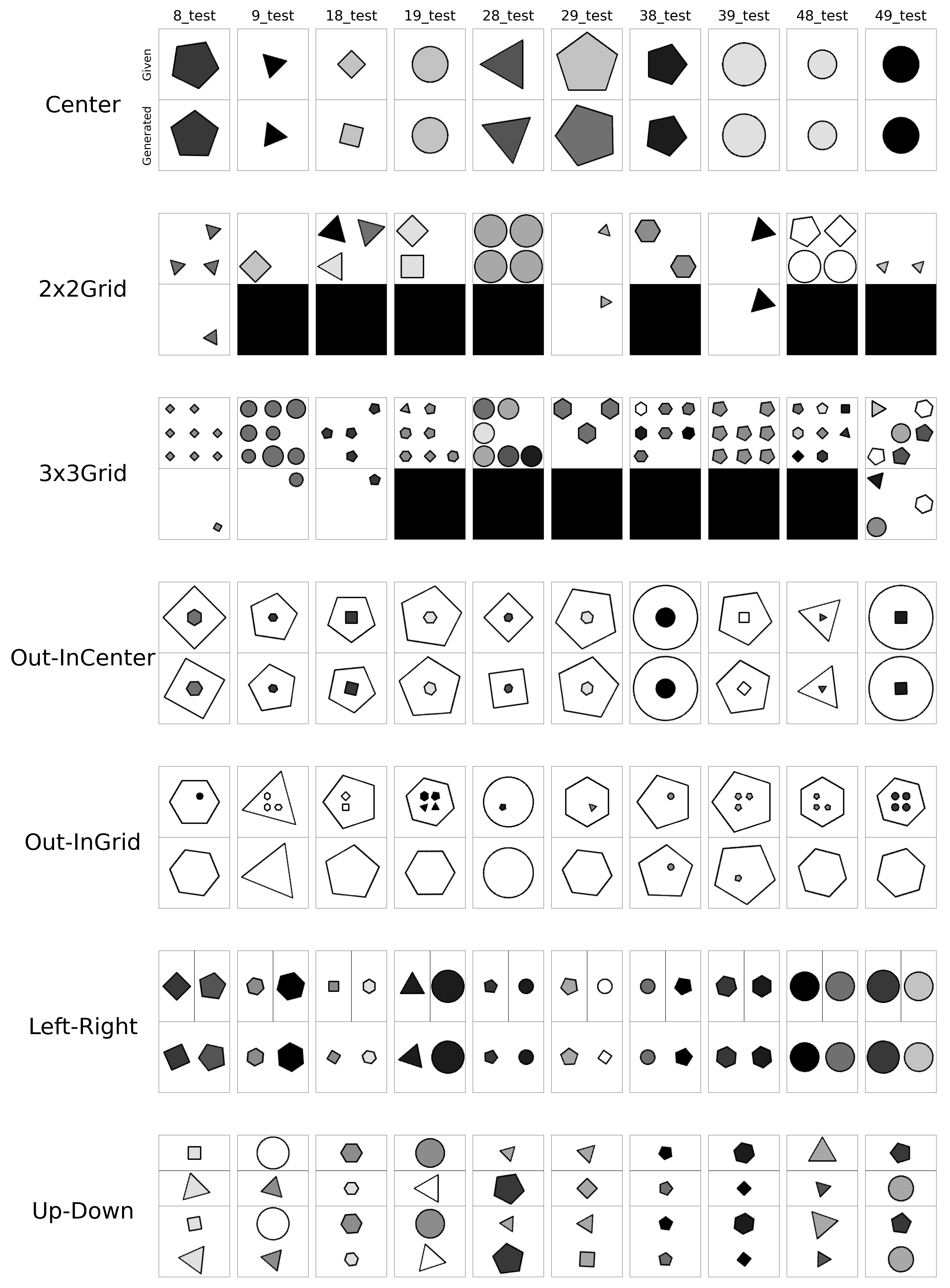}
   \caption{A collection of the given answer images (top) and the corresponding generated images (bottom) for the first 10 I-RAVEN test instances in each configuration. For those test instances for which our reasoning framework fails to extract non-conflicting patterns, the output of Algorithm \ref{alg:answerGenerationDetailed} would be $I_\text{out}=\langle 0 \rangle$; we have indicated such generated results as completely black panels.}
   \label{fig:exampleGeneration}
\end{figure*}

\subsection{Ablation study}
\label{AppendixSubsec:ablation}

In order to study the effect of each invariance module, we conducted an ablation study and evaluated the performance of our algebraic machine reasoning framework on the answer selection task; see Table \ref{table:ablation}. 

For any given RPM instance, there could be a tie in the \texttt{comPattern} list (defined in Algorithm \ref{alg:answerSelection}), i.e. there exist multiple entries in \texttt{comPattern} with the same highest score. Recall that the $i$-th entry in \texttt{comPattern} represents the number of extracted common patterns across 3 rows, with the $i$-th answer option inserted in the question matrix. Hence we report the weighted accuracies of our framework, defined as follows:
\begin{equation*}
    \frac{1}{N}\sum_{i=1}^N \frac{1}{n_i},
\end{equation*}
where $N$ denotes the total number of RPM test instances, and $n_i$ represents the number of answer options with the same highest score for the $i$-th test instance.

As shown in Table \ref{table:ablation}, the inter-invariance module has the most significant contribution towards the final average accuracy. Interestingly, this very same module allows us to discover an unexpected new pattern as depicted in Fig. \ref{fig:dis2}, which is a rather natural pattern that human could conceivably think of, but which is not one of the designed rules for I-RAVEN.

\subsection{Examples of generated answers}
\label{AppendixSubsec:generatedAnswers}
To further evaluate the performance of our algebraic machine reasoning framework on the answer generation task, we provide in Fig. \ref{fig:exampleGeneration} the comparison between the given correct answer images and our generated answer images for the I-RAVEN dataset. For each configuration, we choose the first 10 test files to show the comparison between the given answer images and the generated images. (In the I-RAVEN dataset, every 10 instances are split into 6 training instances, 2 validation instances, and 2 test instances. Hence every test instance has a filename that ends with ``8\_test" or ``9\_test".)

From Fig. \ref{fig:exampleGeneration}, we can see that for the configurations \texttt{Center}, \texttt{Out-InCenter}, \texttt{Left-Right}, and \texttt{Up-Down}, our generated answers capture almost all salient image attribute features, except for the angles of entities (i.e. the orientation with respect to the panels) which is not involved in any of the rules of RAVEN/I-RAVEN. However, for the remaining configurations, our reasoning framework is only able to generate a part of the answer. This is because we simplified the answer generation process by separately extracting patterns restricted to every single common position, across the panels in the question matrix. Hence, for those configurations where the RPM panels always contain a fixed set of positions, our reasoning framework is able to effectively extract the patterns and generate the correct answer. On the other hand, for those configurations where the RPM panels may contain an arbitrary set of positions, our reasoning framework does not perform as well.

\section{Further Discussion}
\label{AppendixSec:discussion}

\subsection{Why is algebraic machine reasoning different from logic-based reasoning?}
\label{subsubsec:PDnoAnalogInLogic}
Logic is the foundation of reasoning. Logic programming~\cite{lloyd2012foundations,jaffar1987constraint,apt1990logic} is a logic-based programming paradigm that serves as the foundational computational framework for logic-based reasoning methods. At the heart of these methods in logic-based reasoning is the idea that reasoning can be realized very concretely as the resolution (or inverse resolution) of logical expressions. Inherent in this idea is the notion of \emph{satisfiability}; cf. \cite{jaffar1987constraint}. Intuitively, we have a logical expression, typically expressed in some canonical normal form, and we want to assign truth values (true or false) to literals in the logical expression, so that the entire logical expression is satisfied (i.e. with truth value ``true''). In fact, much of today's very exciting progress in automated theorem proving \cite{abdelaziz2022learning, irving2016deepmath,lample2020deep,li2021isarstep,Paliwal_Loos_Rabe_Bansal_Szegedy_2020, wang2020learning, wu2021int, zombori2020prolog} is based on logic-based reasoning.

In contrast, algebraic machine reasoning builds upon computational algebra and computer algebra systems; cf. Appendix \ref{subsubsec:ComputationalSubroutines}. As described in Section \ref{Section:Discussion}, at the heart of our algebraic approach is the idea that reasoning can be realized very concretely as solving computational problems in algebra. Crucially, there is no notion of satisfiability. We do not assign truth values (or numerical values) to concepts in $R = \Bbbk[x_1, \dots, x_n]$. In particular, although primitive concepts $\langle x_1\rangle, \dots, \langle x_n\rangle$ in $R$ correspond to the variables $x_1, \dots, x_n$, we do not assign values to primitive concepts. Instead, as we have emphasized in Section \ref{Section:method}, ideals are treated as the ``actual objects of study'', and we reduce ``solving a reasoning task'' to ``solving non-numerical computational problems involving ideals''.

This contrast is perhaps even more evident when comparing algebraic machine reasoning to inductive logic programming~\cite{muggleton1991inductive,Muggleton1995:InverseEntailment}. Both seek to solve reasoning tasks involving the discovery of new patterns, but they are fundamentally different approaches.

In inductive logic programming, we are given (as our starting point) background knowledge $B$ and examples $E$, both in the form of logical formulas. For example, $B$ could be a conjunction of implications $(p_1 \wedge \dots \wedge p_k \to q)$, and $E$ could be a formula in conjunctive normal form. The goal is to discover a hypothesis $H$, which is a logical formula that must satisfy $B \wedge H \models E$ (i.e. $B \wedge H$ entails $E$). Informally, we seek to discover a ``good'' hypothesis $H$ that is consistent with both background knowledge and examples. The key technical challenge is to find a hypothesis $H$ that is ``good'' in a precise computable sense, from among the numerous (potentially exponentially many) possible hypotheses. In essence, inductive logic programming tackles this technical challenge as a search-and-selection problem~\cite{ray2003hybrid,kimber2009induction,Muggleton1995:InverseEntailment}: In the space of all candidate hypotheses, search for a subset of valid hypotheses; then within this subset, select a valid hypothesis that maximizes or minimizes a pre-defined objective function (e.g. select $H$ that minimizes Kolmogorov complexity).

In contrast, algebraic machine reasoning is intrinsically not search-based. When solving the RPM task, the new patterns discovered are computed via algebraic computations. For the RPM answer generation task, we generate answers to RPM instances not by searching for valid answers in the (huge) space of all candidate answers, but by directly computing our answer. Pattern discovery becomes computing: We are computing various primary decompositions, and computing which of the primary components are contained in attribute concepts. For the RPM answer selection task, we solve it as a ``compute-and-select'' problem. For the RPM answer generation task, we compute attribute values from the extracted patterns wherever possible, and randomly select attribute values for those attributes of entities not involved in the extracted patterns.

When solving RPMs, if we try to frame the reasoning process of our algebraic approach in a language similar to inductive logic programming, then have the following: Given an RPM instance represented as a concept matrix $\mathbf{J}$, we extract patterns $\mathcal{P}_{1,2}(\mathbf{J})$ from the first two rows of $\mathbf{J}$; here $\mathcal{P}_{1,2}(\mathbf{J})$ takes the role of ``background knowledge''. We then systematically check if the extracted patterns conflict with concepts $J_{3,1}, J_{3,2}$ (representing the two panels in the 3rd row of the question matrix); the patterns that can be extracted from $[J_{3,1}, J_{3,2}]$ take the role of ``examples''. Finally, we wish to discover a suitable concept $J_{3,3}$ for the missing 9-th panel; here $J_{3,3}$ takes the role of ``hypothesis''.

However, a closer look would reveal several disparities in our attempted analogy.
We seem to have analogous notions for ``background knowledge'', ``examples'' and ``hypothesis'' from the set-up of inductive logic programming, yet what we call ``background knowledge'' (i.e. the extracted patterns $\mathcal{P}_{1,2}(\mathbf{J})$) is not given prior knowledge, but must be computed. Similarly, the ``examples'' are not given prior knowledge, but are extracted patterns from $[J_{3,1}, J_{3,2}]$ that must be computed. To discover a ``hypothesis'', we are not searching for it, but computing it. Effectively, the common missing ingredient for making this ``analogy'' proper is algebraic computations, which is the essence of algebraic machine reasoning.

\subsection{Potential societal impact of algebraic machine reasoning}

Algebraic machine reasoning could potentially help to automate ``easier'' reasoning tasks currently performed by humans. In education, algebraic machine reasoning could help with the design of better intelligence tests. 
In finance, our framework could help in processing personal bank loan applications or detecting fraud, based on invariant features of such cases. A negative implication, however, could also be present. Our reasoning process is based on extracting patterns from a few examples (as in the RPM task) and then generalizing. If the cases of bank loan applications or fraud have already been unfairly associated with certain socio-economic groups, then that inequity would be propagated even in our reasoning framework.

Another potential downstream application of our framework is in medical diagnostics. Diagnostic decision making, based on a patient's reported symptoms and past medical history, could be modeled as reasoning tasks. Specific symptoms and medical conditions could be encoded as concepts and a final medical diagnosis could potentially be computed. If properly implemented, an algebraic machine reasoning framework for medical diagnosis would significantly speed up the diagnostic decision-making process, as well as reduce the cost of medical diagnostics. Unfortunately, such an application, while bringing clear benefits, also comes with ethical concerns. 
If a wrong medical diagnosis is made, when a doctor uses algebraic machine reasoning to aid in diagnostic decision-making, who gets the blame? Will the creators of the reasoning framework be liable for legal action?

Biasness could also be inherent in the \emph{human} encoding process of a machine reasoning framework. If concepts encoded reflect extant human bias, then reasoning output could be flawed, yet seemingly ``reasonable''. 
For any future application of algebraic machine reasoning to legal or criminal cases, care must be taken in assessing and processing information (concepts) associated with each counter-party. Though fairness is outside the scope of our framework, our approach could highlight potential misuses.

Lastly, since algebraic machine reasoning is able to surpass human performance on the RPM task (an intelligence test originally designed to evaluate geeral human intelligence and abstract reasoning), our work is effectively cracking open one of the biggest bastions of human cognition.
This may inadvertently contribute towards the fear by the general public, of AI outcompeting humans. One way to mitigate this is through education and outreach to demystify algebraic machine reasoning.

{\small
\bibliographystyle{ieee_fullname}
\bibliography{references}
}

\end{document}